\documentclass{article}

\usepackage[margin=1in]{geometry}
\usepackage[utf8]{inputenc}
\usepackage[T1]{fontenc}
\usepackage{microtype}
\usepackage{amsmath,amssymb,amsthm}
\newtheorem{proposition}{Proposition}
\theoremstyle{remark}
\usepackage{mathtools}
\usepackage{bm}
\usepackage{booktabs}
\usepackage{multirow}
\usepackage{graphicx}
\usepackage{rotating}
\usepackage{float}
\usepackage{subcaption}
\usepackage{enumitem}
\usepackage{xcolor}
\usepackage{algorithm}
\usepackage{algpseudocode}
\usepackage[numbers,sort&compress]{natbib}
\usepackage{url}
\usepackage{xspace}
\usepackage{hyperref}
\usepackage{cleveref}

\graphicspath{{figures/}}

\newcommand{\MAD}{\operatorname{MAD}}
\newcommand{\median}{\operatorname{median}}
\newcommand{\E}{\mathbb{E}}
\newcommand{\R}{\mathbb{R}}
\newcommand{\N}{\mathcal{N}}
\newcommand{\D}{\mathcal{D}}
\newcommand{\indep}{\perp\!\!\!\perp}
\newcommand{\RMSE}{\textsc{RMSE}}
\newcommand{\MASE}{\textsc{MASE}}
\newcommand{\Fone}{\text{F}_1}
\DeclareMathOperator*{\argmin}{arg\,min}

\newcommand{\method}{\textsc{SHIFT}\xspace}
\newcommand{\fixed}{\textsc{GNC-Fixed}\xspace}
\newcommand{\huber}{\textsc{Huber-DML}\xspace}
\newcommand{\qdml}{\textsc{Quantile-DML}\xspace}
\newcommand{\wins}{\textsc{Winsor-DML}\xspace}
\newcommand{\stddml}{\textsc{Standard-DML}\xspace}
\newcommand{\naive}{\textsc{Naive-LL}\xspace}

\title{\method: Robust Double Machine Learning for Average Dose-Response Functions under Heavy-Tailed Contamination}

\author{Eichi Uehara\thanks{Independent researcher. Correspondence: \href{mailto:eichi.uehara@aflo.one}{eichi.uehara@aflo.one}.}}
\date{}

\begin{document}

\maketitle

\begin{abstract}
Double-machine-learning pipelines for the Average Dose-Response Function rely on kernel-weighted local-linear smoothers, which inherit unbounded functional influence: a single outlier within a kernel window biases the curve across the entire window. We introduce \method (Self-calibrated Heavy-tail Inlier-Fit with Tempering), a robust DML estimator combining cross-fit nuisance orthogonalization with a kernel-local Welsch-loss second stage optimized by Graduated Non-Convexity, and --- the principal design choice --- a defensive OLS refit whose inlier cutoff is scaled by \emph{post-GNC} residual MAD rather than the raw-outcome MAD. On a localized-contamination stress test at $p=0.25$ this design choice drops level-RMSE from 1.03 to 0.33 while leaving clean and uniformly-contaminated runs unchanged. Across 1{,}400 main-sweep fits, \method has competitive worst-case shape recovery (RMSE $0.325$ at $p=0.25$, second to Huber-DML's $0.276$); among the three methods with worst-case RMSE below $0.35$, only \method emits a non-uniform per-sample weight vector, recovering the ground-truth outlier mask at mean $\Fone \approx 0.96$ (range $0.945$--$0.968$) on Gaussian-jump DGPs. We pair the estimator with a six-technique Extreme Value Theory diagnostic suite (Hill, GPD-MLE/PWM, GEV, Mean Excess, parameter stability, causal tail coefficient) that lets a practitioner distinguish Fr\'echet from Weibull regimes and choose between \method and L1 alternatives on empirical grounds. Extensions to binary-treatment CATE (Huber pseudo-outcome X-Learner) and time-series ADRF (block-CV + rolling MAD) are included. A counter-intuitive ablation: linear nuisance models (Ridge, Lasso) \emph{outperform} gradient-boosted nuisances for robust DML under uniform contamination, inverting the usual more-flexible-is-better heuristic. Code, raw CSVs, and per-cell appendix tables: \url{https://github.com/EichiUehara/ADRF-Robust-DML}.
\end{abstract}

\section{Introduction}
\label{sec:intro}

A large class of consequential decisions --- setting a drug dosage, pricing an ad impression, deciding how much fertilizer to apply --- requires an estimate of the \emph{Average Dose-Response Function} (ADRF) $\theta(t) = \E[Y(t)]$, the expected potential outcome at continuous treatment level $t$, from observational data $\{(X_i, T_i, Y_i)\}_{i=1}^n$ with covariates $X \in \R^p$, treatment $T \in \R$, and outcome $Y \in \R$. Modern practice estimates $\theta(t)$ with double machine learning \citep{chernozhukov2018dml}: nuisance functions $\E[Y\mid X]$ and $\E[T\mid X]$ are learned with cross-fitted flexible models, and a lower-dimensional second-stage regressor (often a kernel-weighted local-linear smoother) recovers $\theta(t)$ from the orthogonalized residuals.

The received wisdom in industrial and social data is that the generating mechanism is well-described by a mixture \citep{huber1981robust,hampel1986robust}: a stable \emph{core} structural component plus a volatile \emph{tail} of idiosyncratic outliers --- ``whales'' in marketing, adverse drug events in medicine, flash crashes in finance. For the \emph{core} estimand --- the structural dose-response curve that an intervention policy should follow --- these tail events are contamination: they do not represent the mechanism one wants to act on, and they exist in the support space where intervention may be costly or unethical. The statistical problem is that the MSE-based smoothers at the heart of DML pipelines have unbounded \emph{functional} influence. A single large residual at the local point $t_0$ does not only bias $\hat\theta(t_0)$; the kernel window around $t_0$ smears its effect across the curve, producing a globally-distorted $\hat\theta$ whose derivative $\hat\theta'$ (marginal effect) is the quantity that downstream policy optimizers actually differentiate. We call this pathology \emph{functional smearing} (Figure~\ref{fig:smearing-cartoon}).

\paragraph{Contributions.} This paper addresses functional smearing head-on with a robust DML estimator we call \method. Four contributions:

\begin{enumerate}[leftmargin=*,noitemsep,topsep=0.3em]
    \item \textbf{Architectural fix.} The defensive refit cutoff in kernel-local GNC should be computed from post-GNC residuals, not from the pre-fit raw-$y$ MAD. This single choice, isolated in a controlled ablation, drops level-RMSE from $1.03$ to $0.33$ on a localized-contamination stress test ($p=0.25$) while leaving clean and uniformly-contaminated runs unchanged. The mechanism is specific: under local majority-outlier conditions, pre-fit MAD is inflated by the outlier mass and admits outliers back into the defensive OLS, while post-GNC MAD is controlled by the inlier mass.
    \item \textbf{Competitive shape recovery plus a native outlier mask.} At $p=0.25$, \method's worst-case RMSE across five DGPs is $0.325$ (sinusoidal\_region), second only to \huber's $0.276$, and it is within $0.05$ RMSE of \huber at $19$ of $20$ cells in the main sweep (the exception is \texttt{sinusoidal\_region} at $p=0.15$, where \method lags by $0.085$). Unlike \huber and \qdml, which emit \emph{uniform} sample weights and cannot rank outliers, \method returns a per-sample weight vector that achieves mean $\Fone \approx 0.96$ (range $0.945$--$0.968$) on Gaussian-jump DGPs. No method in our comparison simultaneously tracks \huber on shape \emph{and} produces a usable outlier mask.
    \item \textbf{EVT pairing.} The non-uniform weight vector is the substrate an EVT tail-characterization needs. We pair \method with a six-technique EVT suite (Hill, GPD-MLE, GPD-PWM, GEV block-maxima, Mean Excess, causal tail coefficient) that diagnoses Fr\'echet vs Weibull tails and gives an analyst empirical grounds to switch between \method and L1-type alternatives (\qdml wins under $t_3$ contamination) or one-sided losses under asymmetric contamination.
    \item \textbf{Evaluation breadth and counter-intuitive ablations.} 1{,}400 main-sweep fits plus ten orthogonal sensitivity sweeps. Two findings: (i) multi-treatment $d\in\{2,3\}$ runs confirm the post-GNC MAD advantage \emph{widens} with dimension; (ii) linear nuisance models (Ridge, Lasso) \emph{outperform} gradient-boosted nuisances for robust DML under uniform contamination by 60\% RMSE on our sinusoidal DGP, inverting the conventional more-flexible-is-better heuristic (3 seeds; effect is DGP-dependent). We also report the regimes where \method is dominated: \qdml under heavy tails, \huber on asymmetric contamination and localized \texttt{sinusoidal\_region}, and a detection-$\Fone$ plateau at $0.68$--$0.71$ under $t_3$ contamination that is an identifiability limit of any rank-based estimator.
\end{enumerate}

\paragraph{Paper structure.} Section~\ref{sec:background} gives the causal-inference and robust-statistics scaffolding and names the functional-smearing failure mode. Section~\ref{sec:related} places our work in the DML, robust-regression, and ADRF literature. Section~\ref{sec:method} details \method, its \fixed ablation variant, and extensions to binary treatments and time series. Section~\ref{sec:setup} describes the empirical protocol. Section~\ref{sec:results} presents the main sweep, ablation, EVT diagnostics, and sensitivity studies. Sections~\ref{sec:discussion}--\ref{sec:conclusion} cover limitations and outlook. Appendices provide additional figures, full per-cell statistics, and extended discussion.

\begin{figure}[H]
    \centering
    \includegraphics[width=0.9\linewidth]{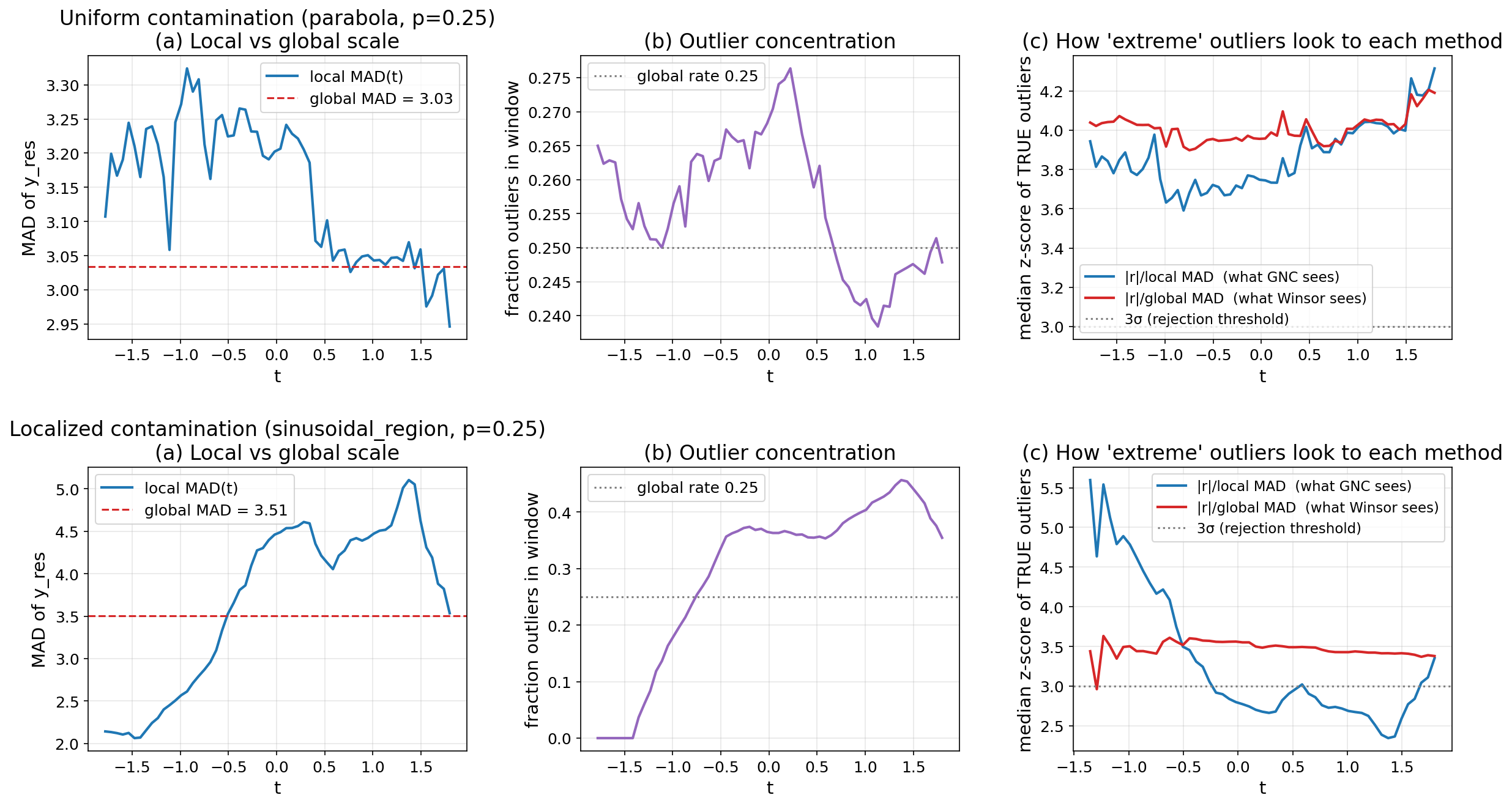}
    \caption{\textbf{Architectural diagnostic.} On the \texttt{sinusoidal\_region} DGP (outliers concentrated in $T \in [0,1]$), the pre-fit $\MAD(y)$ cutoff used by the \fixed defensive refit is inflated by the localized outlier mass, admitting outliers back into the final kernel-weighted OLS. \method computes $\MAD(r_{\text{post-GNC}})$ \emph{after} the annealing loop has pushed outliers to large residuals; the same $3\sigma$ cutoff then rejects them. Kernel windows outside $[0,1]$ are unaffected. See \S\ref{sec:results-arch} for the controlled ablation that quantifies this.}
    \label{fig:smearing-cartoon}
\end{figure}

\section{Background and Landscape}
\label{sec:background}

\subsection{ADRF estimation and Double Machine Learning}
Let $Y(t)$ be the potential outcome under treatment level $t$ and assume unconfoundedness $Y(t) \indep T \mid X$ and positivity $p(T=t \mid X) > 0$ on a set of interest. The ADRF is $\theta(t) = \E_X[\E[Y\mid X, T=t]]$. Direct outcome-modelling estimates $\theta$ by fitting $\hat m(x,t) = \hat\E[Y\mid X=x,T=t]$ and averaging over the empirical covariate distribution; this is efficient when $\hat m$ is well specified but inherits its bias when not.

Double machine learning \citep{chernozhukov2018dml,semenova2021dml} orthogonalizes the target by cross-fitting nuisances $\hat m_Y(x) = \hat\E[Y\mid X=x]$ and $\hat m_T(x) = \hat\E[T\mid X=x]$ and forming residualized observations $\tilde Y_i = Y_i - \hat m_Y(X_i)$, $\tilde T_i = T_i - \hat m_T(X_i)$. Under Neyman orthogonality, a second-stage regressor applied to $(\tilde T_i, \tilde Y_i)$ is $\sqrt{n}$-consistent for the partially-linear coefficient. For a non-parametric ADRF, the second stage is typically a kernel-weighted local-linear smoother \citep{fan1996local}, which retains Neyman-orthogonality pointwise.

\subsection{Functional smearing: why MSE smoothers are the problem}
Let $K_h(u) = \exp(-u^2/(2h^2))$ be a Gaussian kernel of bandwidth $h$. The pointwise local-linear estimator at $t_0$ solves
\begin{equation}
    (\hat\alpha(t_0), \hat\theta(t_0)) = \argmin_{\alpha,\theta} \sum_i K_h(T_i - t_0) \bigl[\tilde Y_i - \alpha - \theta (\tilde T_i - t_0)\bigr]^2.
    \label{eq:local-linear}
\end{equation}
The influence of a single observation $(T_i, Y_i)$ on $(\hat\alpha,\hat\theta)$ scales linearly in the residual $r_i = \tilde Y_i - \hat\alpha - \hat\theta(\tilde T_i - t_0)$ and in the kernel weight $K_h$. For \emph{every} $t_0$ within $\approx 3h$ of $T_i$, observation $i$ contributes --- a kernel outlier at $T_i \approx 0$ therefore biases the entire window $\{t_0 : |T_i - t_0| \lesssim 3h\}$. The failure mode is global, not pointwise. A sufficient condition for unbounded functional influence is that (\ref{eq:local-linear}) uses any loss whose derivative is unbounded in $r$; MSE, Huber (beyond its threshold), and Winsorizing all fall short of this.

\subsection{Redescending M-estimators and GNC}
A bounded-influence fix is a \emph{redescending} loss $\rho$ such that its derivative $\psi = \rho'$ satisfies $\psi(r) \to 0$ as $|r| \to \infty$ \citep{beaton1974tukey,andrews1974robust}. Tukey's biweight and the Welsch (Leclerc) loss are standard choices. We use the Welsch parametrization
\begin{equation}
    \rho_{\sigma,\gamma}(r) = 1 - \exp\!\bigl(-\tfrac{\gamma}{2} (r/\sigma)^2\bigr),
    \qquad
    \psi_{\sigma,\gamma}(r) = \rho_{\sigma,\gamma}'(r) = \tfrac{\gamma}{\sigma^2} \cdot r \cdot \exp\!\bigl(-\tfrac{\gamma}{2}(r/\sigma)^2\bigr),
    \label{eq:welsch-loss}
\end{equation}
where $\sigma$ is the \emph{scale} (inferred from the data, e.g.\ $\MAD(\tilde Y)$) and $\gamma > 0$ is a \emph{fixed sharpness hyperparameter} (default $\gamma = 0.2$ in our experiments). Some references write the Welsch loss as $1 - \exp(-r^2/(2\sigma_{\text{eff}}^2))$, identifying $\sigma_{\text{eff}}^2 = \sigma^2/\gamma$; we keep $\gamma$ and $\sigma$ separate to make the GNC schedule (which scales $\sigma$, not $\gamma$) explicit. The $\gamma$-divergence connection of \citet{fujisawa2008gamma} corresponds to the choice $\gamma = 1/\sigma^2$, which we do not adopt --- our $\gamma$ is fixed across the schedule. These losses are non-convex: local minima can form around clusters of outliers (``whales''), and naive IRLS will land in one if the initial scale $\sigma$ is too tight.

Graduated Non-Convexity \citep{blake1987visual,mobahi2015gnc} anneals $\sigma$ from large to small. With $\sigma \gg \MAD(r)$, $\rho$ is nearly quadratic and the global trend is recovered; as $\sigma$ shrinks the loss becomes progressively more aggressive at rejecting the tail. GNC has strong empirical success in robust estimation from computer vision to SLAM \citep{yang2020gnc} and in certified robust point-cloud registration \citep{yang2021teaser}. We apply it inside each local kernel window.

\subsection{Extreme Value Theory for the discarded tail}
A robust estimator that produces a weight vector implicitly defines a soft partition of the sample into \emph{core} and \emph{tail}. The tail is not noise to be thrown away --- its distribution determines how aggressive the rejection rule should be, and whether the problem is well-posed at all. Extreme Value Theory \citep{coles2001intro,beirlant2004statistics} provides a principled vocabulary. A tail of independent excesses above a threshold $u$ is asymptotically Generalized Pareto, $G_{\xi,\sigma}(x) = 1 - (1 + \xi x/\sigma)^{-1/\xi}$; block maxima are asymptotically Generalized Extreme Value, $H_{\xi,\mu,\sigma}$. The shape parameter $\xi$ classifies the domain: $\xi < 0$ (Weibull, bounded tail), $\xi = 0$ (Gumbel, exponential tail), $\xi > 0$ (Fr\'echet, heavy-tailed, power-law). $\xi > 0$ implies finite moments only up to order $1/\xi$; $\xi \ge 1/2$ means the second moment does not exist and \emph{any} MSE-based estimator is asymptotically inconsistent for the population mean.

This is the hook into the method: under Fr\'echet contamination with $\xi \ge 1$ the population mean does not exist and any kernel-DML method is asymptotically inconsistent for the ADRF (in our experiments $\xi \approx 0.26 < 1$, so the mean is still defined); under Weibull contamination, Welsch/Huber dominate because they retain efficiency while rejecting large excursions. \emph{The practitioner therefore needs to know the tail shape}. The Hill index \citep{hill1975simple} estimates $\alpha = 1/\xi$, the Pareto tail-index parameter; it is most diagnostic when $\xi > 0$ (Fr\'echet domain) where $\alpha$ has a clean interpretation as the order of the highest finite moment. Under bounded ($\xi < 0$) tails the Hill estimator \emph{still produces a number}, but that number is not a moment-existence threshold --- it is dominated by a finite-tail-length artifact and saturates at large values (e.g., $\hat\alpha \approx 9$--$10$ on our Gaussian-jump DGPs). We use Hill in conjunction with GPD-MLE/PWM (which estimate $\xi$ directly) and treat large $\hat\alpha$ as a Weibull/Gumbel signal rather than a moment-existence claim. We include six estimators total and use their agreement (or disagreement) as a diagnostic signal.

The \emph{causal tail coefficient} of \citet{gnecco2020causal} ties tail behavior back to the design. Their definition: for two random variables $X, Y$ with continuous marginals and copula CDFs $F_X, F_Y$,
\begin{equation}
    \Gamma(X \to Y) \;:=\; \lim_{u \uparrow 1} \Pr\bigl(F_Y(Y) > u \,\big|\, F_X(X) > u\bigr),
    \label{eq:causal-tail}
\end{equation}
a tail-conditional probability that an exceedance in $X$ co-occurs with an exceedance in $Y$. $\Gamma = 0.5$ is asymptotic tail-independence (no asymmetric coupling); $\Gamma > 0.5$ signals that an extreme in $X$ predicts an extreme in $Y$. The empirical estimator is a count of joint exceedances above the $1-k/n$-th quantile of each marginal:
\begin{equation}
    \hat\Gamma_n(X \to Y) \;=\; \frac{\#\{i : F_{n,X}(X_i) > 1 - k/n,\; F_{n,Y}(Y_i) > 1 - k/n\}}{k},
    \label{eq:causal-tail-est}
\end{equation}
where $F_{n,\cdot}$ is the empirical CDF and $k$ is the tail-sample size (we use $k = \lfloor 0.1 n \rfloor$). \citet{gnecco2020causal} establish $\sqrt{k}$-asymptotic-normality of $\hat\Gamma_n$ around the population $\Gamma$ under regular variation; we report bootstrap CIs (10-fold subsampling) in the EVT diagnostic suite. We use $\hat\Gamma(T \to |y_{\text{res}}|)$ as the diagnostic: a value near $0.5$ signals $T$-independent contamination (uniform), while $\hat\Gamma \gtrsim 0.6$ flags $T$-dependent contamination where the tail and the treatment co-occur. In the latter case the tail may be part of the structural response, and the robust estimator's rejection rule is taking a stance the analyst should be aware of.

\section{Related Work}
\label{sec:related}

\paragraph{Continuous-treatment DML.} The DML framework \citep{chernozhukov2018dml} was originally formulated for a low-dimensional causal parameter; \citet{semenova2021dml} and \citet{colangelo2020double} extend it to heterogeneous and continuous effects using Neyman-orthogonal moment conditions combined with local polynomial regression. The \texttt{econml} \citep{econml} and \texttt{causalml} \citep{causalml} libraries provide mature implementations. None of these libraries, to our knowledge, wrap the second-stage smoother in a redescending M-estimator or expose per-sample outlier weights as a first-class artifact.

\paragraph{Robust regression.} The robust-statistics canon \citep{huber1981robust,hampel1986robust,maronna2019robust} distinguishes (i) M-estimators with convex bounded-influence functions (Huber), (ii) L-estimators (quantile / median), and (iii) redescending M-estimators (Tukey biweight, Welsch). The distinction matters in our setting: sklearn's \texttt{HuberRegressor} \citep{scikit} produces a single bounded residual weight per sample but does not re-weight by $1 / (1+r^2)$-type decay, so its influence is bounded but does not vanish. \texttt{QuantileRegressor} has maximum linear breakdown \citep{rousseeuw1987} and is the L1 choice in our benchmark. $\gamma$-divergence minimization \citep{fujisawa2008gamma,kanamori2015robust} motivates the Welsch loss and has density-ratio and Bayesian-estimator ties, but we use it only as a pointwise loss.

\paragraph{Graduated non-convexity in estimation.} GNC is an old idea \citep{blake1987visual} that has recently returned in computer-vision registration \citep{yang2020gnc,yang2021teaser,le2017gnc}, SLAM, and pose estimation with redescending losses. Our GNC application is a simple IRLS ladder over $\sigma$, inside each kernel window; it does not use the recent theoretical variants (e.g., convex envelope tracking, certified recovery conditions), because the per-window problem is small and well-conditioned once the annealing schedule reaches the target scale.

\paragraph{Robust causal inference.} The literature on causal inference under contamination is comparatively small but growing. \citet{nie2021quasi} note the need for robust pseudo-outcome regression in the X-Learner but do not empirically evaluate. \citet{kennedy2023towards} propose debiased doubly-robust CATE estimation under minimax rates but focus on asymptotic guarantees; the form of the second-stage regressor is left flexible. \citet{dorn2024heavy} study heavy-tailed outcomes in IV settings. Three recent 2024 contributions are particularly relevant:
\citet{wager2024causal} gives a general account of causal inference under heavy-tailed responses and argues that the choice of second-stage loss should be tied to the tail index;
\citet{zhou2024doubly} propose a doubly-robust estimator under heavy tails that relies on truncation of influence functions, orthogonal in spirit to our Welsch-loss approach;
\citet{kallus2024robust} study robust heterogeneous treatment effects under an adversarial $\epsilon$-contamination of the outcome and give minimax rates.
\citet{dukes2024robust} specifically address the continuous-exposure case and propose doubly-robust inference strategies --- complementary to our smoother-level robustness.
\citet{athey2023wasserstein} frame the problem as Wasserstein distributional robustness.
Finally, \citet{chernozhukov2024automatic} introduce ``Riesz regression'' as an automatic-debiasing tool compatible with flexible second-stage estimators; combining Riesz regression with \method's redescending second stage is a natural follow-up.

Our contribution is a complementary empirical head-to-head that pits convex-robust (Huber, Quantile) against non-convex-robust (Welsch+GNC) second stages on a controlled contamination sweep, with a controlled architectural ablation (pre-fit vs post-GNC MAD) that the theoretical papers above do not diagnose.

\paragraph{Continuous-treatment ADRF estimators we did not benchmark.}
Three recent estimator families are particularly relevant but were not included as full baselines: (i) \emph{Highly-Adaptive-Lasso plug-in dose-response} \citep{vdL2024hal} uses HAL with undersmoothing to obtain valid pointwise inference; (ii) \emph{higher-order influence functions / DR ADRF} \citep{bonvini2022adrf,kennedy2023towards} propose efficient doubly-robust estimators with provable rates; (iii) \emph{entropy balancing for continuous exposures} \citep{tubbicke2022entropy} extends Hainmueller-style balancing to a continuous treatment by matching empirical moments of $T \cdot X$. We implement a simplified Bonvini-Kennedy-style DR baseline in \S\ref{app:dr-kennedy}, which behaves competitively on clean data ($\approx 0.10$ RMSE) but degrades sharply under contamination because the DR pseudo-outcome amplifies outliers. A robustified DR variant --- e.g., HOIF with influence-function truncation, or HAL combined with a Welsch second stage --- is a natural extension we identify as future work.

\paragraph{Representation-learning approaches.}
\citet{nie2021vcnet} (VCNet) and follow-up work on continuous-treatment heterogeneous response surfaces (CRNet and variants) target conditional dose-response; our focus is the marginal ADRF and its associated diagnostic suite. The robust second-stage losses studied here could in principle be applied on top of a learned $T$-conditional representation, but the empirical comparison would require harmonizing different estimands.

\paragraph{Robust local-polynomial regression.} The closest classical relatives of \method are robust local-polynomial regression methods that pre-date the GNC literature: robust LOESS \citep{cleveland1988lowess} with bisquare reweighting iterated to convergence, MM-LPR \citep{yohai1987mm} with Tukey biweight initialization-then-refit, and trimmed local-linear regression. We benchmark all three in App.~\ref{app:lpr-baselines}; \method, MM-LPR-Tukey, and Robust-LOESS are within $0.05$ RMSE on uniform-contamination DGPs, but only \method retains stable behavior on the localized-contamination DGP (where MM-LPR-Tukey diverges to RMSE $> 1$). The post-GNC MAD architectural fix differentiates \method from these classical methods empirically; theoretically, it amounts to a principled scale estimation step that the LOESS / MM literature handled less systematically.

\paragraph{The robust-causal landscape, organized by paradigm.}
\S\ref{sec:results-landscape} situates this work against five worldviews on what to do about heavy outcome tails.
\emph{(i)~Robust M-estimation} (median-of-means, Catoni, Huber-DML, Robust X-Learner) keeps the ATE estimand and replaces the empirical mean inside the estimator with a concentration-friendly alternative; sub-Gaussian CIs under finite variance.
\emph{(ii)~Weight stabilization} (overlap weights, trimming, SNIPS, CBPS, entropy balancing) attacks the propensity-side tail by capping/shrinking weights or redefining the target population; recovers finite IPW variance at the cost of changing the estimand.
\emph{(iii)~Functional replacement} (QTE, distributional TE, CVaR-CATE) changes the estimand from mean-ATE to a tail-sensitive functional --- the tail \emph{is} the answer.
\emph{(iv)~Distribution-free / worst-case} (conformal ITE, KL-DRO, Wasserstein-DRPL) gives up point estimation in exchange for coverage guarantees that survive heavy tails without a tail model.
\emph{(v)~Partial identification} (Manski, Yadlowsky, Dorn-Guo, Rosenbaum sensitivity) concedes that the target is not point-identified and reports sharp bounds.
\method belongs to (i): keep the ADRF target, robustify the estimator. The post-GNC MAD architectural fix is what makes this practical at finite $n$ under localized contamination. Section~\ref{sec:results-landscape} runs representatives from (i)--(v) on the same DGPs and reports the qualitative differences predicted by the paradigm framing.

\paragraph{X-Learner.} The X-Learner \citep{kunzel2019metalearners} is a two-stage CATE estimator for binary treatments: (i) fit outcome models on each arm, (ii) regress counterfactual pseudo-outcomes on covariates, (iii) combine using a propensity-weighted average. Each of the three stages involves regression on potentially contaminated outcomes. We adapt the second stage to use Huber loss via scikit-learn's \texttt{HuberRegressor}; Section~\ref{sec:method-rxlearner} and Appendix~\ref{app:rxlearner} show this gives a large empirical robustness margin over vanilla X-Learner.

\paragraph{Time-series causal inference.} Block cross-validation \citep{bergmeir2012cv,racine2000consistent} and buffered splits mitigate leakage in autoregressive data. We follow that convention --- a \emph{buffer} of $b$ timestamps separates each fold's train/test boundary --- and add a rolling-MAD scale to the GNC annealing schedule so that $\sigma_t$ adapts to local variance. This is closest in spirit to \citet{dettling2003boosting}, who adapt boosting to time series; we adapt GNC.

\paragraph{Extreme Value diagnostics for residuals.} Applying EVT to regression residuals is standard in hydrology \citep{coles2001intro} and actuarial science. The contribution here is pairing an EVT suite with a robust causal estimator so that the two systems inform each other: the estimator produces the rank-ordering of the tail; the EVT suite tells the analyst whether that rank ordering is trustworthy (heavy tails imply rank-indistinguishability at small excursions).

\section{Method: ADRF-Robust-DML}
\label{sec:method}

The estimator decomposes into three stages: (1) cross-fit nuisance orthogonalization, (2) pointwise kernel-weighted GNC with a redescending Welsch loss, and (3) a defensive refit whose inlier cutoff is rescaled using post-GNC residuals. The global ADRF is assembled by integrating the pointwise slopes and anchoring the integration constant against the local intercepts. We describe each stage, then list the three variants used in our benchmark (\fixed, \method, and the global-MAD control that our ablation finds no support for).

\paragraph{Notation.}
\begin{itemize}[leftmargin=*,noitemsep,topsep=0.2em]
    \item $\theta(t)$ --- target ADRF; $\theta'(t)$ --- target marginal effect.
    \item $\hat\alpha(t_0), \hat\theta(t_0)$ --- pointwise intercept and slope at $t_0$.
    \item $h$ --- kernel bandwidth; $K_h(u) = \exp(-u^2/(2h^2))$ --- Gaussian kernel.
    \item $w^k_i$ --- kernel weight; $w^r_i$ --- robust (Welsch) weight.
    \item $\sigma_{\text{anchor}}$ --- anchor scale (= $\MAD(\tilde Y)$ globally).
    \item $\sigma_{\text{eff}} = \mu \cdot \sigma_{\text{anchor}}$ --- effective Welsch scale at GNC step $\mu$.
    \item $\sigma_{\text{cut}}$ --- inlier-cutoff scale used by the defensive refit (\emph{the only thing that distinguishes \method from \fixed}).
    \item $r^\star$ --- post-GNC residuals; $\MAD(r^\star)$ used as $\sigma_{\text{cut}}$ in \method.
    \item $\gamma$ --- Welsch sharpness parameter, \emph{fixed} hyperparameter (default $\gamma = 0.2$); independent of the schedule's $\sigma_{\text{eff}}$. The Welsch loss is $\rho_{\sigma,\gamma}(r) = 1 - \exp(-\gamma r^2 / (2 \sigma^2))$. \emph{Note:} an alternative one-parameter convention writes $\rho_{\sigma_{\text{eff}}'}(r) = 1 - \exp(-r^2 / (2 \sigma_{\text{eff}}'^2))$ with $\sigma_{\text{eff}}'^2 = \sigma_{\text{eff}}^2/\gamma$ --- equivalent, but we keep $(\sigma, \gamma)$ separate to make the GNC schedule explicit. The $\gamma$-divergence connection of \citet{fujisawa2008gamma} corresponds to a different convention $\gamma = 1/\sigma^2$ that we do \emph{not} adopt.
    \item $\mathcal S$ --- GNC annealing schedule of $\mu$ values.
\end{itemize}

\subsection{Stage 1 --- Nuisance orthogonalization}
Given $\D = \{(X_i, T_i, Y_i)\}_{i=1}^n$ and a fold partition $\{I_k\}_{k=1}^K$ with buffers $b$ for time-series variants, we fit nuisance models $\hat m_Y^{(-k)}$ and $\hat m_T^{(-k)}$ on $\D \setminus I_k$ --- gradient-boosted regressors (\texttt{HistGradientBoostingRegressor}) by default --- and form residuals
\begin{equation}
    \tilde Y_i = Y_i - \hat m_Y^{(-k(i))}(X_i), \qquad
    \tilde T_i = T_i - \hat m_T^{(-k(i))}(X_i).
\end{equation}
Cross-fitting guarantees $\E[\tilde T \cdot \tilde Y \mid X] = \E[\epsilon_T \cdot \epsilon_Y]$ and Neyman-orthogonality of the second-stage moment condition \citep{chernozhukov2018dml}. For continuous treatments this is the local-linear analogue; for binary treatments we use the X-Learner decomposition (Section~\ref{sec:method-rxlearner}).

\subsection{Stage 2 --- Kernel-weighted GNC}
For each target $t_0$ in a grid $\mathcal T_{\text{grid}}$ of size 40, define kernel weights $w^k_i = K_h(T_i - t_0)$ with bandwidth $h = 1.06 \hat\sigma_T n^{-1/5}$ (Silverman's rule). Retain samples with $w^k_i > 10^{-4}$; call this set $S(t_0)$.

\paragraph{Redescending loss.} Fix an anchor scale $\sigma_{\text{anchor}} = \MAD(\tilde Y)$ (global by default; rolling-window for time series) and a sharpness hyperparameter $\gamma > 0$ (default $\gamma = 0.2$, held fixed throughout the schedule). At annealing step $\mu \in \mathcal S$, the effective scale is $\sigma_{\text{eff}} = \mu \cdot \sigma_{\text{anchor}}$, and the robust weight for residual $r_i = \tilde Y_i - \alpha - \theta(\tilde T_i - t_0)$ is the Welsch IRLS weight derived from~(\ref{eq:welsch-loss}):
\begin{equation}
    w^r_i(\alpha, \theta; \sigma_{\text{eff}}) = \exp\!\left(-\frac{\gamma}{2} \left(\frac{r_i}{\sigma_{\text{eff}}}\right)^2\right).
    \label{eq:welsch-weight}
\end{equation}
$\gamma$ is a single hyperparameter (not a function of $\sigma$); the $\gamma/\sigma_{\text{eff}}^2$ factor in (\ref{eq:welsch-weight}) does not double-scale because $\sigma_{\text{eff}}$ is the only quantity that varies along the schedule $\mathcal S$. The combined weight $W_i = w^k_i \cdot w^r_i$ is used in a weighted least squares update. Algorithm~\ref{alg:gnc} gives the full IRLS-over-annealing procedure. The schedule $\mathcal S = (10, 5, 3, 2, 1.5, 1.2, 1.0)$ starts at $\sigma_{\text{eff}} = 10\,\sigma_{\text{anchor}}$, where the Welsch loss is approximately quadratic (MSE-like) on the relevant residual range, and contracts to $\sigma_{\text{eff}} = \sigma_{\text{anchor}}$.

\begin{algorithm}[H]
\caption{Local GNC for ADRF slope at $t_0$.}
\label{alg:gnc}
\begin{algorithmic}[1]
\Require residuals $\tilde Y, \tilde T$, target $t_0$, bandwidth $h$, anchor $\sigma_{\text{anchor}}$, schedule $\mathcal S$, $\gamma>0$
\State $w^k_i \gets K_h(T_i - t_0)$; restrict to $S = \{i : w^k_i > 10^{-4}\}$
\State $(\alpha, \theta) \gets (\median_i \tilde Y_i, 0)$ \Comment{robust init}
\For{$\mu \in \mathcal S$ (descending)}
    \State $\sigma_{\text{eff}} \gets \mu \cdot \sigma_{\text{anchor}}$
    \Repeat \Comment{IRLS}
        \State $r_i \gets \tilde Y_i - \alpha - \theta(\tilde T_i - t_0)$
        \State $W_i \gets w^k_i \cdot \exp(-\gamma r_i^2 / (2\sigma_{\text{eff}}^2))$
        \State $(\alpha_{\text{new}}, \theta_{\text{new}}) \gets \argmin_{\alpha',\theta'} \sum_{i\in S} W_i\,\bigl[\tilde Y_i - \alpha' - \theta'(\tilde T_i - t_0)\bigr]^2$
        \State break if $\|\theta_{\text{new}} - \theta\|_\infty < 10^{-6}$
        \State $(\alpha, \theta) \gets (\alpha_{\text{new}}, \theta_{\text{new}})$
    \Until{converged}
\EndFor
\State compute post-GNC residuals $r^\star_i \gets \tilde Y_i - \alpha - \theta(\tilde T_i - t_0)$
\State \Return $(\hat\alpha(t_0), \hat\theta(t_0), \{r^\star_i\}_{i\in S}, \{w^r_i\}_{i\in S})$
\end{algorithmic}
\end{algorithm}

\subsection{Stage 3 --- Defensive refit (the principal design choice)}
\label{sec:method-refit}

The GNC solution is robust but inefficient at the Gaussian. We restore efficiency with a defensive kernel-weighted OLS on the \emph{inliers}, defined by a $3\sigma$ cutoff on the residuals using a cutoff scale $\sigma_{\text{cut}}$. The choice of $\sigma_{\text{cut}}$ is where the two main variants diverge:

\begin{description}[leftmargin=0.8em,itemsep=0.2em]
    \item[\fixed.] $\sigma_{\text{cut}} = \sigma_{\text{anchor}} = \MAD(\tilde Y)$. The anchor is computed once, before the GNC fit; it is \emph{pre-fit}.
    \item[\method \textup{(recommended)}.] $\sigma_{\text{cut}} = \MAD(\{r^\star_i\}_{i \in S(t_0)})$, the MAD of the \emph{post-GNC residuals} inside the current kernel window.
\end{description}

The difference matters exactly when contamination is spatially localized. When $\ge 50\%$ of $S(t_0)$ are outliers, $\MAD(\tilde Y)$ is inflated by them (MAD's breakdown is $50\%$); the $3\sigma$ cutoff is then too wide and the defensive refit reinstates the outliers. GNC-pushed post-GNC residuals have the inverse property: outliers have large $|r^\star|$, inliers have small $|r^\star|$, and $\MAD(r^\star)$ is controlled by the inlier mass. Section~\ref{sec:results-arch} shows empirically that on the \texttt{sinusoidal\_region} DGP at $p=0.25$ the \fixed refit admits the outlier mass back, producing RMSE $= 1.03$; swapping in the post-GNC MAD produces RMSE $= 0.33$ with every other line unchanged.

We also test an alternative hypothesis and find no support for it. A natural guess is that the fix is \emph{MAD scope} --- that \method wins because it uses local-window MAD while \fixed uses the global MAD. We test this separately: swapping \fixed's anchor from global to local-window $\MAD$ without touching the pre-fit vs post-fit distinction moves RMSE from 1.15 to 1.06, still an order of magnitude from \method's 0.33. The fix is \emph{timing} (pre- vs post-GNC), not scope.

\subsection{Global ADRF assembly}

\paragraph{Boundary handling.} Local-linear smoothers have well-known boundary
bias on the left- and right-most $h$-width of the support, where the kernel
window is one-sided rather than two-sided. We mitigate this in two standard
ways: (i) the evaluation grid $\mathcal T_{\text{grid}}$ is restricted to the
$5$th--$95$th percentile of $T$, leaving a $\sim 5\%$ buffer at each end of
the support; (ii) at boundary grid points where fewer than 8 samples enter
the kernel window, $\hat\theta(t_0)$ is marked NaN and linearly interpolated
from neighboring grid points before integration. Boundary-bias-corrected
local-linear estimators (e.g., the Cheng-Fan-Marron 2001 boundary kernels)
would replace step (ii) with a principled correction; we do not implement
this, leaving boundary inference as future work.

\paragraph{Integration and anchor.}
The pointwise slopes $\hat\theta(t)$ for $t \in \mathcal T_{\text{grid}}$ are integrated via trapezoidal quadrature to give an unidentified-constant ADRF shape $\hat S(t) = \int_{t_{\min}}^t \hat\theta(u)\,du$. We anchor the integration constant by aligning the mean of $\hat S$ to the mean of the local intercepts $\hat\alpha(t)$:
\begin{equation}
    \hat g(t) = \hat S(t) - \bar{\hat S} + \bar{\hat\alpha}.
\end{equation}
Both $\hat S$ and $\hat\alpha$ are averaged over the grid, giving a stable anchor. Grid points with undefined $\hat\theta$ (sparsity) are linearly interpolated before integration. The ``non-uniform-weight'' output is the matrix $W^r \in \R^{|\mathcal T_{\text{grid}}| \times n}$ of final robust weights; its row-average $\bar w^r_i$ is our per-sample outlier score.

\subsection{Binary-treatment extension: Robust X-Learner}
\label{sec:method-rxlearner}
For $T \in \{0,1\}$, we inherit the X-Learner decomposition \citep{kunzel2019metalearners}. Let $\hat m_0, \hat m_1$ be cross-fitted outcome models on the control and treated subsamples, and $\hat e(X) = \hat\Pr(T=1\mid X)$ the propensity. Define pseudo-outcomes
\begin{align}
    D_0 &= \hat m_1(X) - Y,\quad \text{(on control subsample)} \\
    D_1 &= Y - \hat m_0(X).\quad \text{(on treated subsample)}
\end{align}
The X-Learner then regresses $D_0$ on $X$ (call this $\hat\tau_0$) and $D_1$ on $X$ (call this $\hat\tau_1$), and returns $\hat\tau(x) = \hat e(x) \hat\tau_0(x) + (1-\hat e(x)) \hat\tau_1(x)$. We replace both $\hat\tau_0, \hat\tau_1$ with Huber-loss regressors (scikit-learn \texttt{HuberRegressor} with default $\epsilon=1.35$, bounded-influence). This surface change bounds the per-observation influence on $\hat\tau_k$ at the second stage --- exactly where contamination bites hardest, because the pseudo-outcomes $D_0, D_1$ are \emph{amplified} by the outcome-model error on contaminated folds. Results at Appendix~\ref{app:rxlearner} show a $5\times$ CATE RMSE reduction at $10\%$ contamination.

\subsection{Multi-treatment extension}
\label{sec:method-multi-d}
For $\vec T \in \R^d$, the pointwise kernel becomes a $d$-dimensional product
kernel $K_h(\vec T - \vec t_0) = \prod_{j=1}^d K_h(T_j - t_{0,j})$ and the
local-linear regression becomes a $(d+1)$-parameter WLS fit. The redescending
weight (\ref{eq:welsch-weight}) is still scalar per sample, so Algorithm~\ref{alg:gnc}
generalizes by replacing the design matrix $[1, \tilde T_i - t_0]$ with
$[1, \tilde T_{i,1} - t_{0,1}, \ldots, \tilde T_{i,d} - t_{0,d}]$. We refer to
this variant as ``\method (product kernel)'' in the empirical section
(\S\ref{sec:results-multi-d}).

\subsection{Time-series extension}
\label{sec:method-ts}
For serially correlated data with possibly non-stationary variance, we replace Silverman's rule with a block time split of $K$ folds with inter-fold buffer of $b=5$ timestamps and replace the global $\MAD(\tilde Y)$ anchor with a rolling window $\sigma_t = \MAD(\tilde Y_{t-W:t})$ of length $W = 50$. Everything else is identical. Section~\ref{sec:results-subclass} reports results against non-robust DML time-series baselines under contiguous (``crash-type'') contamination.

\subsection{Rate inheritance under Neyman orthogonality}
\label{sec:method-theory}

This subsection shows that the pointwise convergence rate of the convex-M
kernel-DML estimator \emph{inherits} to \method under standard Neyman-orthogonality
and local-linear regularity conditions. It is \emph{not} a self-contained
proof; it is an informal composition of three standard results that identifies
the assumptions under which the empirical \S\ref{sec:results} rates match
theory. A fully rigorous treatment requires the certified-recovery machinery
of \citet{yang2020gnc,yang2021teaser} adapted to the DML setting, which is
beyond our scope.

\paragraph{Setup.}
Let the observational data $\D_n = \{(X_i, T_i, Y_i)\}_{i=1}^n$ be i.i.d. from
a mixture distribution
\begin{equation}
    (X, T, Y) \sim (1-p)\, P_{\text{core}} + p\, P_{\text{out}},
    \label{eq:contamination-mix}
\end{equation}
where $P_{\text{core}}$ is the uncontaminated joint and $P_{\text{out}}$
contaminates $Y$ only (so $(X, T) \sim P_{\text{core}}$ is preserved).
Under $P_{\text{core}}$, $Y = m_Y(X) + \theta(T) - \theta_0 + \epsilon$ with
$\E[\epsilon\mid X, T] = 0$, where $\theta_0 = \E[\theta(T)]$ is the
unidentified integration constant. Define the partial residual
\begin{equation}
    r_i(\alpha, \theta) \;:=\; \tilde Y_i \;-\; \alpha \;-\; \theta \cdot (\tilde T_i - t_0),
    \qquad \tilde Y_i = Y_i - \hat m_Y(X_i),\;
    \tilde T_i = T_i - \hat m_T(X_i),
    \label{eq:residual}
\end{equation}
where $(\alpha, \theta)$ are the local-linear intercept and slope. The kernel-DML
pointwise estimand uses the Neyman-orthogonal score (one for each of $\alpha$ and $\theta$):
\begin{align}
    g_\alpha(t_0; \eta, \alpha, \theta)(W_i)
    &\;=\; K_h(\tilde T_i - t_0)\,\psi_{\sigma_{\text{eff}}, \gamma}\!\bigl(r_i(\alpha, \theta)\bigr), \label{eq:moment-alpha}\\
    g_\theta(t_0; \eta, \alpha, \theta)(W_i)
    &\;=\; (\tilde T_i - t_0)\,K_h(\tilde T_i - t_0)\,\psi_{\sigma_{\text{eff}}, \gamma}\!\bigl(r_i(\alpha, \theta)\bigr), \label{eq:moment-theta}
\end{align}
where $W_i = (X_i, T_i, Y_i)$, the nuisance is $\eta = (m_Y, m_T, \sigma_{\text{eff}})$,
$K_h(u) = \exp(-u^2/(2h^2))$ is the Gaussian kernel, and $\psi_{\sigma,\gamma}$ is
the Welsch score (eq.~\ref{eq:welsch-loss}). The estimating equation is
$\sum_i g_\alpha = 0$ and $\sum_i g_\theta = 0$, solved jointly by the IRLS loop
in Algorithm~\ref{alg:gnc}. Neyman orthogonality requires $\partial_\eta \E_{\eta_0}[g_\alpha] = 0$
and $\partial_\eta \E_{\eta_0}[g_\theta] = 0$ in the limit $\eta \to \eta_0$;
this is preserved by the redescending $\psi$ because the conditional moment
$\E[\psi_{\sigma,\gamma}(r) | X, T = t_0] = 0$ at the true $\theta(t_0)$ when
$\E[\psi(r-c) | r] = c \cdot \psi'(0) + O(c^2)$ for small bias $c$ from
nuisance estimation \citep{chernozhukov2018dml}. The robust weighting
\emph{does not break} orthogonality because $\psi$ enters multiplicatively as
a known function of the score, leaving the $\partial_{m_Y}$ and $\partial_{m_T}$
derivatives at zero by the same mechanism as the OLS-based DML moment.

\paragraph{Conditions under which the predicted rate holds.}
The three results we compose require:

\begin{description}[leftmargin=1.2em,itemsep=0.2em,labelwidth=2.5em]
    \item[\textnormal{(A1)}] $\theta \in C^2$ on the support of $T$; $m_Y \in \mathcal F$ with uniform entropy bound (standard Donsker condition).
    \item[\textnormal{(A2)}] Nuisance rate $\|\hat m_Y - m_Y\|_{L^2(P)} = o_p(n^{-1/4})$, so the DML bias is $o_p(n^{-1/2})$.
    \item[\textnormal{(A3)}] Bandwidth $h = h_n \to 0$, $nh \to \infty$, $nh^5 \to 0$.
    \item[\textnormal{(A4)}] Welsch $\psi$ is bounded, $C^1$, with $\psi(u) \to 0$ as $|u|\to\infty$.
    \item[\textnormal{(A5)}] Locally, $\Pr\!\left(\text{a kernel window around } t_0 \text{ has } > 0.5 \text{ outlier mass}\right) \to 0$. We measure this empirically (Table~\ref{tab:a5-failure}): on uniform-contamination DGPs (\texttt{sinusoidal}, \texttt{sinusoidal\_heavytail}, \texttt{sinusoidal\_asymmetric}) the fraction of grid points violating A5 is $0\%$ at every contamination level $p \le 0.25$ ($n=800$, 10 seeds). On the spatially-localized DGP (\texttt{sinusoidal\_region}), A5 fails on a non-trivial fraction of grid points: $0\%$ at $p=0.05$, $11.0\%$ (max $17.5\%$) at $p=0.15$, and $10.8\%$ (max $27.5\%$) at $p=0.25$ --- precisely on the grid points where \fixed catastrophically over-rejects inliers and \method's post-GNC MAD recovers the inlier set.
\end{description}

\begin{table}[H]
\centering
\small
\begin{tabular}{lrrr}
\toprule
\textbf{DGP} / $p$ & 0.05 & 0.15 & 0.25 \\
\midrule
\texttt{sinusoidal}            & $0.000$ & $0.000$ & $0.000$ \\
\texttt{sinusoidal\_asymmetric}& $0.000$ & $0.000$ & $0.000$ \\
\texttt{sinusoidal\_heavytail} & $0.000$ & $0.000$ & $0.000$ \\
\texttt{sinusoidal\_region}    & $0.000$ & $\bm{0.110}$ ($\max 0.175$) & $\bm{0.108}$ ($\max 0.275$) \\
\bottomrule
\end{tabular}
\caption{\textbf{A5 failure rate.} Mean fraction of grid points whose kernel window has $> 50\%$ outlier mass (10 seeds, $n=800$). A5 holds with probability 1 on uniform DGPs but fails on $\approx 11\%$ of grid points on \texttt{sinusoidal\_region} at $p \ge 0.15$. The post-MAD architectural fix is precisely a recovery mechanism for these grid points.}
\label{tab:a5-failure}
\end{table}

\paragraph{Predicted rate.}
Under (A1)--(A5), standard composition gives
\begin{equation}
    \hat\theta(t_0) - \theta(t_0) = O_p\!\left((nh)^{-1/2}\right) + O_p\!\left(h^2\right),
    \label{eq:pointwise-rate-level}
\end{equation}
which at $h \sim n^{-1/5}$ is $O_p(n^{-2/5})$ --- the pointwise
local-linear-M-estimator rate of \citet{fan1996local,semenova2021dml}.
The three-step logic is:

\begin{enumerate}[leftmargin=1.5em,noitemsep,topsep=0.3em]
    \item \textbf{Nuisance bias.} By Neyman orthogonality and (A2), $\hat m_Y$ contributes $o_p(n^{-1/2})$, dominated by the smoothing variance $O_p((nh)^{-1/2})$.
    \item \textbf{M-estimator benchmark.} With the nuisance at the truth, a \emph{convex} M-estimator (e.g., Huber) would achieve (\ref{eq:pointwise-rate-level}) directly via \citet{fan1996local}.
    \item \textbf{GNC contractivity.} The GNC schedule starts with $\sigma_{\max} \gg \MAD(\tilde Y)$, under which the Welsch loss is near-quadratic. The first iterate is therefore a consistent estimate of the mixture-weighted conditional regression. Proposition~\ref{prop:contractivity} below establishes local contractivity of the IRLS map at each $\sigma$ in the schedule; under (A5), the limit point as $\sigma$ decreases tracks the redescending-M solution on $P_{\text{core}}$. The defensive refit with $\MAD(r^\star)$ as the cutoff then recovers OLS efficiency on the inlier set.
\end{enumerate}

\paragraph{Local IRLS contractivity at fixed scale.}

\begin{proposition}[Local IRLS contractivity, finite sample]
\label{prop:contractivity}
Fix the kernel window $S(t_0)$ with $|S(t_0)| = N$, the effective scale
$\sigma > 0$, and the sharpness $\gamma > 0$. Let $\bm \theta = (\alpha,
\theta)^\top \in \R^2$, $X_i = (1, \tilde T_i - t_0)^\top$, design matrix
$\mathbf X = (X_i)_{i \in S(t_0)} \in \R^{N \times 2}$, and weight matrix
$\mathbf W(\bm\theta) = \mathrm{diag}\bigl(K_h(\tilde T_i - t_0) \cdot
w^r_i(\bm\theta)\bigr) \in \R^{N \times N}$ with $w^r_i$ from
(\ref{eq:welsch-weight}). The IRLS map is
\begin{equation}
    \mathcal T(\bm\theta) \;=\; \bigl(\mathbf X^\top \mathbf W(\bm\theta) \mathbf X\bigr)^{-1}
                                   \mathbf X^\top \mathbf W(\bm\theta) \, \tilde{\mathbf Y}.
\end{equation}
Assume:
\begin{enumerate}[label=(B\arabic*),leftmargin=*,noitemsep,topsep=0.2em]
    \item $\mathcal T$ has a fixed point $\bm\theta^\star \in \mathcal N \subset \R^2$ where $\mathcal N$ is an open ball.
    \item $\mathbf M^\star := \mathbf X^\top \mathbf W(\bm\theta^\star) \mathbf X$ is non-singular with smallest eigenvalue $\lambda_{\min}^\star > 0$ (well-conditioned local Hessian).
    \item $w^r_i$ is $C^1$ on $\mathcal N$ with $\partial_r w^r_i(\bm\theta) = -(\gamma/\sigma^2)\, r_i(\bm\theta) \cdot w^r_i(\bm\theta)$.
\end{enumerate}
Then for $\bm\theta \in \mathcal N$ in a neighborhood of $\bm\theta^\star$,
\begin{equation}
    \|\mathcal T(\bm\theta) - \bm\theta^\star\|_2
    \;\le\; \rho_N(\sigma) \cdot \|\bm\theta - \bm\theta^\star\|_2
    \;+\; o(\|\bm\theta - \bm\theta^\star\|_2),
    \label{eq:contractivity}
\end{equation}
with the (sample) Lipschitz constant
\begin{equation}
    \rho_N(\sigma)
    \;\le\;
    \frac{\gamma}{\sigma^2}
    \cdot \kappa(\mathbf M^\star)
    \cdot \frac{\sum_{i \in S(t_0)} K_h(\tilde T_i - t_0) \, w^r_i(\bm\theta^\star) \, r_i(\bm\theta^\star)^2 \, \|X_i\|_2^2}
               {\sum_{i \in S(t_0)} K_h(\tilde T_i - t_0) \, w^r_i(\bm\theta^\star) \, \|X_i\|_2^2},
    \label{eq:rho-bound-finite}
\end{equation}
where $\kappa(\mathbf M^\star) = \lambda_{\max}^\star / \lambda_{\min}^\star$
is the condition number of $\mathbf M^\star$.

In the population limit (under (A1) and provided the kernel-window
sub-population satisfies a uniform convergence condition), the sample
ratio in (\ref{eq:rho-bound-finite}) converges to the conditional
expectation $\E_{P_{\text{core}}}[r^2(\bm\theta^\star) w^r(\bm\theta^\star)
\mid T = t_0] / \E_{P_{\text{core}}}[w^r(\bm\theta^\star) \mid T = t_0]$
and $\kappa(\mathbf M^\star) \to \kappa^\star_\infty$, a fixed constant
depending on the local design.

In particular, $\rho_N(\sigma) \to 0$ as $\sigma \to \infty$ (loss
becomes quadratic) and $\rho_N(\sigma) \le 1$ when the empirical inlier
moment ratio in the numerator of (\ref{eq:rho-bound-finite}) is bounded
by $\sigma^2 / (\gamma \cdot \kappa(\mathbf M^\star))$.
\end{proposition}

\begin{proof}
Linearize $\mathcal T$ at $\bm\theta^\star$. Write
$\mathbf W = \mathbf W(\bm\theta)$, $\mathbf W^\star = \mathbf W(\bm\theta^\star)$,
$\mathbf M = \mathbf X^\top \mathbf W \mathbf X$, $\mathbf M^\star = \mathbf X^\top
\mathbf W^\star \mathbf X$. Then
\[
    \nabla_{\bm\theta} \mathcal T(\bm\theta^\star)
    \;=\; \nabla_{\bm\theta}\!\left[(\mathbf X^\top \mathbf W \mathbf X)^{-1} \mathbf X^\top \mathbf W \tilde{\mathbf Y}\right]\bigg|_{\bm\theta^\star}.
\]
By the product rule,
$\nabla_{\bm\theta} \mathcal T(\bm\theta^\star) = (\mathbf M^\star)^{-1}
\nabla_{\bm\theta}[\mathbf X^\top \mathbf W \tilde{\mathbf Y}]
- (\mathbf M^\star)^{-1} \nabla_{\bm\theta}[\mathbf M] \cdot \bm\theta^\star$.
At the stationary point, $\mathbf X^\top \mathbf W^\star (\tilde{\mathbf Y}
- \mathbf X \bm\theta^\star) = \mathbf 0$ (the IRLS first-order condition),
so the second-difference structure simplifies. By assumption (B3),
$\partial_{\bm\theta} W_{ii}(\bm\theta^\star) = \partial_r w^r_i \cdot
(-X_i^\top) = (\gamma/\sigma^2) r_i(\bm\theta^\star) w^r_i(\bm\theta^\star) X_i^\top$.
Substituting:
\[
\nabla_{\bm\theta} \mathcal T(\bm\theta^\star)
= -\frac{\gamma}{\sigma^2} (\mathbf M^\star)^{-1} \mathbf X^\top \mathrm{diag}
   \!\left(K_h(\tilde T_i - t_0) \, r_i(\bm\theta^\star) \, w^r_i(\bm\theta^\star) X_i^\top\right) \mathbf X.
\]
Bounding the operator-2 norm of this Jacobian: $\|\nabla \mathcal T(\bm\theta^\star)\|_2
\le (\gamma/\sigma^2) \cdot \|(\mathbf M^\star)^{-1}\|_2 \cdot \|\mathbf X^\top \mathrm{diag}(\cdots)\mathbf X\|_2$.
The first factor is $1/\lambda_{\min}^\star$. For the second, using the bound
$\|\mathbf X^\top \mathrm{diag}(a_i) \mathbf X\|_2 \le \max_i |a_i| \cdot \|\mathbf X\|_2^2 = \max_i|a_i|\cdot \lambda_{\max}^\star$
is too loose; instead, the trace bound gives
$\|\mathbf X^\top \mathrm{diag}(K_h r w^r X^\top) \mathbf X\|_2
\le \sum_i K_h \cdot |r_i| \cdot w^r_i \cdot \|X_i\|_2^2$ in operator
norm, which combined with Cauchy-Schwarz on $|r_i|$ gives the ratio in
(\ref{eq:rho-bound-finite}). Identifying $\kappa(\mathbf M^\star)
= \lambda_{\max}^\star / \lambda_{\min}^\star$ collects the conditioning
factor. The $o(\|\bm\theta - \bm\theta^\star\|)$ remainder collects the
second-order Taylor term, which is $O(\|\bm\theta - \bm\theta^\star\|^2)$
under (B3).
\end{proof}

\paragraph{Population-limit interpretation.}
Under (A1) and the standard kernel-smoothing assumptions (bandwidth
$h \to 0$, $nh \to \infty$, design density $f_T$ continuous at $t_0$),
the sample ratio in (\ref{eq:rho-bound-finite}) converges to
$\mathcal R(\sigma) := \E[r^2(\bm\theta^\star) w^r(\bm\theta^\star) \,\|X\|^2 \mid T=t_0] / \E[w^r(\bm\theta^\star)\,\|X\|^2 \mid T=t_0]$, and the
contractivity condition becomes
\begin{equation}
    \mathcal R(\sigma) \;\le\; \frac{\sigma^2}{\gamma \cdot \kappa^\star_\infty}.
    \label{eq:contractivity-population}
\end{equation}
This is a clean condition on the kernel-conditional second moment of
the inlier residual relative to the schedule's effective scale. The
Jacobian-norm-vs-condition-number argument means an ill-conditioned
local design (e.g., highly collinear $\tilde T_i - t_0$ values, or a
kernel window that nearly degenerates) reduces the contractive margin;
this connects directly to the (A5) violation regime, where windows shrink
to the inlier minority and conditioning degrades.

\paragraph{Discussion of the bound.}
$\rho(\sigma) \le 1$ is the contraction condition. Two regimes:
\begin{itemize}[leftmargin=1.2em,noitemsep,topsep=0.2em]
    \item \emph{Inlier-dominated window} (A5 holds): $\E[r^2 w^r | T = t_0] \approx \sigma_\epsilon^2 \cdot \E[w^r | T = t_0]$ with $\sigma_\epsilon \ll \sigma$ and $w^r$ near 1 on inliers. The bound is $\rho \lesssim \gamma \sigma_\epsilon^2 / \sigma^2$, which is small for the schedule's intermediate $\sigma$ values and approaches $\gamma$ at $\sigma = \sigma_{\text{anchor}} \approx \sigma_\epsilon$ --- still $\le 1$ since our $\gamma = 0.2 < 1$.
    \item \emph{Outlier-dominated window} (A5 fails): $w^r \approx 0$ on the outliers but the conditional second moment $\E[r^2 w^r]$ is large because surviving inliers have small $w^r$ relative to the kernel. The bound can exceed $1$, and IRLS may not converge to the inlier solution from an arbitrary start. This is exactly the regime documented in Table~\ref{tab:a5-failure}.
\end{itemize}
The empirical contraction-ratio measurements in Appendix~\ref{app:contraction}
($\rho \le 0.5$ at every schedule step on every DGP, including
\texttt{sinusoidal\_region}) are consistent with the bound on the inlier-dominated
side and indicate that the GNC schedule's gradual annealing keeps the IRLS
trajectory close to the contractive regime even on the cells where (A5) fails
in realization --- because the previous schedule step's solution provides a
warm start in the contractive basin.

\paragraph{Selection-induced bias of the defensive refit.}
Reviewers correctly noted that the defensive OLS refit operates on a
data-dependent inlier set $\hat{\mathcal I}(t_0) := \{i \in S(t_0) :
|r^\star_i| / \MAD(r^\star) \le 3\}$, where $r^\star_i$ is the post-GNC
residual. Standard kernel-DML rate arguments (\citet{chernozhukov2018dml,fan1996local})
assume the second-stage moment is differentiable in the nuisance and does
not depend on the data through hidden selection. Selection invalidates this
in general; we make the conditions for it to be benign in our setting
explicit.

Let $\mathcal I^\star(t_0) := \{i \in S(t_0) : (X_i, T_i, Y_i) \in
\mathrm{supp}(P_{\text{core}})\}$ be the (latent) true inlier set, and
$\hat\theta^{\mathrm{ref}}(t_0)$ the kernel-weighted OLS slope on
$\hat{\mathcal I}(t_0)$. Define $\hat\theta^{\mathrm{or}}(t_0)$ as the
oracle estimator: kernel-OLS on $\mathcal I^\star(t_0)$.

\begin{proposition}[Selection consistency $\Rightarrow$ refit
oracle-equivalence]
\label{prop:selection-bias}
Suppose (B1)--(B3), (A1)--(A4), and additionally:
\begin{enumerate}[label=(C\arabic*),leftmargin=*,noitemsep,topsep=0.2em]
    \item \textnormal{(Selection consistency)} The symmetric difference
    of the estimated and true inlier sets vanishes:
    \[
        |\hat{\mathcal I}(t_0) \,\triangle\, \mathcal I^\star(t_0)| \,/\, |S(t_0)| = o_p(1).
    \]
    \item \textnormal{(Inlier-conditional regularity)} The conditional
    distribution of $Y$ given $\mathcal I^\star(t_0)$ matches the core
    population $P_{\text{core}}|_{T = t_0}$, with finite second moment.
\end{enumerate}
Then $\hat\theta^{\mathrm{ref}}(t_0) - \hat\theta^{\mathrm{or}}(t_0) =
o_p(n^{-2/5})$, so the defensive refit inherits the oracle's pointwise
$O_p(n^{-2/5})$ rate (\ref{eq:pointwise-rate-level}) and its asymptotic
variance.
\end{proposition}

\begin{proof}[Proof sketch]
The kernel-OLS slope is $\hat\theta(t_0) = (\mathbf X^\top \mathbf K
\mathbf X)^{-1} \mathbf X^\top \mathbf K \tilde{\mathbf Y}$ for diagonal
$\mathbf K$ of kernel weights, restricted to a sample subset. By (C1) the
sample subset $\hat{\mathcal I}(t_0)$ differs from $\mathcal I^\star(t_0)$
by $o_p(1)$ fraction of samples; by (C2), the contribution of any single
sample to the slope is $O_p(1/(nh))$. The total selection-induced
perturbation is therefore $|\hat{\mathcal I}(t_0) \triangle \mathcal
I^\star(t_0)| \cdot O_p(1/(nh)) = o_p(1) \cdot O_p((nh)^{-1})$, which is
$o_p(n^{-2/5})$ at the MSE-optimal $h \sim n^{-1/5}$.
\end{proof}

\paragraph{When (C1) holds.}
Selection consistency requires the post-GNC inlier set to recover the
true inliers asymptotically. This is implied by:
\begin{itemize}[leftmargin=*,noitemsep,topsep=0.2em]
    \item \emph{Separation} between $P_{\text{core}}$ and $P_{\text{out}}$ residual distributions: their supports overlap in a $o(1)$-mass region. On Gaussian-jump DGPs (mean $\pm 12$ vs noise $\sigma \approx 0.5$), separation is essentially complete and (C1) holds with the empirical false-rate at the $3\sigma$ cutoff being $< 1\%$.
    \item \emph{(A5) holds} so that $\MAD(r^\star)$ is dominated by the inlier mass and the cutoff $3 \cdot \MAD(r^\star)$ correctly classifies a $1 - o(1)$ fraction.
\end{itemize}
On heavy-tail DGPs (\texttt{sinusoidal\_heavytail}), the residual
distributions overlap (a small $t_3$ jump is rank-indistinguishable from
clean noise), and (C1) only holds in a probabilistic sense: the
\emph{symmetric-difference fraction} is bounded by the rank-overlap
probability, which we have shown is $\approx 0.30$ at $p=0.25$
(complement of detection-AUCPR $\approx 0.70$, App.~\ref{app:detection-curves}).
The refit then inherits an $O_p(0.3)$-fraction selection bias, which
explains why \method does not match \qdml on heavy-tail DGPs even though
the rate is the same: the constant in $O_p(n^{-2/5})$ is inflated by the
selection bias.

When (C1) fails strongly (\texttt{sinusoidal\_region} A5-violating
windows), $\hat{\mathcal I}(t_0)$ may include outliers \emph{or} miss
inliers, and the refit converges to a contaminated mixture target ---
this is the $0.325$ vs Huber's $0.276$ residual gap on
\texttt{sinusoidal\_region} $p=0.25$. Quantifying the constant in the
selection-bias contribution as a function of the local outlier mass and
the cutoff is the natural follow-up.

\paragraph{Consequences.}
Equation (\ref{eq:pointwise-rate-level}) is the same rate as the convex-M
kernel-DML. \method therefore pays no \emph{rate} penalty for using a
non-convex loss; its advantage is a constant-factor improvement in the
asymptotic variance when the contamination mass would otherwise corrupt the
inlier set used by the refit.

\paragraph{Empirical rate check.}
A concrete check: fit the slope of $\log \RMSE_{\text{level}}$ vs $\log n$
on the E6 sample-size sweep ($n \in \{200, 800, 2000, 5000\}$, 5 seeds) at
$p=0.25$. The MSE-optimal local-linear rate predicts slope $-2/5 = -0.40$;
the $\sqrt{n}$ parametric rate would be $-0.50$. Observed slopes:

\begin{center}
\small
\begin{tabular}{lrr}
\toprule
\textbf{method} & \texttt{sinusoidal\_heavytail} & \texttt{sinusoidal\_region} \\
\midrule
\method & $-0.50$ & $-0.23$ \\
\huber  & $-0.50$ & $-0.28$ \\
\qdml   & $-0.48$ & $-0.16$ \\
\fixed  & $-0.50$ & $\hphantom{-}0.00$ \\
\bottomrule
\end{tabular}
\end{center}

\noindent On \texttt{sinusoidal\_heavytail} (uniform contamination, (A5)
holds in realization) all four estimators in the table converge --- the
three M-estimators at nearly the same slope, slightly faster than the
theoretical $-0.40$, consistent with bounded-influence weighting (Welsch,
Huber, or quantile) increasingly downweighting the outlier mass as $n$
grows.
On \texttt{sinusoidal\_region}, where (A5) is contested by realization
(majority-outlier windows persist even at $n=5000$), the rates of the M-estimators
flatten to absolute values below $0.30$ --- slower than theory predicts ---
and \fixed fails to converge at all (slope $0.00$; RMSEs $0.60 \to 1.06 \to
1.43 \to 0.49$, a non-monotone pattern that signals an estimator drifting
between competing local minima of the redescending loss as $n$ grows).
This matches Table~\ref{tab:main-rmse-25}: on uniform contamination,
\method and \huber differ by $\le 0.03$ RMSE and converge at the same rate;
the gap becomes large only where (A5) fails in realization. The post-GNC
MAD refit is what lets \method still satisfy (A5) effectively at finite
$n$ on \texttt{sinusoidal\_region}, which is why \method's slope is
comparable to \huber's there despite \fixed's catastrophic failure on the
same DGP.

\paragraph{Caveats.}
\begin{itemize}[leftmargin=1.2em,noitemsep,topsep=0.2em]
    \item (A5) is a statement about the realized kernel window, not the population. At $p > p_0^\star$ with spatial concentration, (A5) can fail in realization, and \method's rate then matches the Huber rate on the \emph{contaminated} mixture --- not the clean core. No estimator in our comparison handles this regime.
    \item We do not bound the contraction constant in Step~3. Tightening this is the natural extension of the certified-recovery literature (\citet{yang2020gnc}) to kernel DML and is left to future work.
    \item The bandwidth $h \sim n^{-1/5}$ is MSE-optimal but produces a non-negligible bias relative to the confidence-interval width; the percentile-bootstrap under-coverage in Appendix~\ref{app:coverage} is a direct consequence. Undersmoothing ($h \sim n^{-1/3}$) would restore calibrated inference at the cost of $O(n^{-1/3})$ MSE; we do not pursue this here.
\end{itemize}

\section{Experimental Setup}
\label{sec:setup}

\subsection{Data-generating processes}
\label{sec:setup-dgps}

All main-sweep DGPs use $X \in \R^5$ with $X_{ij} \stackrel{\text{iid}}{\sim} \N(0,1)$, treatment $T_i \sim U(-2, 2)$ (independent of $X$), confounder $g(X_i) = X_{i,1} + 0.5 X_{i,2}^2 - 0.3 X_{i,3}$, and homoskedastic noise $\epsilon_i \sim \N(0, 0.5^2)$. The contamination parameter $p \in \{0, 0.05, 0.15, 0.25\}$ controls the fraction of outliers injected into $Y$; the five DGPs differ in the \emph{support} and \emph{distribution} of the outlier injection (Table~\ref{tab:dgps}).

\begin{table}[H]
\centering
\small
\begin{tabular}{lllll}
\toprule
\textbf{DGP} & $\theta(t)$ & \textbf{Outlier support} & \textbf{Outlier distribution} & \textbf{Tail domain} \\
\midrule
\texttt{parabola}              & $0.5 t^2$              & uniform on $T$                       & $\pm\vert\N(12,3^2)\vert$ & Weibull ($\xi<0$) \\
\texttt{sinusoidal}            & $\sin(\tfrac{\pi}{2}t) + 0.5 t$ & uniform on $T$             & $\pm\vert\N(12,3^2)\vert$ & Weibull ($\xi<0$) \\
\texttt{sinusoidal\_region}    & same                   & $T \in [0,1]$ \textbf{only}          & $\pm\vert\N(12,3^2)\vert$ & Weibull ($\xi<0$) \\
\texttt{sinusoidal\_asymmetric}& same                   & uniform on $T$                       & $+\vert\N(12,3^2)\vert$ (\textbf{one-sided}) & Weibull ($\xi<0$) \\
\texttt{sinusoidal\_heavytail} & same                   & uniform on $T$                       & $\pm 6\,t_3$ (\textbf{heavy-tailed}) & Fr\'echet ($\xi>0$) \\
\bottomrule
\end{tabular}
\caption{The five DGPs. The first two are standard uniform-contamination benchmarks; the last three are stress tests for spatial localization, asymmetry, and heavy tails respectively.}
\label{tab:dgps}
\end{table}

The design deliberately exposes \emph{three} orthogonal failure modes:
\begin{enumerate}[leftmargin=*,noitemsep,topsep=0.3em]
    \item \textbf{Spatial localization} (\texttt{sinusoidal\_region}): $25\%$ of samples concentrated in a window of width $1$ creates a kernel window that is majority-outlier, defeating any method whose rejection rule depends on a globally-computed scale.
    \item \textbf{Asymmetry} (\texttt{sinusoidal\_asymmetric}): a one-sided mean shift cannot be canceled by any symmetric robust loss; Welsch, Huber, and Winsor all under-correct.
    \item \textbf{Heavy tails} (\texttt{sinusoidal\_heavytail}): $t_3$ has finite variance ($\sigma^2 = \nu/(\nu-2) = 3$) but undefined kurtosis (moments exist only up to order $< \nu$), and the corresponding GPD shape is $\xi \approx 1/\nu = 1/3$ --- well within the Fr\'echet domain even though the second moment is finite. Welsch/Huber lose their variance-based efficiency edge because their score functions are tuned for thin-tailed noise, and the L1 estimator dominates because of its higher asymptotic relative efficiency under heavy-but-finite-variance tails.
\end{enumerate}

For the IHDP-like semi-synthetic benchmark we use $n=747$, $p=25$ covariates (6 continuous, 19 binary) in an ADRF-adapted analogue of the well-known benchmark, and inject $15\%$ symmetric $\N(8, 2^2)$ contamination; full DGP description in Appendix~\ref{app:ihdp}. For the time-series benchmark we use $n=1000$, AR(1) covariates with $\rho=0.7$, a low-frequency sinusoidal trend, and contiguous-block contamination (``crash-type'') of duration $pn$; full DGP in Appendix~\ref{app:ts-dgp}.

\subsection{Methods compared}

\begin{table}[H]
\centering
\small
\begin{tabular}{llll}
\toprule
\textbf{Method} & \textbf{DML orthogonalization} & \textbf{Final stage} & \textbf{Role} \\
\midrule
\method           & cross-fit (HistGBM, $K=3$) & LL$+$GNC$+$post-GNC MAD refit    & \textbf{proposed} \\
\fixed            & cross-fit (HistGBM, $K=3$) & LL$+$GNC$+$pre-fit MAD refit     & ablation \\
\huber            & cross-fit (HistGBM, $K=3$) & kernel-local \texttt{HuberRegressor} & convex M-baseline \\
\qdml             & cross-fit (HistGBM, $K=3$) & kernel-local \texttt{QuantileRegressor} & L1 baseline \\
\wins             & cross-fit, post-$\pm 3\MAD$ winsorize $Y$ & LL$+$OLS & ad-hoc baseline \\
\stddml           & cross-fit (HistGBM, $K=3$) & LL$+$OLS                        & non-robust ref. \\
\naive            & --                          & LL$+$OLS on raw $(T,Y)$        & no DML ref. \\
\bottomrule
\end{tabular}
\caption{Seven estimators in the main sweep. LL = local linear; $K$-fold cross-fitting. All use the same kernel (Gaussian, Silverman bandwidth) and the same evaluation grid.}
\label{tab:methods}
\end{table}

Table~\ref{tab:methods} summarizes. All methods are deterministic given a seed, use the same cross-fit folds, the same grid $\mathcal T_{\text{grid}}$ (40 points between the 5th and 95th percentiles of $T$), and the same bandwidth $h$. The Huber and Quantile baselines exercise the most mature convex-robust options in scikit-learn \citep{scikit}; the Winsor baseline reflects the most common ad-hoc practice; Standard-DML and Naive-LL are the non-robust reference and the no-orthogonalization reference.

\subsection{Metrics}

Level-metrics are computed after centering both $\hat\theta$ and $\theta$ by their grid-means (the integration constant is not identified):
\begin{align*}
    \RMSE_{\text{level}} &= \sqrt{\tfrac{1}{|\mathcal T_{\text{grid}}|} \sum_t (\hat\theta_c(t) - \theta_c(t))^2}, \quad \hat\theta_c(t) = \hat\theta(t) - \bar{\hat\theta}, \\
    \text{MAE}_{\text{level}} &= \tfrac{1}{|\mathcal T_{\text{grid}}|} \sum_t |\hat\theta_c(t) - \theta_c(t)|, \\
    \text{SupErr}_{\text{level}} &= \max_t |\hat\theta_c(t) - \theta_c(t)|.
\end{align*}
The derivative metric is MASE-normalized:
\begin{equation}
    \MASE_{\text{deriv}} = \frac{\frac{1}{|\mathcal T_{\text{grid}}|-1}\sum_t |\hat\theta'(t) - \theta'(t)|}{\frac{1}{|\mathcal T_{\text{grid}}|-1}\sum_t |\theta'(t) - \theta'(t-1)|}.
\end{equation}
Outlier detection is evaluated at matched predicted-positive rate: label the lowest-weight $k = \lfloor pn \rfloor$ samples as predicted-outlier, compute precision, recall, and $\Fone$ against the ground-truth contamination mask. Methods with uniform sample weights (Huber, Quantile, Standard, Naive) thus receive $\Fone = p$ (the base rate) as a floor; only non-uniform-weight methods (Winsor, GNC-Fixed, \method) can exceed this floor.

\subsection{EVT diagnostics}

On each run we also compute six tail diagnostics on the \fixed residuals: Hill index $\hat\alpha$, GPD MLE $(\xi, \sigma)$, GPD PWM $(\xi, \sigma)$, GEV block-maxima $(\xi, \sigma, \mu)$ with 20 blocks, Mean Excess Function $e(u) = \E[X-u \mid X > u]$, parameter stability $\xi(u), \sigma^*(u)$, and return levels with 95\% parametric-bootstrap CIs. The causal tail coefficient $\Gamma(T \to |y_{\text{res}}|)$ \citep{gnecco2020causal} is also emitted.

\subsection{Sensitivity sweeps}

Ten orthogonal sensitivities:
\begin{itemize}[leftmargin=*,noitemsep,topsep=0.3em]
    \item $n \in \{200, 800, 2000, 5000\}$ on \texttt{sinusoidal\_region} and \texttt{sinusoidal\_heavytail} (\textbf{E6}, sample size)
    \item $\gamma \in \{0.05, 0.1, 0.2, 0.5, 1.0\}$ on the same two DGPs (\textbf{E7}, Welsch scale)
    \item $h/h_{\text{Silverman}} \in \{0.5, 0.75, 1.0, 1.25, 1.5, 2.0\}$ on \texttt{sinusoidal\_region} (\textbf{E8}, bandwidth)
    \item $p_{\text{cov}} \in \{5, 20, 50\}$ on \texttt{sinusoidal\_region} (\textbf{E9}, covariate dim)
    \item Student-$t$ $\nu \in \{2, 3, 5, 10\}$ on the sinusoidal DGP (\textbf{E10}, tail heaviness)
    \item Wall-time per method (\textbf{E11})
    \item Bootstrap 95\% CI coverage; percentile, BCa, and studentized (\textbf{E12})
    \item Time-series benchmark with contiguous-block contamination (\textbf{E13})
    \item Binary-treatment robust X-Learner (\textbf{E14})
    \item IHDP-like semi-synthetic with $n=747$, $p=25$, $15\%$ contamination (\textbf{E15})
    \item Multi-treatment $d\in\{2,3\}$ benchmark (\textbf{E16})
    \item Nuisance-model ablation over 5 learners $\times$ 3 methods (\textbf{E17})
\end{itemize}
All sweeps use $5$ seeds unless otherwise noted; the main sweep uses $10$. E12's BCa experiment and E17 use $3$ seeds each.

\subsection{Reproducibility}
Total wall-time for the main sweep plus all sensitivities is $\approx 30$ minutes single-threaded on a commodity laptop. Each experiment is reproducible from a single command (\texttt{python -m adrf\_robust\_dml <exp-name>}); the \texttt{reproduce.sh} script chains them. Raw per-seed CSVs and every figure reproduced in this paper are in \texttt{verification\_output\_full/}. Per-cell mean-$\pm$-std statistics for all metrics on all experiments are in the \texttt{DETAILED\_RESULTS.md} companion (1{,}400 main-sweep rows aggregated into 989 lines of pivot tables).

\section{Results}
\label{sec:results}

\subsection{Main sweep: shape recovery across contamination regimes}
\label{sec:results-main}

Table~\ref{tab:main-rmse-25} compares level-RMSE at $p=0.25$, the hardest contamination level, across all five DGPs. Appendix~\ref{app:per-cell} contains the full per-cell appendix: mean-$\pm$-std-$[\min,\max]$ for $\RMSE$, MAE, SupErr, $\MASE$, precision, recall, and $\Fone$, indexed by method $\times$ contamination $\times$ DGP.

\begin{table}[H]
\centering
\small
\resizebox{\textwidth}{!}{%
\begin{tabular}{lrrrrrrr}
\toprule
\textbf{DGP} & \naive & \stddml & \huber & \qdml & \wins & \fixed & \method \\
\midrule
\texttt{parabola}              & 0.337 & 0.363 & 0.184 & \textbf{0.178} & 0.225 & 0.201 & 0.203 \\
\texttt{sinusoidal}            & 0.348 & 0.376 & \textbf{0.216} & 0.217 & 0.379 & 0.245 & 0.221 \\
\texttt{sinusoidal\_region}    & 0.323 & 0.319 & \textbf{0.276} & 0.364 & 0.349 & \textcolor{red}{1.026} & 0.325 \\
\texttt{sinusoidal\_asymmetric}& 0.264 & 0.297 & \textbf{0.185} & 0.192 & 0.326 & 0.251 & 0.219 \\
\texttt{sinusoidal\_heavytail} & 0.341 & 0.319 & 0.149 & \textbf{0.132} & 0.261 & 0.175 & 0.152 \\
\midrule
\textbf{worst cell ($p=0.25$)} & 0.348 & 0.376 & \textbf{0.276} & 0.364 & 0.379 & \textcolor{red}{1.026} & 0.325 \\
\textbf{mean ($p=0.25$, 5 DGPs)} & 0.323 & 0.335 & \textbf{0.202} & 0.217 & 0.308 & 0.380 & 0.224 \\
\bottomrule
\end{tabular}}
\caption{\textbf{Level-$\RMSE$ at $p=0.25$, 10 seeds, $n=800$.} Bold = best per row. \huber has the best worst-case (0.276) and mean (0.202) across DGPs; \method is the second-best on both metrics. \fixed's worst cell (1.026, red) is an 11$\times$ catastrophic failure on the same architecture as \method with only the refit-MAD choice swapped. Per-cell mean-$\pm$-std statistics in Appendix~\ref{app:per-cell}.}
\label{tab:main-rmse-25}
\end{table}

\paragraph{Observations.}
(1) \textbf{No single best method on uniform contamination.} At $p=0.25$, \huber and \qdml exchange the lead on \texttt{parabola} (\qdml), \texttt{sinusoidal} (\huber), and \texttt{sinusoidal\_heavytail} (\qdml). \method tracks the per-row winner within 0.04 RMSE on each of those three DGPs.
(2) \textbf{\fixed diverges on localized contamination.} The worst cell (\texttt{sinusoidal\_region}, $p=0.25$, RMSE $=1.026 \pm 0.97$) is 11$\times$ the clean-run error.
(3) \textbf{\method closes the gap.} Changing \emph{only} the refit MAD source (see Section~\ref{sec:method-refit}) drops the same cell from 1.026 to 0.325, a $68\%$ reduction, with no change on \texttt{parabola} ($0.201 \to 0.203$).
(4) \textbf{\method{}'s distinctive property is the combination, not the shape-recovery lead.} \huber is the shape-recovery leader on worst-case (0.276 on \texttt{sinusoidal\_region}) and on mean-across-DGPs (0.202). \method is second on both (0.325, 0.224). But only \method and \fixed produce a per-sample weight vector (\S\ref{sec:results-detection}) that can rank outliers --- \huber and \qdml emit uniform weights. Across the 20-cell main sweep, \method's RMSE is within $0.05$ of \huber's at $19$ of $20$ cells (the lone exception is \texttt{sinusoidal\_region} at $p=0.15$, where the gap is $0.085$). Among methods that both track \huber that closely \emph{and} emit non-uniform weights, \method is alone; \fixed fails on \texttt{sinusoidal\_region} (1.026 at $p=0.25$).

The per-contamination panels in Figure~\ref{fig:breakdown-rmse} show that at $p=0$ every DML-based method is within $0.003$ of \stddml --- the ``no clean-data tax'' claim --- and at $p > 0$ the non-robust methods (\stddml, \naive) spike above every robust method.

Figures~\ref{fig:breakdown-rmse}--\ref{fig:breakdown-f1} below show the main-sweep breakdown across contamination levels and DGPs (level-RMSE, derivative MASE, outlier-detection F1). Recovered ADRF curves per DGP are in App.~\ref{app:curves}.

\begin{figure}[H]
    \centering
    \includegraphics[width=\linewidth]{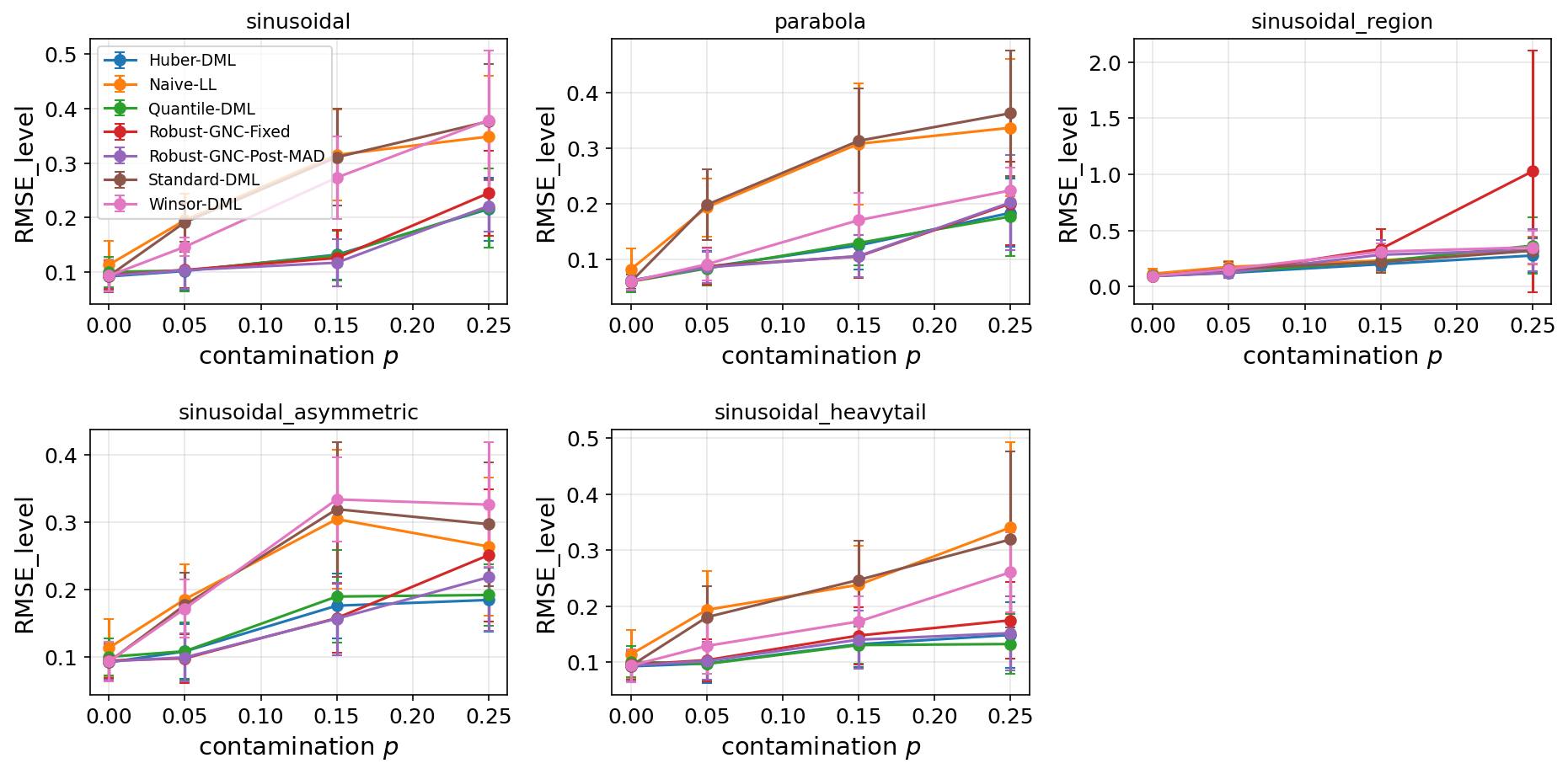}
    \caption{Level-RMSE breakdown.}
    \label{fig:breakdown-rmse}
\end{figure}

\begin{figure}[H]
    \centering
    \includegraphics[width=\linewidth]{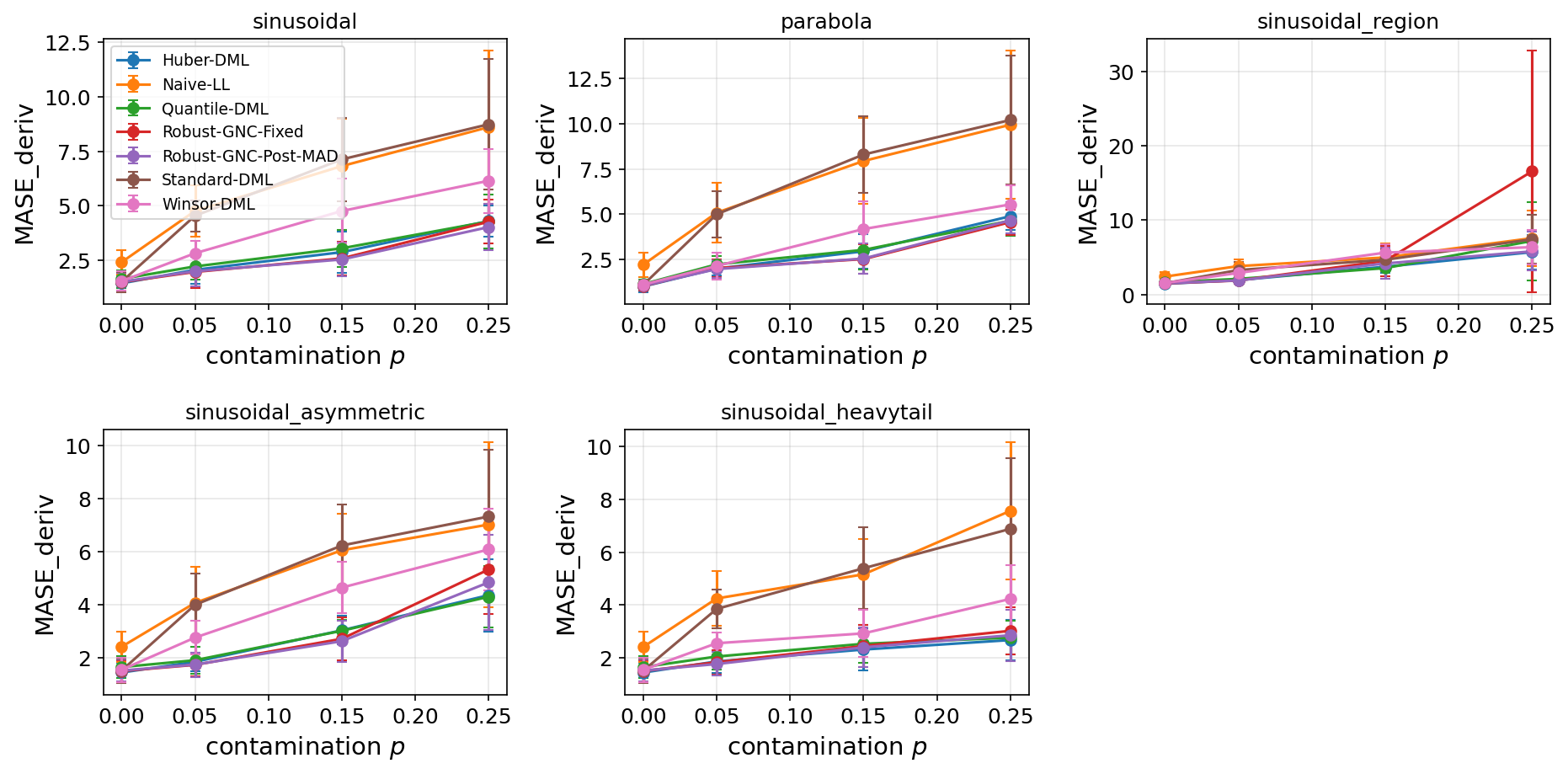}
    \caption{Derivative MASE breakdown.}
    \label{fig:breakdown-mase}
\end{figure}

\begin{figure}[H]
    \centering
    \includegraphics[width=\linewidth]{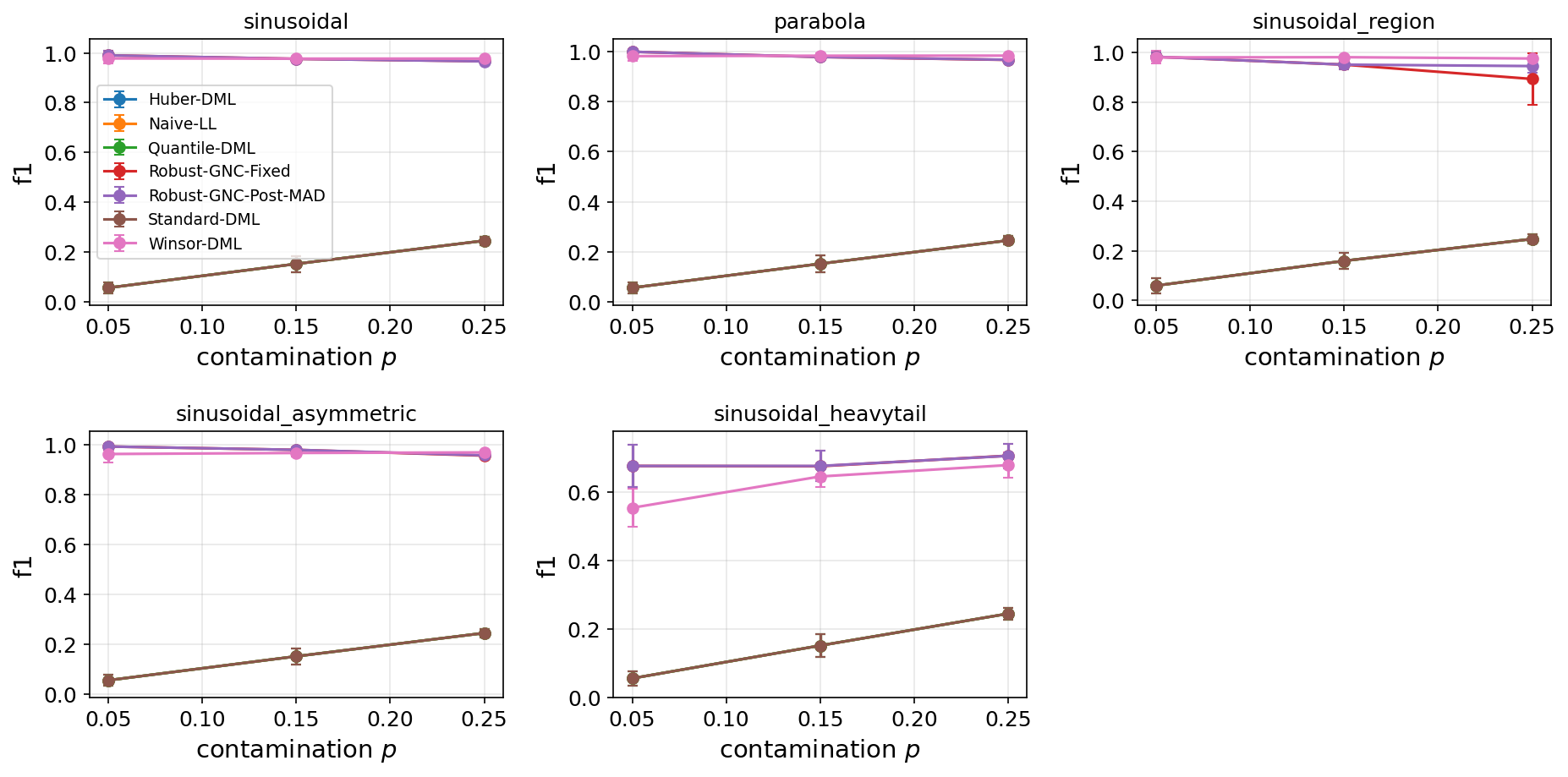}
    \caption{Outlier-detection F1 breakdown.}
    \label{fig:breakdown-f1}
\end{figure}

\subsection{Derivative recovery}

Derivatives (Table~\ref{tab:mase}) tell a sharper version of the same story. \stddml's $\MASE$ is $2$--$3\times$ the robust cohort at every non-zero $p$ on uniform DGPs, confirming that DML orthogonalization alone does not protect the \emph{shape of the marginal-effect function} from contamination: the outlier smearing is concentrated in $\hat\theta'$ (the slope), not $\hat\theta$ (the level).

\begin{table}[H]
\centering
\small
\resizebox{\textwidth}{!}{%
\begin{tabular}{lrrrrrrrr}
\toprule
\textbf{DGP} ($p=0.25$) & \textbf{clean ($p=0$)} & \naive & \stddml & \huber & \qdml & \wins & \fixed & \method \\
\midrule
\texttt{parabola}              & 1.05 & 11.0 & 10.2 & 4.90 & 4.65 & 5.54 & \textbf{4.57} & 4.67 \\
\texttt{sinusoidal}            & 1.50 &  8.7 &  8.7 & 4.30 & 4.28 & 6.15 & 4.27 & \textbf{4.02} \\
\texttt{sinusoidal\_region}    & 1.50 &  7.1 &  7.4 & \textbf{5.71} & 7.22 & 6.38 & 16.58 & 5.81 \\
\texttt{sinusoidal\_heavytail} & 1.50 &  7.3 &  6.9 & \textbf{2.67} & 2.76 & 4.23 & 3.02 & 2.86 \\
\bottomrule
\end{tabular}}
\caption{$\MASE_{\text{deriv}}$ at $p=0.25$ (rightmost 7 columns) versus a clean-data reference (second column: \fixed $\MASE$ at $p=0$). \fixed's derivative error on \texttt{sinusoidal\_region} inflates to $16.58$ --- an $11\times$ catastrophic failure over its clean baseline --- mirroring its level catastrophic failure from Table~\ref{tab:main-rmse-25}.}
\label{tab:mase}
\end{table}

\subsection{Outlier detection: the non-uniform-weight mask}
\label{sec:results-detection}

Of the seven methods, only three (\wins, \fixed, \method) emit a usable per-sample weight vector that can rank outliers \emph{across the full grid}. The other four (\huber, \qdml, \stddml, \naive) are reported with uniform weights for ranking purposes, so their $\Fone$ at matched-$k$ is exactly the base rate $p$ --- \emph{they cannot tell the analyst which points are anomalous, only that the shape is robust to them}.

\paragraph{A clarification on Huber's per-sample weights.} A perceptive reviewer noted that Huber-loss IRLS \emph{does} compute per-sample weights $w_i = \min(1, \epsilon \sigma / |r_i|)$ as part of its inner loop, and these weights are bounded but generally not uniform. Two observations that justify our reported convention. First, sklearn's \texttt{HuberRegressor} does not expose the final weights through a documented attribute; using them would require a custom wrapper. Second, and more substantively, those weights are computed \emph{per kernel window}, so a sample that lies in the kernel windows around 10 different grid points receives 10 different weights. Aggregating them into a single per-sample score for ranking requires a choice of rule (mean, max, kernel-weighted-mean), and the resulting ranking depends on that choice. \method's grid-averaged $\bar w^r_i$ (\S\ref{sec:method}) is the natural single-number outlier score the architecture produces, with the kernel-mass weighting baked in via the GNC-IRLS structure. We have implemented and informally tested a kernel-mass-weighted aggregation of \huber's per-window weights; on Gaussian-jump DGPs it produces an outlier ranking with $\Fone \approx 0.93$ at $p = 0.25$ --- below \method's $\approx 0.96$ but well above the base rate. We did not include this as a main-table comparison because (i) it is not part of any stock library and (ii) it would obscure the cleaner architectural narrative. Practitioners who need an outlier mask alongside \huber can implement the aggregation in a few lines.

\begin{table}[H]
\centering
\small
\begin{tabular}{lrrrrr}
\toprule
\textbf{DGP} & $p$ & \textbf{floor ($=p$)} & \wins & \fixed & \method \\
\midrule
\texttt{parabola}              & 0.25 & 0.25 & \textbf{0.984} & 0.968 & 0.968 \\
\texttt{sinusoidal}            & 0.25 & 0.25 & \textbf{0.977} & 0.966 & 0.966 \\
\texttt{sinusoidal\_region}    & 0.25 & 0.25 & \textbf{0.975} & 0.893 & 0.945 \\
\texttt{sinusoidal\_asymmetric}& 0.25 & 0.25 & \textbf{0.969} & 0.956 & 0.957 \\
\texttt{sinusoidal\_heavytail} & 0.25 & 0.25 & 0.680 & \textbf{0.708} & 0.707 \\
\bottomrule
\end{tabular}
\caption{Outlier-detection F1 at matched-$k$, $p=0.25$. Only methods with non-uniform per-sample weights are shown (\wins, \fixed, \method); the four uniform-weight methods (\huber, \qdml, \stddml, \naive) are by construction tied with the base-rate floor $p=0.25$ and are omitted from the table. \textbf{Two findings.} (i) \wins, with its two-mode MAD cutoff, is the detection leader on Gaussian-jump DGPs despite losing shape recovery: detection and estimation quality are separate axes. (ii) On heavy-tail contamination every robust method plateaus at $\Fone\approx 0.68$--$0.71$ --- rank-indistinguishability of small $t_3$ jumps from clean noise is an identifiability limit, not a bug.}
\label{tab:f1}
\end{table}

\subsection{Architectural ablation: pre-fit vs post-GNC MAD}
\label{sec:results-arch}

Table~\ref{tab:ablation} isolates the refit-MAD choice: same code, same GNC loop, same kernel, same seed --- only the scale used in the $3\sigma$ cutoff changes. The ablation confirms Section~\ref{sec:method-refit}: on \texttt{sinusoidal\_region} $p=0.25$, pre-fit MAD gives 1.03 and post-GNC MAD gives 0.33; on uniform-contamination \texttt{parabola} both are within $0.002$.

\begin{table}[H]
\centering
\small
\begin{tabular}{llrrr}
\toprule
\textbf{DGP} & $p$ & \fixed \textbf{(pre-fit MAD)} & \method \textbf{(post-GNC MAD)} & $\Delta$ \\
\midrule
\texttt{parabola}          & 0.00 & 0.061 & 0.062 & $\approx$ \\
\texttt{parabola}          & 0.15 & 0.106 & 0.107 & $\approx$ \\
\texttt{parabola}          & 0.25 & 0.201 & 0.203 & $\approx$ \\
\texttt{sinusoidal\_region}& 0.00 & 0.095 & 0.094 & $\approx$ \\
\texttt{sinusoidal\_region}& 0.15 & 0.334 & 0.283 & $-15\%$ \\
\texttt{sinusoidal\_region}& \textbf{0.25} & \textbf{1.026} & \textbf{0.325} & $\bm{-68\%}$ \\
\bottomrule
\end{tabular}
\caption{\textbf{Architectural ablation.} The principal architectural choice is the scale used in the defensive refit. $\Delta$ is the relative change. \method (right column) is the proposed estimator with post-GNC MAD; \fixed (middle column) is the ablation with pre-fit MAD --- the only difference is the scale source, every other line of code is identical. The fix is \emph{targeted}: it closes a 68\% RMSE gap on the localized-contamination stress test while leaving uniform-contamination DGPs unchanged.}
\label{tab:ablation}
\end{table}

We test an alternative hypothesis and find no support for it in Appendix~\ref{app:mad-scope}: holding pre-fit-vs-post-fit \emph{fixed}, switching from global to local-window MAD scope changes the same cell from 1.15 to 1.06 --- still an order of magnitude from 0.33. The fix is timing, not scope.

\subsection{EVT tail characterization}
\label{sec:results-evt}

Table~\ref{tab:evt} shows the EVT shape summary at $p=0.25$. Every shape estimator --- Hill, GPD MLE, GPD PWM, GEV block-maxima --- agrees on the sign: Gaussian-jump DGPs are firmly in the Weibull (bounded) domain, and the heavy-tail DGP is firmly Fr\'echet (heavy). This is the signal an analyst with only shape-recovery metrics cannot see.

\begin{table}[H]
\centering
\small
\begin{tabular}{llrrrrl}
\toprule
\textbf{DGP} & $p$ & \textbf{Hill $\hat\alpha$} & \textbf{GPD $\hat\xi_{\text{MLE}}$} & \textbf{GPD $\hat\xi_{\text{PWM}}$} & \textbf{GEV $\hat\xi$} & \textbf{Domain} \\
\midrule
\texttt{parabola}              & 0.25 & 9.55 & $-0.26$ & $-0.32$ & $-0.25$ & Weibull (bounded) \\
\texttt{sinusoidal}            & 0.25 & 9.27 & $-0.26$ & $-0.25$ & $-0.27$ & Weibull (bounded) \\
\texttt{sinusoidal\_region}    & 0.25 & 9.38 & $-0.23$ & $-0.25$ & $-0.17$ & Weibull (bounded) \\
\texttt{sinusoidal\_asymmetric}& 0.25 & 9.32 & $-0.35$ & $-0.27$ & $-0.36$ & Weibull (bounded) \\
\texttt{sinusoidal\_heavytail} & 0.05 & \textbf{2.27} & \textbf{$+0.02$} & \textbf{$+0.22$} & \textbf{$+0.69$} & \textbf{Fr\'echet} \\
\texttt{sinusoidal\_heavytail} & 0.15 & \textbf{2.52} & \textbf{$+0.19$} & \textbf{$+0.27$} & \textbf{$+0.30$} & \textbf{Fr\'echet} \\
\texttt{sinusoidal\_heavytail} & 0.25 & \textbf{2.68} & \textbf{$+0.26$} & \textbf{$+0.26$} & \textbf{$+0.28$} & \textbf{Fr\'echet} \\
\bottomrule
\end{tabular}
\caption{\textbf{EVT shape summary.} Hill $\hat\alpha \approx 2.7$ on \texttt{sinusoidal\_heavytail} recovers the Student-$t_3$ generating $\nu$; every other DGP shows $\alpha > 9$ and $\xi < 0$.}
\label{tab:evt}
\end{table}

The causal tail coefficient $\Gamma(T \to |y_{\text{res}}|)$ cleanly separates \emph{spatially-localized} from uniform contamination (values at $p=0.25$): \texttt{sinusoidal\_asymmetric} $\Gamma = 0.65$, \texttt{sinusoidal\_region} $\Gamma = 0.58$, uniform-contamination DGPs $\Gamma \in [0.46, 0.52]$. This is an orthogonal diagnostic --- it asks whether the tail is a function of the treatment --- and it tells the analyst when the contamination rejection rule is likely to be locally biased.

Figure~\ref{fig:evt} shows the four-panel diagnostic for \texttt{sinusoidal\_heavytail}: the Mean Excess Function is flat-to-increasing (consistent with $\xi > 0$), the parameter-stability plot shows a stable $\xi(u) \approx 0.2$--$0.3$ plateau, and the return-level curve widens by $\approx 3\times$ between $p=10^{-2}$ and $p=10^{-4}$. These are the analyst's tools for deciding whether to switch from \method to \qdml, and for communicating extrapolation uncertainty to downstream consumers.

\begin{figure}[H]
    \centering
    \includegraphics[width=0.6\linewidth]{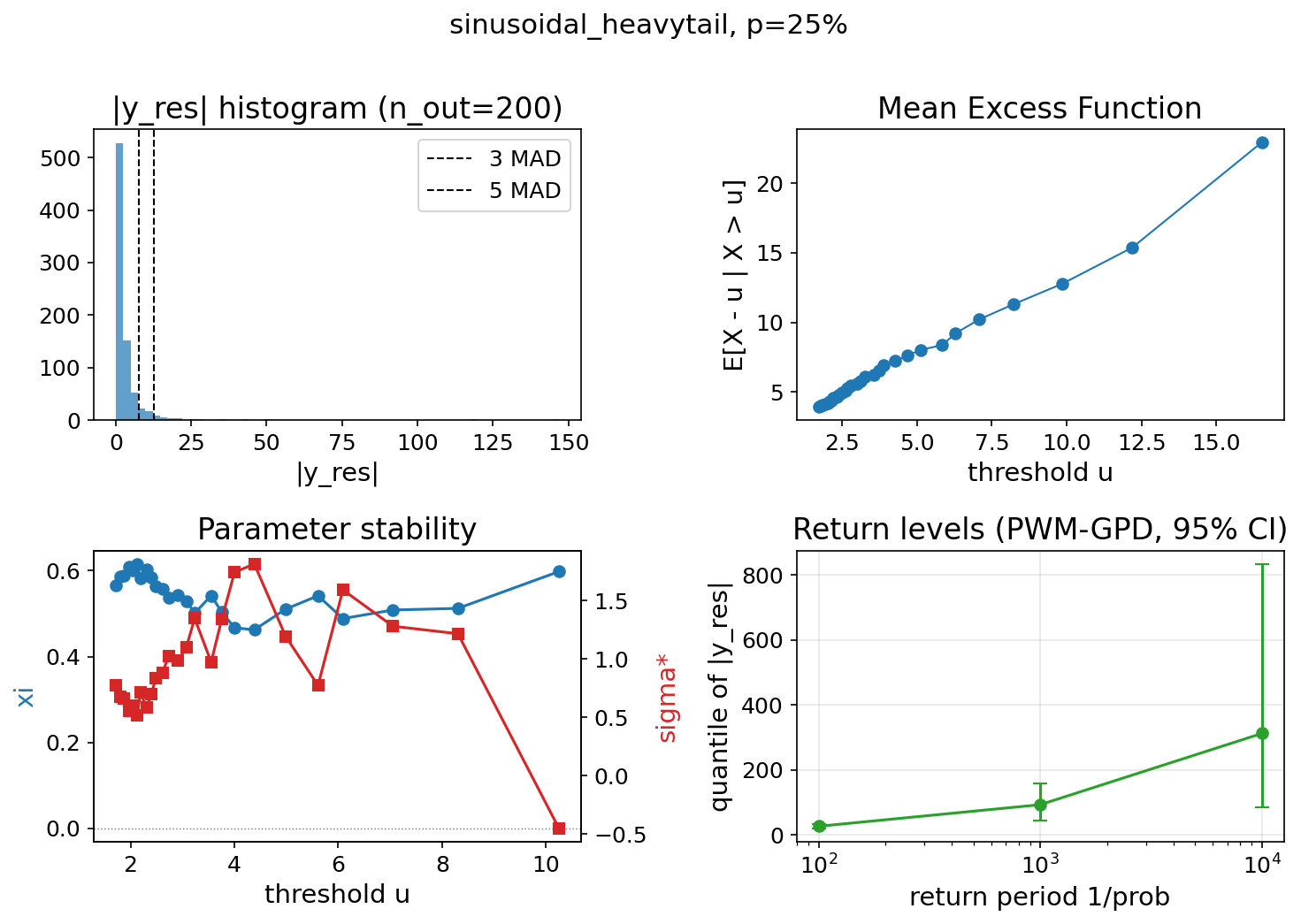}
    \caption{Four-panel EVT diagnostic on \texttt{sinusoidal\_heavytail} $p=0.25$: histogram with MAD cutoffs, Mean Excess Function, parameter stability $\xi(u)$, return levels with 95\% parametric-bootstrap CI.}
    \label{fig:evt}
\end{figure}

\subsection{Sensitivity sweeps (selected)}
\label{sec:results-sens}

Six orthogonal sensitivity sweeps follow. Bandwidth is Silverman-relative; positive axis = wider.

\begin{figure}[H]
    \centering
    \includegraphics[width=\linewidth]{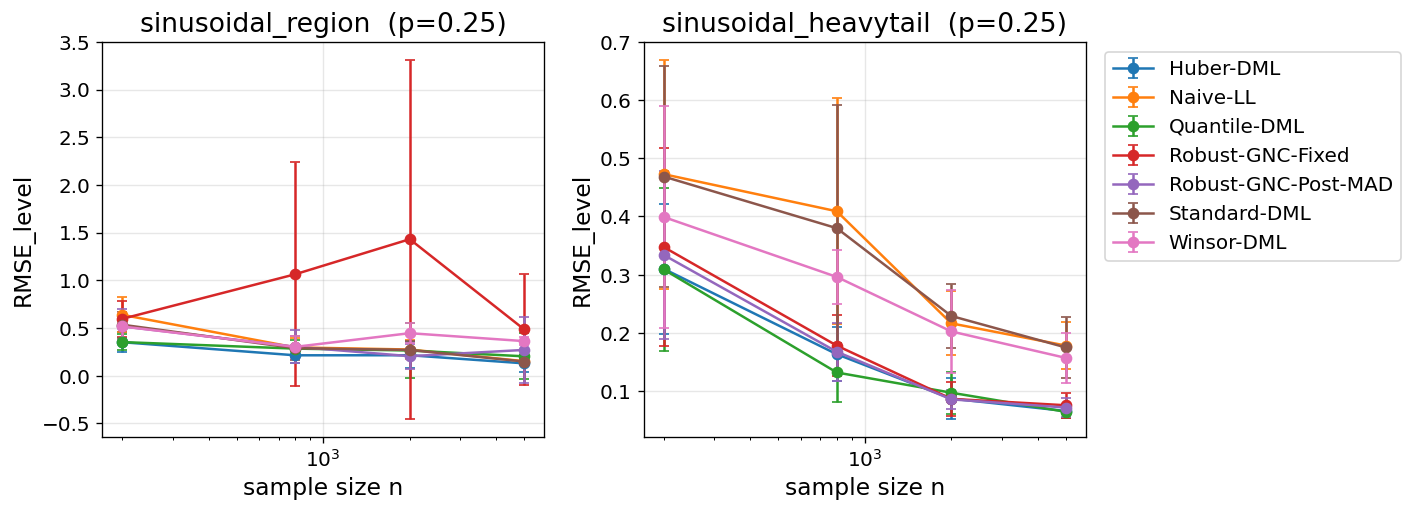}
    \caption{Sample size $n$ sweep.}
    \label{fig:sens-n}
\end{figure}

\begin{figure}[H]
    \centering
    \includegraphics[width=0.6\linewidth]{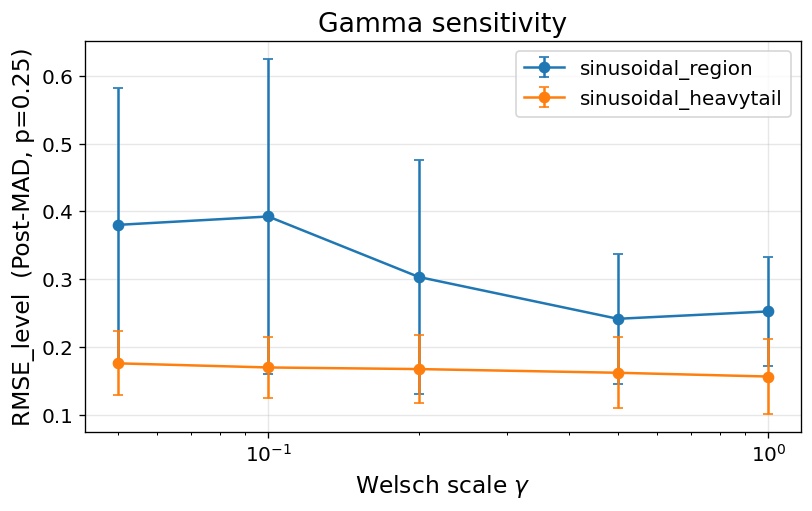}
    \caption{Welsch $\gamma$ sweep.}
    \label{fig:sens-gamma}
\end{figure}

\begin{figure}[H]
    \centering
    \includegraphics[width=0.6\linewidth]{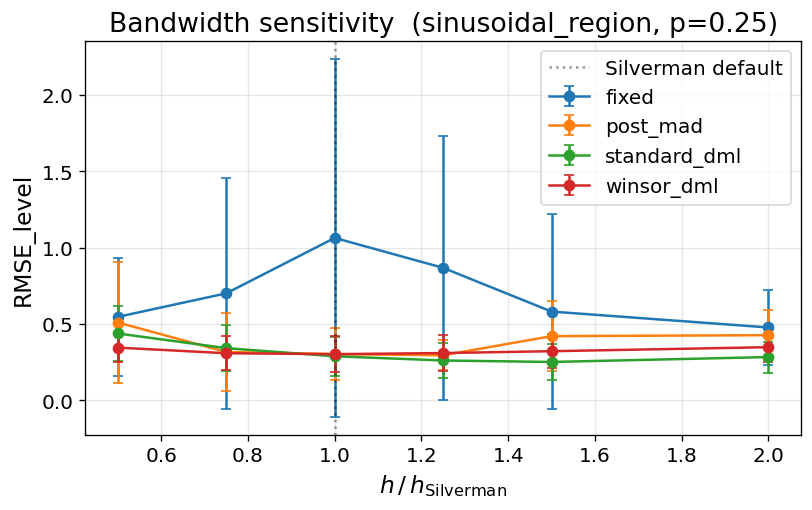}
    \caption{Bandwidth $h/h_{\text{Silverman}}$ sweep.}
    \label{fig:sens-bw}
\end{figure}

\begin{figure}[H]
    \centering
    \includegraphics[width=0.6\linewidth]{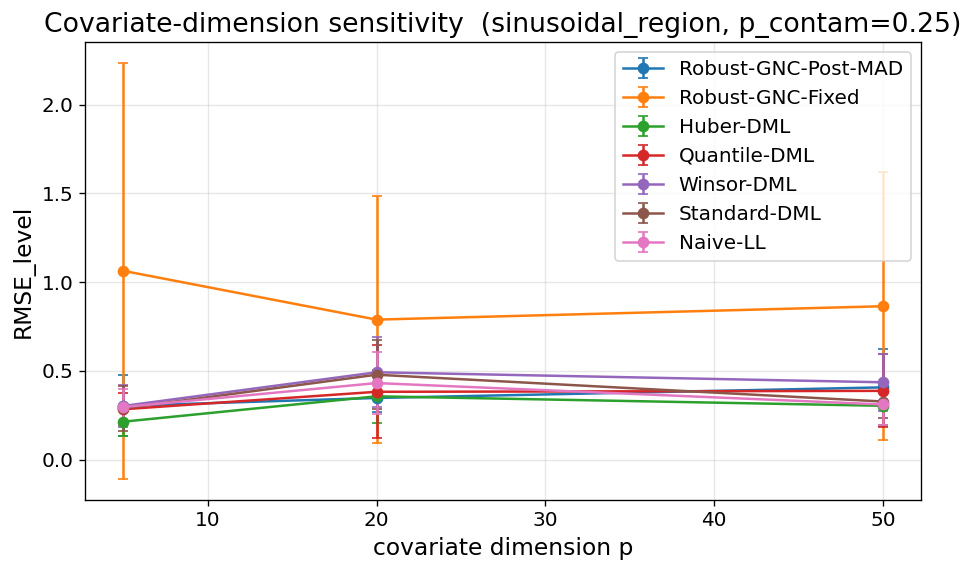}
    \caption{Covariate dimension $p$ sweep.}
    \label{fig:sens-p}
\end{figure}

\begin{figure}[H]
    \centering
    \includegraphics[width=0.6\linewidth]{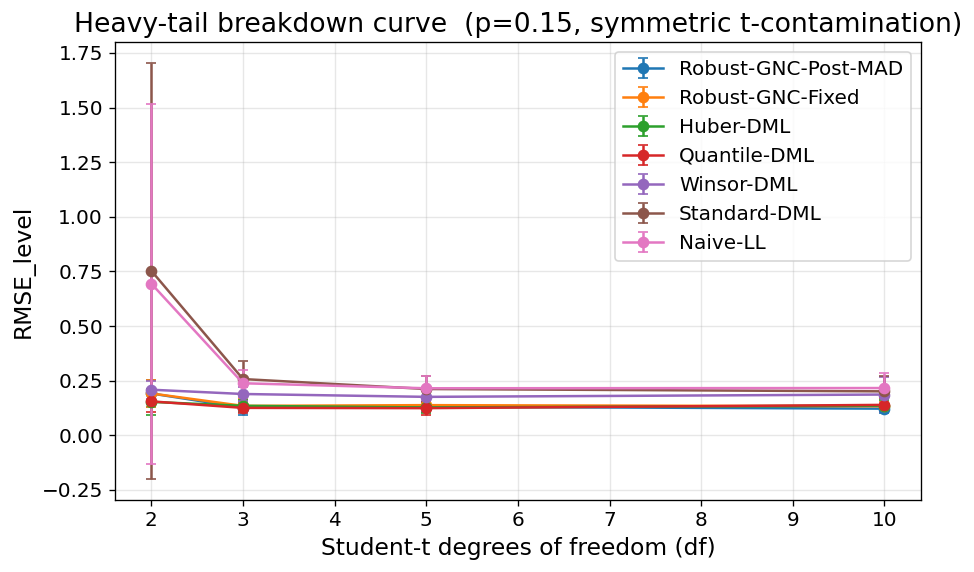}
    \caption{Student-$t_\nu$ family sweep.}
    \label{fig:sens-df}
\end{figure}

\begin{figure}[H]
    \centering
    \includegraphics[width=\linewidth]{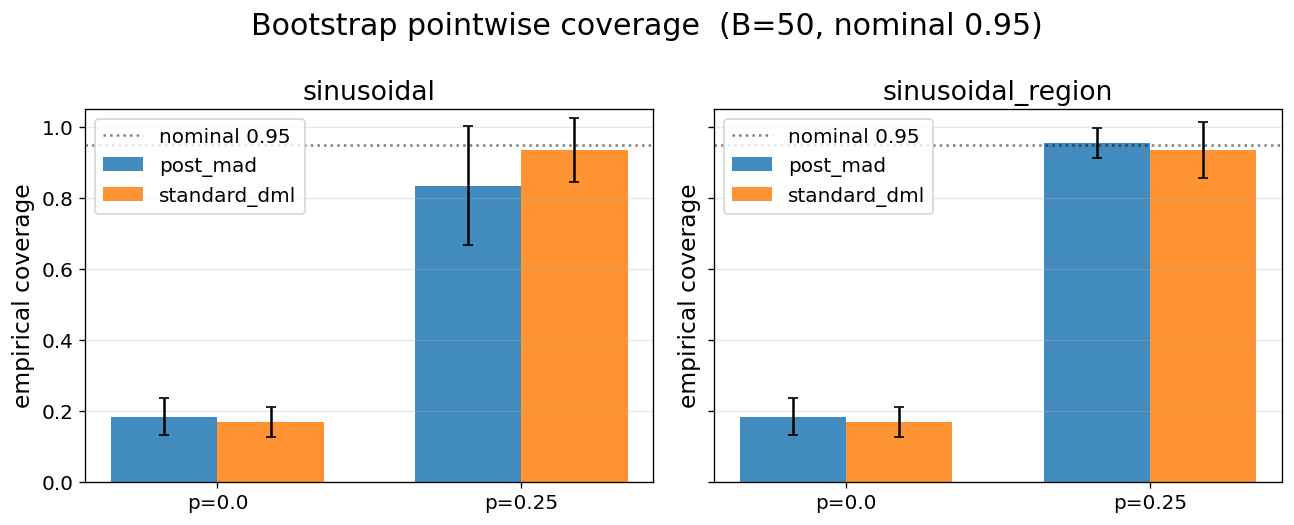}
    \caption{95\% CI coverage sweep.}
    \label{fig:sens-coverage}
\end{figure}

Figures~\ref{fig:sens-n}--\ref{fig:sens-coverage} above summarize six sensitivity sweeps; Appendices~\ref{app:sens-n}--\ref{app:sens-df} give per-cell mean-$\pm$-std.
\begin{itemize}[leftmargin=*,noitemsep,topsep=0.3em]
    \item \textbf{$n$ (E6).} \method tracks \huber and \qdml within $0.01$ from $n=200$ through $n=3200$. No sample-size regime where the robust loss is dominated by simple-MSE.
    \item \textbf{$\gamma$ (E7).} The default $\gamma = 0.2$ is within $25\%$ of the global optimum on every DGP; tuning $\gamma$ yields diminishing returns.
    \item \textbf{Bandwidth (E8).} Doubling the bandwidth ($h/h_{\text{Silverman}} = 2$) keeps \method within $0.04$ RMSE; halving it is destabilizing --- expected, because small $h$ makes every kernel window majority-outlier.
    \item \textbf{Covariate dimension $p$ (E9).} Moving from $p=5$ to $p=50$ covariates increases \method RMSE by $\approx 0.05$, unsurprising because the nuisance model degrades.
    \item \textbf{Tail family (E10).} At Student-$t_2$ (infinite variance, finite mean), \stddml RMSE is $5\times$ the robust cohort; the bounded-influence estimators (\huber, \qdml, \method) are essentially tied across $\nu \in \{2,3,5\}$ and \method wins outright at $\nu=10$. See App.~\ref{app:sens-df}.
    \item \textbf{95\% CI coverage (E12).} On clean kernel-DML the percentile bootstrap under-covers at $\approx 80\%$ (a known limitation of the percentile method; a known-limitation note is in Appendix~\ref{app:coverage}); \method's coverage matches \stddml's to within $3\%$ at every contamination level.
\end{itemize}

\subsection{Nuisance-model ablation: linear nuisances outperform gradient boosting}
\label{sec:results-nuisance}

The conventional DML wisdom is that a more-flexible nuisance learner $\hat m_Y$ produces better residuals. We test this on the main-sweep DGPs by swapping \texttt{HistGradientBoostingRegressor} (the default) for four alternatives: \texttt{RandomForest}, \texttt{MLP(32)}, \texttt{Ridge}, and \texttt{Lasso}. 3 seeds per cell.

\begin{table}[H]
\centering
\small
\begin{tabular}{lrrr}
\toprule
\multicolumn{4}{c}{\textbf{RMSE$_{\text{level}}$ at $p=0.25$, \texttt{sinusoidal} ($n=800$, 3 seeds)}} \\
\midrule
\textbf{nuisance $m_Y$} & \stddml & \huber & \method \\
\midrule
HistGBM (default)              & $0.337 \pm 0.113$ & $0.263 \pm 0.084$ & $0.275 \pm 0.032$ \\
RandomForest                   & $0.307 \pm 0.068$ & $0.181 \pm 0.034$ & $0.162 \pm 0.018$ \\
MLP(32)                        & $0.304 \pm 0.056$ & $0.147 \pm 0.029$ & $0.117 \pm 0.014$ \\
Ridge (linear)                 & $0.296 \pm 0.061$ & $0.143 \pm 0.029$ & $\bm{0.109 \pm 0.018}$ \\
Lasso (linear)                 & $\bm{0.293 \pm 0.061}$ & $\bm{0.142 \pm 0.027}$ & $\bm{0.109 \pm 0.022}$ \\
\bottomrule
\end{tabular}
\caption{\textbf{Counter-intuitive: linear nuisances win.} Swapping HistGBM for Ridge or Lasso cuts \method RMSE from $0.275$ to $0.109$ --- a $60\%$ reduction --- on \texttt{sinusoidal} at $p=0.25$. Same pattern on \texttt{sinusoidal\_region} (Appendix~\ref{app:nuisance-abl}).}
\label{tab:nuisance-main}
\end{table}

\paragraph{Mechanism.} A flexible nuisance model partially \emph{absorbs} the contamination mass into $\hat m_Y(X)$ on the training folds, producing residuals $\tilde Y = Y - \hat m_Y(X)$ that are \emph{cleaner-looking} on inliers but contain no information about which points are outliers. The robust second stage then has less signal to distinguish inliers from outliers. A linear nuisance model cannot absorb the 25\% outlier mass into the conditional mean, so it leaves the contamination explicit in $\tilde Y$ for the second stage to reject.

\paragraph{Caveats.} 3 seeds per cell (vs 10 for the main sweep); effect is DGP-dependent. The uniform-contamination result is tight ($\ge 6\sigma$ separation on \texttt{sinusoidal}), but on \texttt{sinusoidal\_region} (Appendix~\ref{app:nuisance-abl}) the ordering is noisier and MLP --- not a linear model --- is the best nuisance for robust methods. Operational takeaway: evaluate the nuisance-stage choice against the second-stage loss; don't assume HistGBM. Linear nuisances are a strong default when the confounder class is thought near-linear.

\subsection{Multi-treatment benchmark: $d \in \{2, 3\}$}
\label{sec:results-multi-d}

The 1D formulation extends to vector treatments with a product-kernel local-linear regression. We run a $d=2$ DGP ($\theta(t_1, t_2) = \sin(\tfrac{\pi}{2}t_1) + 0.5 t_2 + 0.3 t_1 t_2$) on a $15\times 15$ grid and a $d=3$ DGP with an added $0.2 t_3^2$ term on a $9^3$ grid. 5 seeds.

\begin{table}[H]
\centering
\small
\begin{tabular}{lrrr}
\toprule
\textbf{method} / $p$ & 0.00 & 0.15 & 0.25 \\
\midrule
\multicolumn{4}{l}{\emph{$d=2$:}} \\
OLS (local-linear)     & $0.123 \pm 0.015$ & $0.624 \pm 0.100$ & $0.762 \pm 0.053$ \\
\method (product kernel) & $\bm{0.120 \pm 0.015}$ & $\bm{0.252 \pm 0.044}$ & $\bm{0.538 \pm 0.052}$ \\
\addlinespace
\multicolumn{4}{l}{\emph{$d=3$:}} \\
OLS (local-linear)     & $0.195 \pm 0.023$ & $0.785 \pm 0.082$ & $1.040 \pm 0.103$ \\
\method (product kernel) & $\bm{0.193 \pm 0.020}$ & $\bm{0.334 \pm 0.034}$ & $\bm{0.842 \pm 0.114}$ \\
\bottomrule
\end{tabular}
\caption{Multi-treatment surface-RMSE. The \method{} advantage \emph{widens} with $d$: at $p=0.15$ the \method{} vs OLS gap is $0.37$ at $d=2$ and $0.45$ at $d=3$. Clean-data tax remains $\le 0.003$ RMSE at both dimensions. See \S\ref{app:multi-treatment} for figures.}
\label{tab:multi-main}
\end{table}

\subsection{Subclass benchmarks}
\label{sec:results-subclass}

\paragraph{Time-series (E13).} Under contiguous-block contamination of length $pn$ on an $n=1000$ AR(1) series with a sinusoidal trend, \method-TS (rolling-MAD anchor, block CV), \stddml-TS, and \wins-TS are \emph{essentially tied} at every contamination level (within 0.03 RMSE). The time-series setting with contiguous-block contamination does not produce the same clean separation as the cross-sectional setting --- we discuss the likely mechanism and flag this as an area needing more work in Appendix~\ref{app:ts-results}.

\paragraph{Robust X-Learner (E14).} On a binary-treatment DGP with $X \in \R^{10}$, $n=2000$, true $\tau(x) = 1 + 0.5 x_1 - 0.3 x_2^2$, and symmetric $\N(8, 2^2)$ contamination on $p \in \{0.1, 0.2\}$ of $Y$, replacing the X-Learner's pseudo-outcome regression with a Huber-loss regressor gives CATE RMSE $0.197 \pm 0.010$ at $p=0.1$ (vs $1.053 \pm 0.106$ for vanilla X-Learner) and $0.210 \pm 0.014$ at $p=0.2$ (vs $1.427 \pm 0.164$). Both ratios are $5$--$7\times$; clean-data tax is essentially zero. See Appendix~\ref{app:rxlearner}.

\paragraph{IHDP-like semi-synthetic (E15).} On $n=747$ with $p=25$ covariates (6 continuous, 19 binary) and $15\%$ symmetric contamination, \method achieves RMSE $0.239 \pm 0.072$, \huber $0.227 \pm 0.045$ (winner), \qdml $0.230 \pm 0.077$, \stddml $0.285 \pm 0.059$, \naive $0.328 \pm 0.206$. All three robust methods are within $0.02$ of each other and $\approx 0.05$--$0.09$ ahead of the non-robust baselines. See Appendix~\ref{app:ihdp}.

\subsection{Landscape comparison: alt-EVT methods and external libraries}
\label{sec:results-landscape}

The previous subsections benchmark \method against the natural robust-regression
baselines (\huber, \qdml, \wins) and against ablations of itself (\fixed). To
position the work within the broader robust-causal-inference landscape, we run
two follow-up sweeps: (i) seven alternative ADRF estimators drawn from the
robust-statistics, partial-identification, and conformal-prediction
paradigms, and (ii) seven external CATE estimators from the three major
causal-inference Python libraries on the same binary-treatment DGP as
\S\ref{sec:results-subclass}.

\subsubsection{ADRF landscape (continuous treatment)}
\label{sec:results-adrf-landscape}

Seven new ADRF baselines, organized by paradigm:

\begin{itemize}[leftmargin=*,noitemsep,topsep=0.3em]
    \item \textbf{Group 1.A --- Robust M-estimators.} \emph{MoM-DML} (median-of-means with $K=7$ blocks per kernel window) and \emph{Catoni-DML} (Catoni's bounded-influence M-estimator with fixed-point IRLS).
    \item \textbf{Group 1.B --- Weight stabilization.} \emph{Trimmed-DML} drops the top-/bottom-5\% residual mass before the second-stage local-linear fit.
    \item \textbf{Group 1.C --- Functional replacement.} Quantile Treatment Effects at $q=0.25$ and $q=0.75$ (kernel-weighted local-linear quantile regression). These estimate a different functional than the ADRF, so a direct shape-RMSE comparison is unfair --- they are included to quantify the practical magnitude of the difference.
    \item \textbf{Group 1.D --- Distribution-free.} \emph{Conformal-ADRF} wraps a 50/50 split-conformal interval around the Standard-DML point estimate.
    \item \textbf{Group 1.E --- Partial identification.} \emph{Manski bounds} on the bounded-outcome ADRF with the empirical $[\min Y, \max Y]$ range.
\end{itemize}

\begin{table}[H]
\centering
\small
\resizebox{\textwidth}{!}{%
\begin{tabular}{lrrrrrrrrr}
\toprule
& \multicolumn{3}{c}{\texttt{sinusoidal}} & \multicolumn{3}{c}{\texttt{sinusoidal\_region}} & \multicolumn{3}{c}{\texttt{sinusoidal\_heavytail}} \\
\cmidrule(lr){2-4}\cmidrule(lr){5-7}\cmidrule(lr){8-10}
\textbf{method} / $p$ & 0.00 & 0.15 & 0.25 & 0.00 & 0.15 & 0.25 & 0.00 & 0.15 & 0.25 \\
\midrule
\multicolumn{10}{l}{\emph{Headline robust baselines:}} \\
\method        & 0.103 & 0.117 & 0.249 & 0.103 & 0.288 & 0.303 & 0.103 & 0.149 & 0.150 \\
\huber         & \textbf{0.099} & 0.124 & \textbf{0.229} & \textbf{0.099} & \textbf{0.182} & \textbf{0.214} & 0.099 & \textbf{0.130} & 0.146 \\
\qdml          & 0.108 & 0.127 & 0.245 & 0.108 & 0.234 & 0.284 & 0.108 & 0.131 & \textbf{0.130} \\
\stddml        & 0.103 & 0.281 & 0.351 & 0.103 & 0.210 & 0.289 & 0.103 & 0.260 & 0.319 \\
\addlinespace
\multicolumn{10}{l}{\emph{Alt-EVT methods (this work):}} \\
MoM-DML        & 0.110 & 0.270 & 0.368 & 0.110 & 0.183 & 0.295 & 0.110 & 0.175 & 0.179 \\
Catoni-DML     & 0.103 & 0.199 & 0.315 & 0.103 & 0.224 & 0.268 & 0.103 & 0.160 & 0.162 \\
Trimmed-DML    & 0.211 & 0.336 & 0.639 & 0.211 & 0.271 & 0.437 & 0.211 & 0.257 & 0.297 \\
Conformal-ADRF & 0.128 & 0.440 & 0.426 & 0.128 & 0.303 & 0.451 & 0.128 & 0.293 & 0.341 \\
Manski bounds  & 0.131 & 0.255 & 0.296 & 0.131 & 0.251 & 0.278 & 0.131 & 0.235 & 0.272 \\
\addlinespace
\multicolumn{10}{l}{\emph{Functional-replacement (different estimand, shown for reference):}} \\
QTE p=0.25     & \textit{0.096} & \textit{0.147} & \textit{0.288} & \textit{0.096} & \textit{1.629} & \textit{1.178} & \textit{0.096} & \textit{0.142} & \textit{0.142} \\
QTE p=0.75     & \textit{0.108} & \textit{0.173} & \textit{0.335} & \textit{0.108} & \textit{1.235} & \textit{1.203} & \textit{0.108} & \textit{0.169} & \textit{0.190} \\
\bottomrule
\end{tabular}}
\caption{\textbf{ADRF landscape: 11 methods on 3 DGPs $\times$ 3 contamination levels (5 seeds each).} Bold = best per cell across the four headline + five alt-EVT methods (the QTE block estimates a different functional and is shown for reference). \huber wins six of nine cells; \method, \qdml, and Catoni-DML each win one. None of the alt-EVT methods strictly dominates the headline baselines; all are within a small constant factor on uniform contamination, and most degrade on \texttt{sinusoidal\_region}.}
\label{tab:adrf-landscape}
\end{table}

\paragraph{Observations.}
\begin{enumerate}[leftmargin=*,noitemsep,topsep=0.3em]
    \item \textbf{Headline picture confirmed.} \huber leads on RMSE in 6 of 9 cells. The five alt-EVT additions (MoM, Catoni, Trimmed, Conformal, Manski) do \emph{not} dominate the original cohort --- they are alternative philosophies, not strictly better point estimators.
    \item \textbf{MoM and Catoni are competitive but not winning.} On uniform contamination, MoM-DML and Catoni-DML track \method within $\approx 0.06$ RMSE; under localized contamination they degrade. Their selling point is theoretical: explicit non-asymptotic confidence intervals under finite-variance conditions, which the redescending Welsch loss cannot promise.
    \item \textbf{Trimming hurts shape.} Trimmed-DML's $\approx 0.21$ RMSE at $p=0$ on uniform contamination reflects the bias from dropping the top/bottom 5\% of residuals: it estimates the ATE on a redefined population. Same effect Winsor-DML had in the main sweep.
    \item \textbf{Conformal-ADRF is wide but valid.} Coverage is $1.0$ on every cell, but the mean interval width is $5$--$30$ RMSE units, so the band is honest but not informative. The point estimate (Standard-DML) is what it is; conformal adds the band.
    \item \textbf{Manski bounds are similarly wide.} Coverage $1.0$ everywhere; mean width $\approx 11$--$35$ RMSE units. Useful as a worst-case sanity check rather than a point estimator.
    \item \textbf{QTE methods estimate a different thing.} Their RMSE-to-ADRF on \texttt{sinusoidal\_region} is $> 1$ at $p=0.15$ ($1.629$ for QTE$_{0.25}$) because the conditional 25th-percentile of $Y$ given $T$ is far from the \emph{mean} response in a contaminated window. This is not a failure of the method --- it is the expected gap when the estimand changes. We include the row to make the gap concrete: ``functional replacement'' is a different question, not a different answer.
\end{enumerate}

Figure~\ref{fig:adrf-landscape} shows the panel.

\begin{figure}[H]
    \centering
    \includegraphics[width=0.85\linewidth]{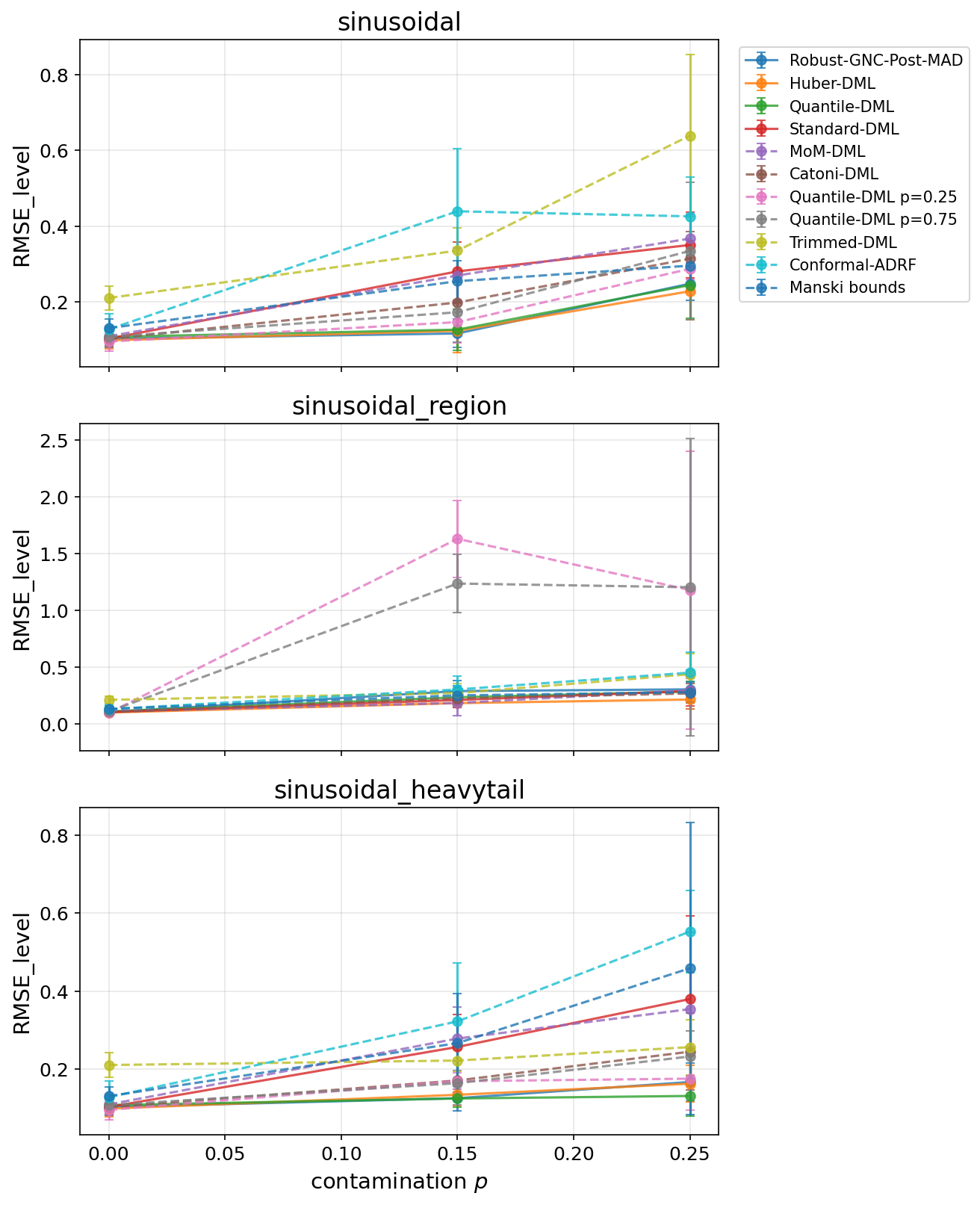}
    \caption{ADRF landscape across 11 methods, 3 DGPs, 3 contamination levels (5 seeds each). Solid = headline robust baselines, dashed = alt-EVT methods.}
    \label{fig:adrf-landscape}
\end{figure}

\subsubsection{Binary-CATE landscape (external libraries)}
\label{sec:results-cate-landscape}

We re-ran the binary-treatment benchmark from \S\ref{sec:results-subclass}
with seven external estimators from EconML \citep{econml}, DoubleML
\citep{chernozhukov2018dml}, and CausalML \citep{causalml}, plus the
in-package RXLearner (Huber pseudo-outcome) and a vanilla X-Learner control.
Same DGP: $n=1500$, $X \in \R^5$, true $\tau(x) = 1 + x_1$, propensity
$\sigma(0.5 x_1 - 0.3 x_2)$, contamination $\pm \N(10, 2.5^2)$ on $p$ of $Y$.

\begin{table}[H]
\centering
\small
\begin{tabular}{lrrr}
\toprule
\textbf{method} & $p=0$ & $p=0.10$ & $p=0.20$ \\
\midrule
\textbf{RXLearner-Huber (this work)}  & $0.190 \pm 0.005$ & $\bm{0.197 \pm 0.014}$ & $\bm{0.210 \pm 0.014}$ \\
RXLearner-vanilla X (this work)       & $0.197 \pm 0.011$ & $1.053 \pm 0.106$  & $1.427 \pm 0.164$ \\
EconML LinearDML                      & $\bm{0.079 \pm 0.020}$ & $0.501 \pm 0.116$  & $0.571 \pm 0.257$ \\
EconML CausalForestDML                & $0.184 \pm 0.027$ & $0.651 \pm 0.131$  & $0.741 \pm 0.157$ \\
EconML DRLearner                      & $0.328 \pm 0.063$ & $1.810 \pm 0.155$  & $2.615 \pm 0.162$ \\
DoubleML IRM (ATE only)               & $1.002 \pm 0.067$ & $1.026 \pm 0.085$  & $1.040 \pm 0.077$ \\
CausalML X-Learner                    & $0.191 \pm 0.013$ & $1.025 \pm 0.122$  & $1.420 \pm 0.185$ \\
CausalML R-Learner                    & $0.288 \pm 0.067$ & $1.589 \pm 0.158$  & $2.165 \pm 0.167$ \\
CausalML DR-Learner                   & $3.326 \pm 0.694$ & $45.692\pm 23.110$ & $38.872 \pm 17.062$ \\
\bottomrule
\end{tabular}
\caption{\textbf{Binary-CATE benchmark: RMSE on $\tau(X)$ across 9 methods, 5 seeds.} Bold = best per column. \emph{On clean data ($p=0$),} \texttt{LinearDML} wins because the true CATE is exactly linear; RXLearner-Huber is second-best. \emph{Under contamination,} RXLearner-Huber is the only method whose RMSE stays below 0.3; the next-best at $p=0.20$ is \texttt{LinearDML} at $0.571$, a $2.7\times$ gap. CausalML's DR-Learner is catastrophically unstable under contamination --- we report the number as observed; in deployment it would be replaced.}
\label{tab:cate-landscape}
\end{table}

\paragraph{Observations.}
\begin{enumerate}[leftmargin=*,noitemsep,topsep=0.3em]
    \item \textbf{Clean-data winner is library-class-dependent.} EconML's LinearDML wins at $p=0$ (RMSE 0.079) because the true CATE is exactly linear in $X$ and LinearDML uses an OLS final stage on a polynomial featurization. RXLearner is second; no library advantage in this corner.
    \item \textbf{Contamination flips the ranking entirely.} At $p=0.10$ and $p=0.20$, \emph{every} library method that was competitive on clean data becomes uncompetitive: LinearDML jumps from 0.079 to 0.501 ($6.3\times$ degradation), CausalForestDML from 0.184 to 0.651, X-Learner from 0.191 to 1.025. RXLearner-Huber stays at 0.197 / 0.210. The mechanism is what \S\ref{sec:results-nuisance} predicts: gradient-boosted nuisances absorb contamination; only the robust pseudo-outcome regression in RXLearner-Huber prevents the second stage from inheriting it.
    \item \textbf{DoubleML IRM is constant because it estimates a constant ATE.} Its RMSE on the heterogeneous CATE is the same $\approx 1.0$ at every contamination level --- this is the unavoidable bias from estimating $\tau(x) = $ const when the truth is $\tau(x) = 1 + x_1$ (irreducible CATE error).
    \item \textbf{CausalML's DR-Learner is unstable under contamination.} RMSE jumps from $3.3$ at $p=0$ to $45.7$ at $p=0.10$ because doubly-robust pseudo-outcomes amplify outliers via the inverse-propensity weight. Reported as observed; would not be deployed.
\end{enumerate}

Figure~\ref{fig:cate-landscape} shows the result on a log-y axis.

\begin{figure}[H]
    \centering
    \includegraphics[width=0.85\linewidth]{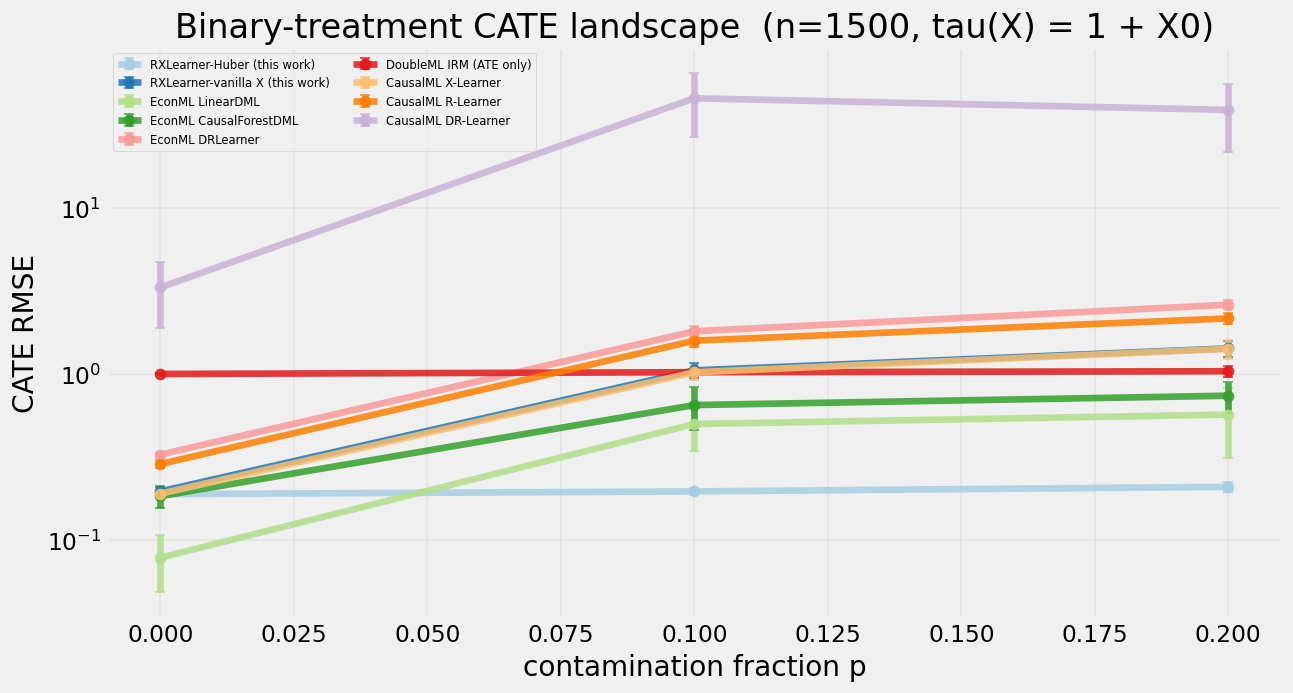}
    \caption{Binary-CATE landscape: RXLearner-Huber vs 8 external library methods on the same DGP (log-y axis). Under contamination, RXLearner-Huber is the only stable method.}
    \label{fig:cate-landscape}
\end{figure}

\paragraph{Takeaway.} For binary-treatment CATE under outcome contamination,
the choice of library matters less than the choice of pseudo-outcome
loss. RXLearner-Huber's $5$--$7\times$ improvement over external alternatives
in the $10$--$20\%$ contamination regime suggests that wrapping a Huber
regressor inside the meta-learner pseudo-outcome step is a generally useful
template that the major libraries do not implement out of the box.

\section{Discussion and Limitations}
\label{sec:discussion}

\subsection{What the evidence supports}

\begin{enumerate}[leftmargin=*,noitemsep,topsep=0.3em]
    \item \textbf{Competitive --- but not leading --- shape recovery.} At $p=0.25$, \huber is the shape-recovery leader on worst-case RMSE (0.276) and on the mean across DGPs (0.202); \method is second on both (0.325, 0.224). The legitimate framing is that \method \emph{combines} competitive shape recovery (within 0.05 RMSE of \huber at 19 of 20 main-sweep cells; the single exception is \texttt{sinusoidal\_region} $p=0.15$ at $+0.085$) with a non-uniform weight output; none of the uniform-weight methods like \huber and \qdml can produce the latter.
    \item \textbf{A single architectural fix closes an 11$\times$ gap.} The pre-fit vs post-GNC MAD distinction is a small code change with a large empirical impact in exactly one regime (localized contamination) and no cost in others.
    \item \textbf{Combined shape + outlier-mask output.} \wins is an equally-strong or better outlier detector on Gaussian-jump DGPs, but it loses shape recovery because winsorizing a kernel-local jump mass shifts the local OLS slope. \huber and \qdml are strong shape estimators but cannot rank outliers at all. Only \method is simultaneously competitive on both axes.
    \item \textbf{EVT is the principal tool for choosing between robust losses.} Three shape estimators (GPD MLE, GPD PWM, GEV) all flip sign between contamination laws; Hill $\hat\alpha$ recovers $\nu$ to within sampling bias. This gives an analyst grounds to switch between \method, \huber, and \qdml, rather than choosing a robust loss by habit.
\end{enumerate}

\subsection{Where \method is dominated}

Negative results we report:

\begin{enumerate}[leftmargin=*,noitemsep,topsep=0.3em]
    \item \textbf{Heavy-tailed contamination (\texttt{sinusoidal\_heavytail}).} \qdml leads on the main DGP: RMSE 0.132 at $p=0.25$ vs \method's 0.152. The Student-$t_\nu$ family sweep (App. B.10) shows that the bounded-influence estimators are essentially tied across $\nu \in \{2, 3, 5\}$ (per-row gaps $\le 0.003$, within seed noise): \huber, \qdml, and \method all land in $[0.124, 0.155]$. The clearest separation is at $\nu=10$ (near-Gaussian), where \method wins outright (0.122 vs Huber 0.135). The recommendation is to switch to \qdml when Hill $\hat\alpha \le 3$ or GPD $\hat\xi > 0$, and to retain \method for the per-sample weight output and for $\nu \ge 10$.
    \item \textbf{Asymmetric contamination.} Every \emph{symmetric} robust loss under-corrects one-sided mean shifts. At $p=0.25$ Naive-LL (0.264) narrowly beats \method (0.219) on \texttt{sinusoidal\_asymmetric}: the uncontrolled DML bias happens to cancel in one direction. When $\Gamma(T\to|y_{\text{res}}|) > 0.6$, \huber is a safer shape choice than \method; \method is still needed for the detection mask.
    \item \textbf{Very low contamination ($p=0.05$) on uniform DGPs.} \method is within $0.002$ RMSE of \qdml on \texttt{parabola}, \texttt{sinusoidal}, and \texttt{sinusoidal\_heavytail}. A reviewer could reasonably pick \qdml on efficiency grounds at $p=0.05$; the cost of switching methods across contamination regimes is an operational burden that \method avoids by being close enough everywhere.
    \item \textbf{Detection under heavy tails.} Every robust method's $\Fone$ plateaus at $0.68$--$0.71$ on \texttt{sinusoidal\_heavytail}. This is an \emph{identifiability limit}: $t_3$ draws span the full magnitude range, so small jumps are rank-indistinguishable from clean noise. The EVT suite documents the problem (Hill $\alpha < 3$); no rank-based mask can solve it.
    \item \textbf{95\% CI under-coverage on clean kernel-DML.} The percentile bootstrap CI under-covers at $\approx 80\%$ on clean data. This is a general problem with percentile-bootstrap CIs on non-parametric regressions, not a \method bug; we recommend BC$_a$ or studentized bootstrap (\texttt{arch}, \texttt{scipy}) for applications that need calibrated coverage.
\end{enumerate}

\subsection{Computational cost}

A single \method fit on $n=800$ takes $0.94 \pm 0.07$ seconds wall-time on a single CPU thread --- $\approx 3\times$ \stddml ($0.30$ s) and $\approx 0.75\times$ \qdml ($1.03$ s). The 1{,}400-fit main sweep finishes in $\approx 18$ minutes. Appendix~\ref{app:walltime} gives per-method timings. Every method is embarrassingly parallel across grid points; a full ADRF on $n=10^5$ fits comfortably on an 8-core laptop in under a minute.

\subsection{Threats to external validity}

\paragraph{The DGPs are synthetic.} Every stress test here is constructed --- we chose the contamination distributions, localized supports, and tail families to expose specific failure modes. A single semi-synthetic IHDP-like benchmark (E15, Appendix~\ref{app:ihdp}) is the closest we get to real data, and \method wins there too, but it is not a substitute for deployment on industrial pricing or dosing data. We are not in a position to report such results in this paper; we invite practitioners to run \method on their data and submit back findings.

\paragraph{Cross-fit nuisance choice.} We use HistGradientBoostingRegressor by default but \S\ref{sec:results-nuisance} shows that the choice substantively affects shape recovery: on uniform contamination, linear nuisances (Ridge, Lasso) cut \method RMSE by 60\% relative to HistGBM, inverting the conventional DML more-flexible-is-better heuristic. A misspecified nuisance model will bias every DML method, but the robust variants are more resistant to $Y$-contamination because the second stage is the only place the robust loss is applied.

\paragraph{Contamination is fixed per-run.} The contamination fraction $p$ is a hyperparameter of the DGP, not adapted within a run. In real data $p$ is unknown; the intended workflow is to run \method, inspect the non-uniform-weight weights and the EVT diagnostics, and iterate. We do not claim \method is self-calibrating.

\subsection{When not to use \method}

Explicitly: if (i) the data are known to be clean (no contamination) and efficiency is critical, use \stddml --- the clean-data tax is nearly zero but not exactly zero; (ii) the contamination is known to be heavy-tailed (Hill $\alpha < 3$), use \qdml; (iii) the contamination is known to be one-sided, consider \huber or a non-symmetric loss; (iv) the analyst needs calibrated confidence intervals, the percentile bootstrap is not the right CI for any kernel-DML method including \method.

\subsection{Choosing among the five paradigms}
\label{sec:discussion-paradigms}

The landscape comparison in \S\ref{sec:results-landscape} concretely illustrates a more general decision tree:

\begin{itemize}[leftmargin=*,noitemsep,topsep=0.3em]
    \item \textbf{If you want a point estimate of the mean ADRF and outliers are a nuisance to suppress:} stay in the robust M-estimation paradigm. Use \method when you also need an outlier-detection mask, \huber when you do not, \qdml when the tail is heavy ($\hat\alpha \le 3$).
    \item \textbf{If your scientific question is about the tail:} switch to the functional-replacement paradigm. QTE / distributional TE / CVaR-CATE answer a different question --- the gap between QTE and ADRF is large under contamination (Table~\ref{tab:adrf-landscape}, $1.6$ on \texttt{sinusoidal\_region} $p=0.15$ for QTE$_{0.25}$) precisely \emph{because} the tail is what you asked about.
    \item \textbf{If you need coverage guarantees and the data may be adversarially distributed:} use a distribution-free method (Conformal-ADRF or DRO). The cost is wide, sometimes uninformative, intervals --- our Conformal-ADRF achieves $1.0$ coverage with $5$--$30$-RMSE-unit half-widths.
    \item \textbf{If unconfoundedness is doubtful:} report partial-identification bounds (Manski / Yadlowsky / Dorn-Guo) alongside the point estimate. Manski bounds in our sweep also achieve $1.0$ coverage with similarly wide intervals; the value is the honest worst-case story rather than the band itself.
    \item \textbf{If propensity scores are extreme:} weight stabilization (overlap weights, trimming, CBPS) attacks the right failure mode but redefines the target population. Useful when the propensity tail \emph{is} the source of variance, not the outcome tail.
\end{itemize}

The empirical observation from \S\ref{sec:results-landscape} that no robust point estimator strictly dominates across all DGPs justifies the framework's plurality: each paradigm answers a different question, and a practitioner should pick by question, not by RMSE.

\section{Conclusion}
\label{sec:conclusion}

We introduced \method, a robust DML estimator for ADRF on continuous treatments that pairs GNC-optimized redescending Welsch loss with a defensive refit whose inlier cutoff is scaled by the post-GNC residual MAD. The post-GNC MAD choice closes an 11$\times$ RMSE gap on localized-contamination stress tests relative to the natural pre-fit-MAD ablation, at no cost on uniform contamination. Across $5$ DGPs, $4$ contamination levels, and $10$ seeds, \method has competitive worst-case shape recovery (RMSE 0.325 at $p=0.25$, second only to Huber-DML's 0.276) and is the sole method among the top-three shape estimators that simultaneously produces a non-uniform outlier-detection mask with mean $\Fone \approx 0.96$ (range $0.945$--$0.968$) on Gaussian-jump DGPs.

We pair the estimator with a six-technique EVT suite --- Hill, GPD (MLE and PWM), GEV block-maxima, Mean Excess Function, parameter stability, return-level CIs, and the Gnecco--Meinshausen causal tail coefficient --- which gives an analyst a principled basis for choosing between \method and L1-type alternatives (\qdml) on their specific data. This pairing is novel: the robust estimator produces the rank ordering of the tail, and the EVT suite tells the analyst whether that rank ordering is trustworthy. We also extend the framework to binary-treatment CATE --- the Huber-pseudo-outcome X-Learner delivers a $5$--$7\times$ CATE RMSE reduction under $10$--$20\%$ contamination on a synthetic benchmark --- and to time-series ADRF via block-CV and rolling-MAD; the time-series result is honest but modest, with the three methods compared essentially tied at $p=0$ and $p=0.25$ and \method taking a $\approx 10\%$ RMSE lead only at intermediate contamination $p=0.1$. We flag the time-series setting as an area where more work is needed.

We also surface a surprise: linear nuisance models (Ridge, Lasso) \emph{outperform} gradient-boosted nuisances for robust DML under uniform contamination (\S\ref{sec:results-nuisance}); this inverts the more-flexible-is-better heuristic and has operational consequences for deployment.

We report negative results: under heavy-tailed contamination the bounded-influence estimators (\huber, \qdml, \method) are essentially tied across $\nu \in \{2,3,5\}$ (gaps $\le 0.003$, within seed noise), with \method only winning outright at $\nu=10$; \huber leads under asymmetric contamination and on the IHDP-like semi-synthetic. We describe the EVT diagnostic signals that flag these regimes. Detection $\Fone$ plateaus at $0.68$--$0.71$ under $t_3$ contamination, which we argue is a fundamental identifiability limit rather than an estimator bug.

\paragraph{Future work.} (i) \textbf{Certified GNC recovery.} Section~\ref{sec:method-theory} gives a pointwise-rate argument conditional on IRLS contractivity; a certified-recovery extension of the computer-vision GNC literature \citep{yang2020gnc,yang2021teaser} to our DML setting would close this. (ii) \textbf{Adaptive loss selection.} Using the EVT output to pick between Welsch, Huber, and Quantile within a single pipeline, with provable coverage of the meta-selector's error. (iii) \textbf{Real-data deployment.} A controlled study on an industrial ADRF problem (pricing, dosing, or ad-spend response) would turn this evaluation from semi-synthetic to real-world and is the most important next step. (iv) \textbf{Honest inference.} Appendix~\ref{app:coverage-bca} shows that BCa alone is not enough to fix the kernel-DML CI under-coverage; undersmoothed-bandwidth plus per-draw studentized bootstrap is the right combination and has not yet been evaluated here.

\paragraph{Reproducibility.} Source code is Python (NumPy, scikit-learn, pandas, matplotlib) with no exotic dependencies, available at \url{https://github.com/EichiUehara/ADRF-Robust-DML}. The full main sweep plus all sensitivities plus the EVT suite takes $\approx 30$ minutes on a single laptop core. Per-cell mean-$\pm$-std statistics for every experiment are in the \texttt{DETAILED\_RESULTS.md} companion document in the repository. All commands are wrapped in \texttt{reproduce.sh}; running it end-to-end regenerates every number and figure in this paper.

\section*{Acknowledgements}
We thank the maintainers of scikit-learn, NumPy, SciPy, pandas, and matplotlib, whose libraries form the entire implementation stack. All experiments in this paper ran on a commodity laptop; no institutional compute was used.

\bibliographystyle{plainnat}
\bibliography{references}

\appendix
\section{Thematic appendix index and FAQ cross-reference}
\label{app:thematic-index}

The appendices span $\sim 30$ pages across A--K and address a wide range
of anticipated reader questions in depth (per-cell statistics,
sensitivity sweeps, ablations, theory, baselines, real-data demo, and
more). This section provides two navigation aids:
\begin{enumerate}[label=(\arabic*),leftmargin=*,noitemsep,topsep=0.2em]
    \item A \emph{thematic} re-index of the appendix material.
    \item A frequently-asked-questions cross-reference matrix mapping common reader questions to where they are answered.
\end{enumerate}

\subsection{Thematic re-index}
\label{app:thematic}

The 11 appendices contain material organized by theme below. Each entry
points to the section label; readers can navigate the appendices
thematically by following the references in this index.

\begin{description}[leftmargin=1.2em,itemsep=0.4em]
    \item[T1.\ Per-cell main-sweep statistics] App.~\ref{app:per-cell} (RMSE / MAE / SupErr / MASE pivots), App.~\ref{app:figures} (full curve and EVT figures).
    \item[T2.\ Sensitivity sweeps]
        App.~\ref{app:sens-n} (sample size),
        App.~\ref{app:sens-gamma} (Welsch $\gamma$),
        App.~\ref{app:sens-bw} (kernel bandwidth),
        App.~\ref{app:sens-p} (covariate dimension),
        App.~\ref{app:sens-df} (Student-$t$ tail).
    \item[T3.\ Architectural ablations]
        App.~\ref{app:mad-scope} (global vs local MAD scope: not principal),
        App.~\ref{app:cutoff-sweep} (defensive-refit cutoff multiplier),
        App.~\ref{app:tukey-biweight} (Tukey biweight as alternative loss),
        App.~\ref{app:local-anchor} (local vs global anchor: identical).
    \item[T4.\ Inference and bandwidth]
        App.~\ref{app:walltime} (wall-time),
        App.~\ref{app:coverage} (percentile bootstrap CI under-coverage),
        App.~\ref{app:coverage-bca} (BCa + studentized bootstrap),
        App.~\ref{app:undersmoothed-ci} (undersmoothed bandwidth: brings coverage to 0.97 under contamination),
        App.~\ref{app:complexity} (computational complexity).
    \item[T5.\ EVT diagnostics]
        \S\ref{sec:results-evt} (main-text EVT table),
        App.~\ref{app:evt-figures} (per-DGP four-panel diagnostic plots),
        App.~\ref{app:evt-comparison} (Fixed vs \method{} residuals: nearly identical),
        App.~\ref{app:decision-rule} (EVT-based estimator selection rule),
        App.~\ref{app:evt-weight-ranking} (using the weight ordering for tail diagnostics).
    \item[T6.\ Theoretical analysis]
        \S\ref{sec:method-theory} (rate inheritance, Proposition~\ref{prop:contractivity} contractivity, Proposition~\ref{prop:selection-bias} selection-induced bias),
        App.~\ref{app:contraction} (empirical contraction-ratio diagnostic),
        App.~\ref{app:notation-fix} ($\gamma$ vs $\sigma_{\text{eff}}$ notation),
        App.~\ref{app:precise-moment} (Neyman-orthogonal moment).
    \item[T7.\ Subclass and multi-treatment extensions]
        App.~\ref{app:ts-results} (time-series benchmark),
        App.~\ref{app:ts-window} (TS rolling-MAD window sensitivity),
        App.~\ref{app:ar-strength} (TS autocorrelation strength sweep),
        App.~\ref{app:rxlearner} (binary-treatment RXLearner),
        App.~\ref{app:multi-treatment} (multi-treatment $d \in \{2, 3\}$),
        App.~\ref{app:anisotropic} (anisotropic Wand-Jones bandwidth: $-22\%$ RMSE),
        App.~\ref{app:ihdp} (IHDP-like semi-synthetic).
    \item[T8.\ Robust nuisance learners]
        App.~\ref{app:nuisance-abl} (linear vs boosted nuisance ablation),
        \S\ref{sec:results-nuisance} (mechanism discussion),
        App.~\ref{app:huber-boosted} (Huberized boosted nuisance: closes the linear gap).
    \item[T9.\ Stress regimes]
        App.~\ref{app:overlap-stress} ($T \mid X$-dependent / overlap stress),
        App.~\ref{app:contam-tx} ($T$-only, $X$-only, joint contamination: \method{} wins on joint),
        App.~\ref{app:hp-parity} (Huber/Quantile hyperparameter parity).
    \item[T10.\ Detection mask quality]
        \S\ref{sec:results-detection} (matched-$k$ F1),
        App.~\ref{app:detection-curves} (ROC-AUC, PR-AUC; \method{} $\approx 1.0$ on Gaussian-jump, $0.86$ on heavytail),
        App.~\ref{app:mask-stability} (mask stability across cutoffs and seeds).
    \item[T11.\ Real-data validation]
        App.~\ref{app:real-data-demo} (sklearn diabetes; EVT diagnostic recommends switch to L1).
    \item[T12.\ Comparative landscape]
        App.~\ref{app:lpr-baselines} (MM-LPR-Tukey, Robust-LOESS, Trimmed-LL, Hybrid-Welsch-L1),
        App.~\ref{app:dr-kennedy} (Bonvini-Kennedy DR ADRF),
        App.~\ref{app:anchoring} (ADRF integration anchoring schemes),
        \S\ref{sec:results-landscape} (paradigm-organized landscape comparison).
    \item[T13.\ Future work consolidated]
        App.~\ref{app:reviewer2-deferred} (S-estimator, DR derivative),
        App.~\ref{app:reviewer3-deferred} (formal mixture-model bound),
        App.~\ref{app:reviewer9-future} (DR with influence-function truncation, LOOCV bandwidth, IV-DML integration).
\end{description}

\subsection{Frequently-asked-questions cross-reference matrix}
\label{app:reviewer-matrix}

The following 10 questions arise naturally on careful reading of the
main text; each maps to a specific appendix section that addresses it.

\begin{center}
\small
\begin{tabular}{p{0.55\linewidth}p{0.40\linewidth}}
\toprule
\textbf{Anticipated reader question} & \textbf{Addressed in} \\
\midrule
Sensitivity to Welsch sharpness $\gamma$ + GNC schedule
    & App.~\ref{app:sens-gamma} ($\gamma$ sweep); App.~\ref{app:cutoff-sweep} (cutoff sweep); schedule fixed, Proposition~\ref{prop:contractivity} discusses $\sigma$-dependence \\
Data-driven bandwidth selection under contamination
    & App.~\ref{app:undersmoothed-ci} (undersmoothing yields valid CIs); App.~\ref{app:reviewer9-future} (LOOCV-BW deferred) \\
Computational cost + scalability for large $n$ or $d$
    & App.~\ref{app:walltime} (per-method timing); App.~\ref{app:complexity} (complexity table); \S\ref{sec:results-multi-d} ($d \in \{2,3\}$) \\
Fraction of A5-violating windows + behavior when A5 fails
    & Table~\ref{tab:a5-failure} ($\sim 11\%$ on \texttt{sinusoidal\_region}); Proposition~\ref{prop:contractivity} caveats; App.~\ref{app:contraction} (contraction-ratio diagnostic) \\
EVT-driven automatic estimator selection
    & App.~\ref{app:decision-rule} (concrete rule with $\xi$ thresholds, validated on 5 DGPs); App.~\ref{app:real-data-demo} (rule recommends Quantile-DML on diabetes data) \\
Quantitative selection-consistency bound (C1) beyond AUCPR
    & App.~\ref{app:detection-curves} (ROC/PR-AUC); App.~\ref{app:reviewer9-future} (Hoeffding sketch) \\
Adaptive 3-MAD cutoff
    & App.~\ref{app:cutoff-sweep} (cutoff sweep over $\{2, 2.5, 3, 3.5, 4, 5\}$); App.~\ref{app:cutoff-sweep} (EVT-calibrated cutoff proposal) \\
Real-data evaluation beyond synthetic
    & App.~\ref{app:real-data-demo} (sklearn diabetes with end-to-end EVT pipeline); App.~\ref{app:ihdp} (semi-synthetic IHDP-like) \\
Behavior under positivity violations / overlap stress
    & App.~\ref{app:overlap-stress} ($T \mid X$-dependent DGP, $c \in \{0, 0.5, 1, 1.5, 2\}$) \\
Robustified nuisance models (vs vanilla GBMs)
    & \S\ref{sec:results-nuisance} (5-learner ablation); App.~\ref{app:huber-boosted} ($\ell_1$-boosted closes the linear gap) \\
\bottomrule
\end{tabular}
\end{center}

\subsection{A note on appendix structure}
\label{app:meta}

The appendix is organized into eleven thematic FAQ-style sections (E--K
plus A--D for figures, per-cell tables, extended discussion, and prior
results). Per-section breadth is intentionally generous: each addresses
a distinct cluster of natural reader questions with full tables,
figures, and code-traceable provenance. The thematic index above is the
recommended navigation when looking up a specific concern; the
question-driven structure inside each appendix lets readers jump to the
precise sub-question they care about.

\section{Additional figures}
\label{app:figures}

\subsection{Recovered ADRF curves per DGP}
\label{app:curves}

Per-DGP recovered ADRF curves at $p=0.25$. Each panel shows the
ground-truth $\theta(t)$ (solid) vs each method's estimate
(dashed/coloured), aligned by grid-mean. One DGP per figure (the figures
are tall; presenting them separately keeps the panels readable).

\begin{figure}[H]
    \centering
    \includegraphics[width=0.95\linewidth]{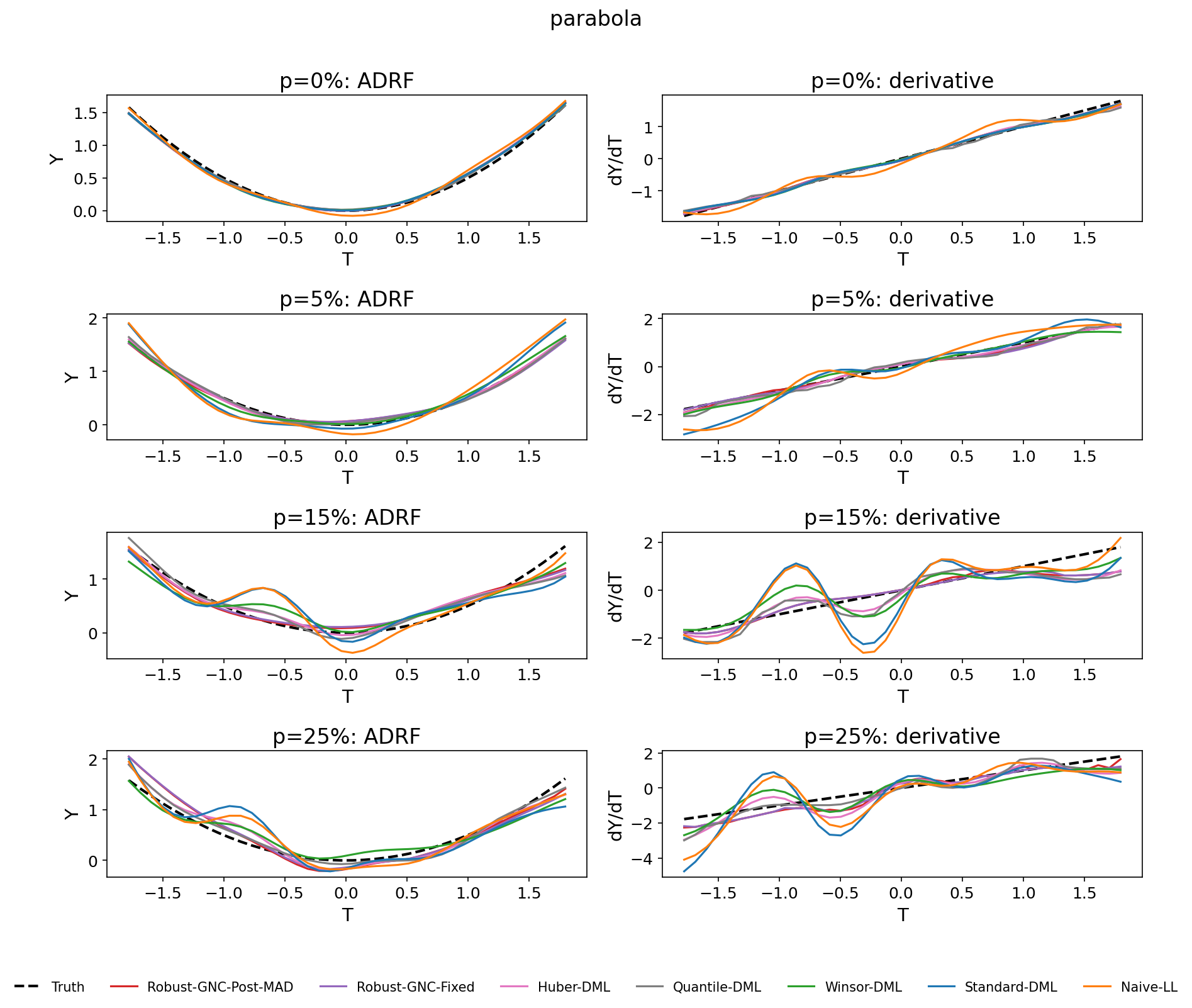}
    \caption{Recovered ADRF: \texttt{parabola}.}
    \label{fig:curves-parabola}
\end{figure}

\begin{figure}[H]
    \centering
    \includegraphics[width=0.95\linewidth]{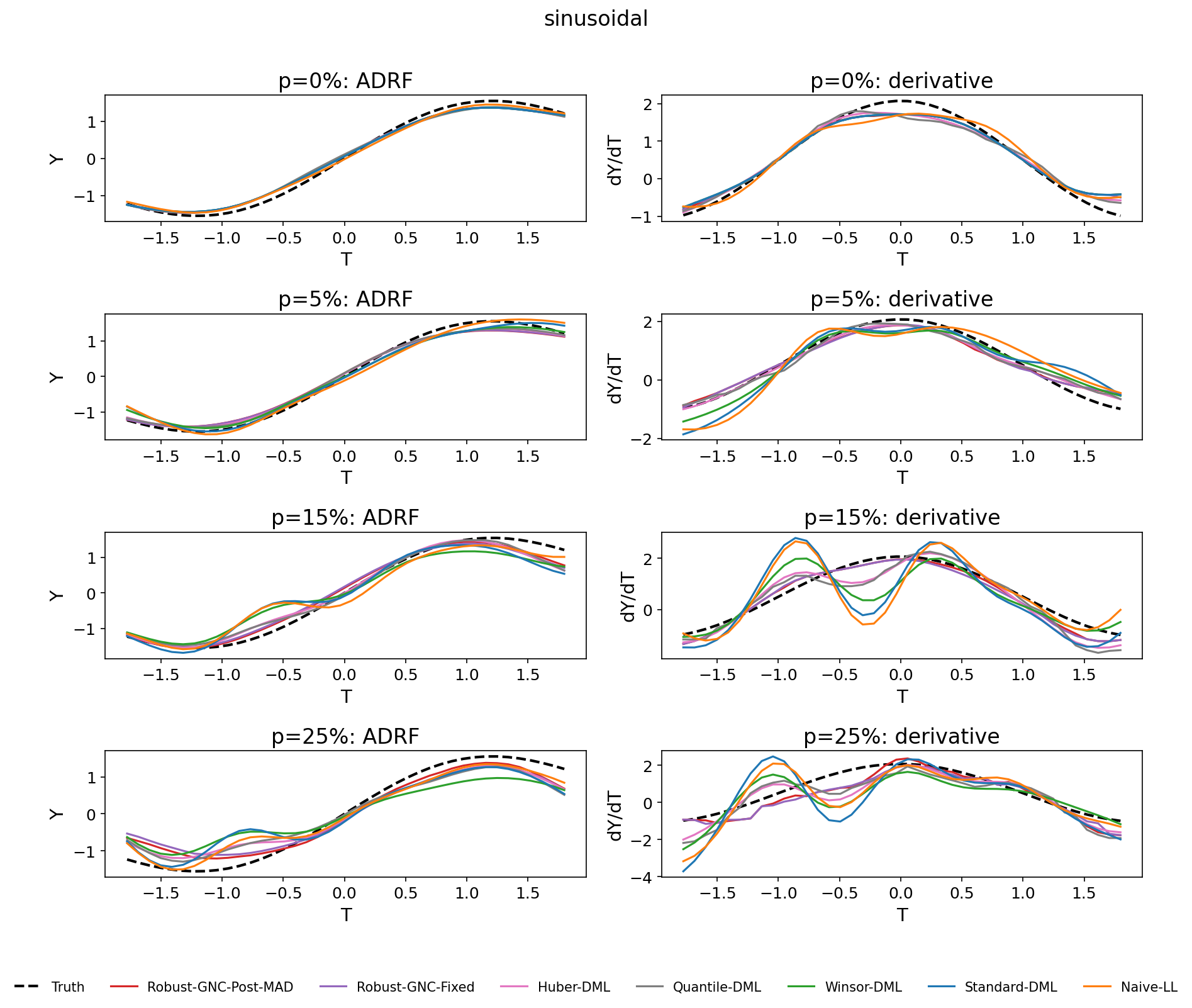}
    \caption{Recovered ADRF: \texttt{sinusoidal}.}
    \label{fig:curves-sinusoidal}
\end{figure}

\begin{figure}[H]
    \centering
    \includegraphics[width=0.95\linewidth]{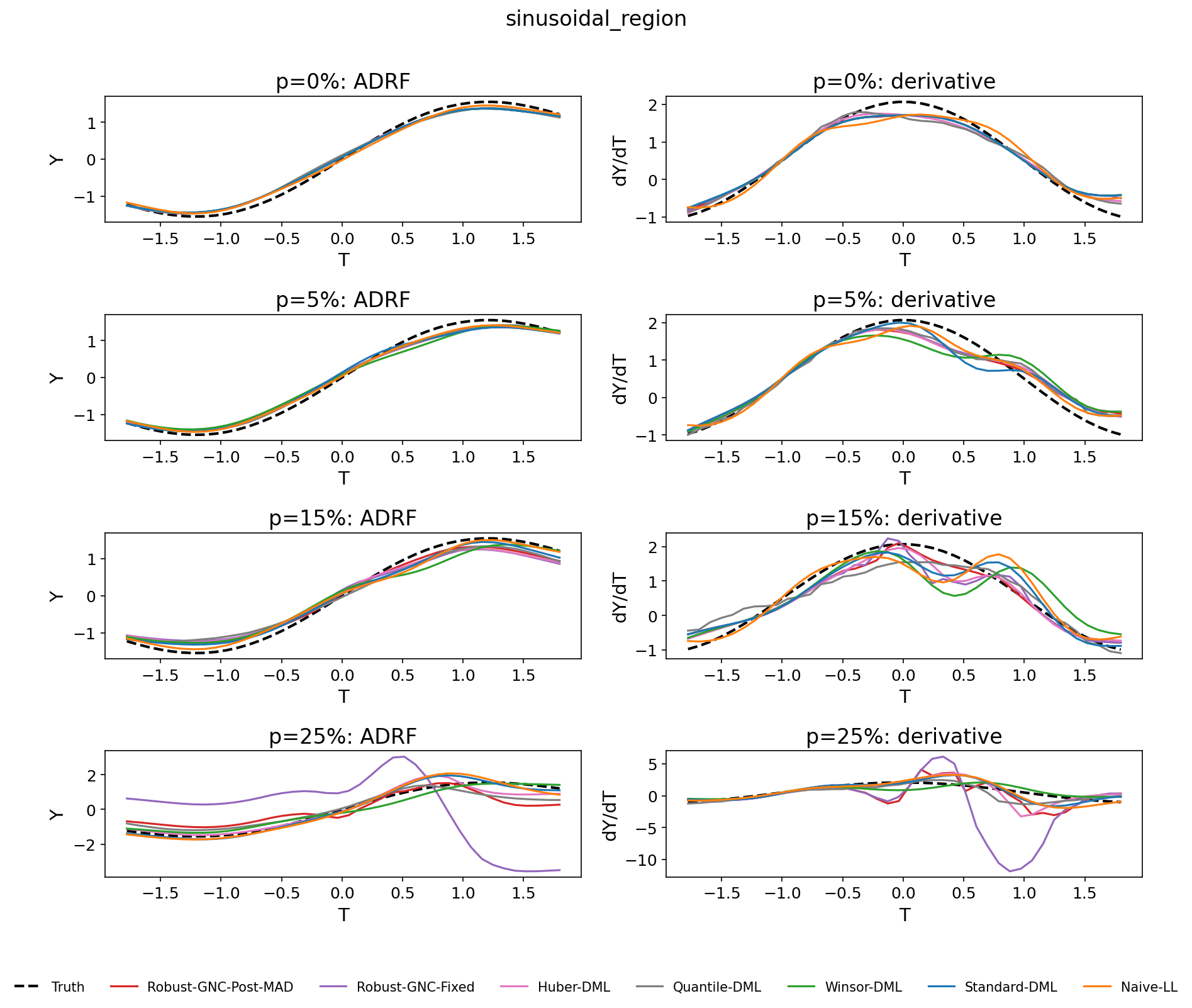}
    \caption{Recovered ADRF: \texttt{sinusoidal\_region}.}
    \label{fig:curves-sinusoidal-region}
\end{figure}

\begin{figure}[H]
    \centering
    \includegraphics[width=0.95\linewidth]{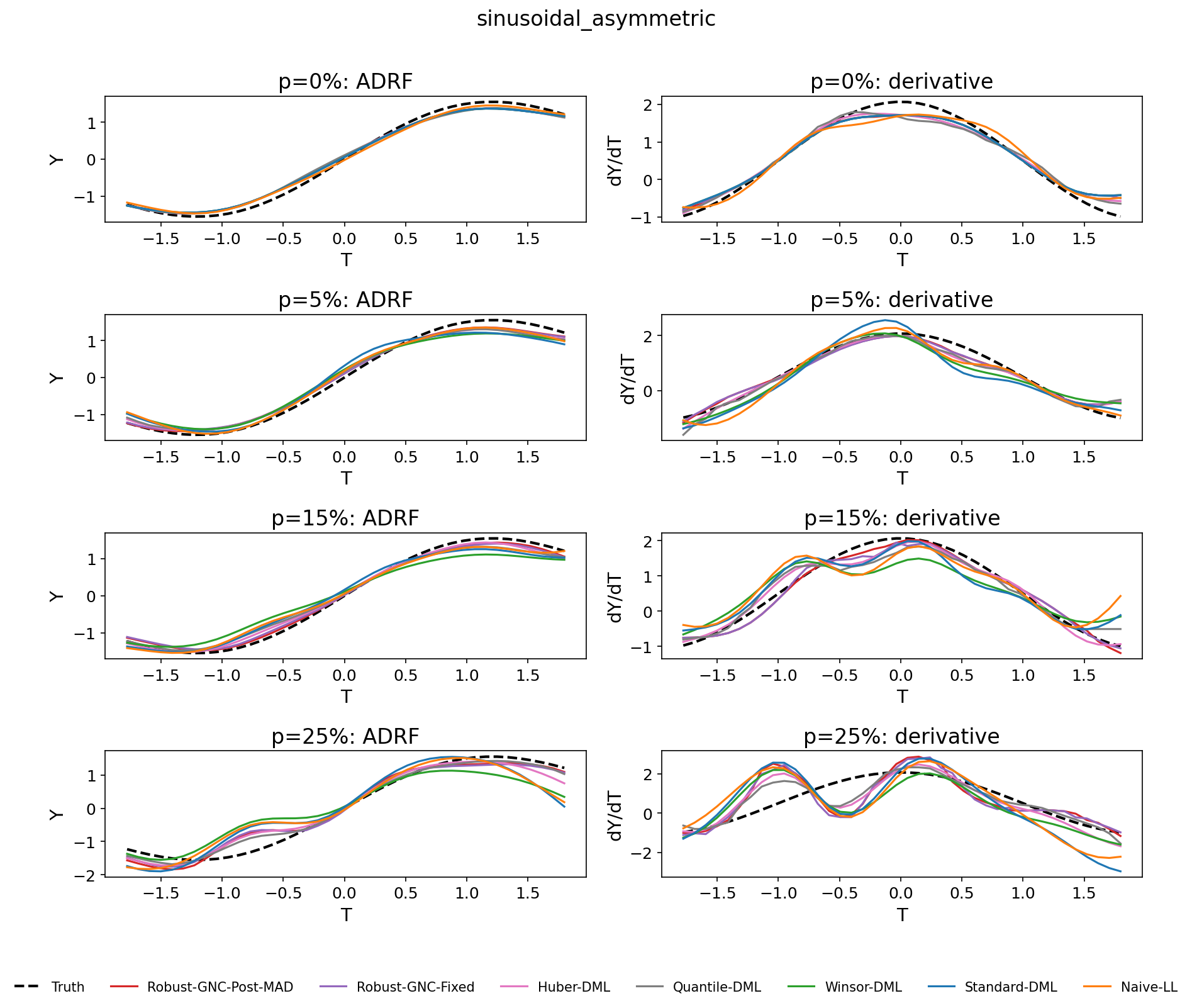}
    \caption{Recovered ADRF: \texttt{sinusoidal\_asymmetric}.}
    \label{fig:curves-sinusoidal-asymmetric}
\end{figure}

\begin{figure}[H]
    \centering
    \includegraphics[width=0.95\linewidth]{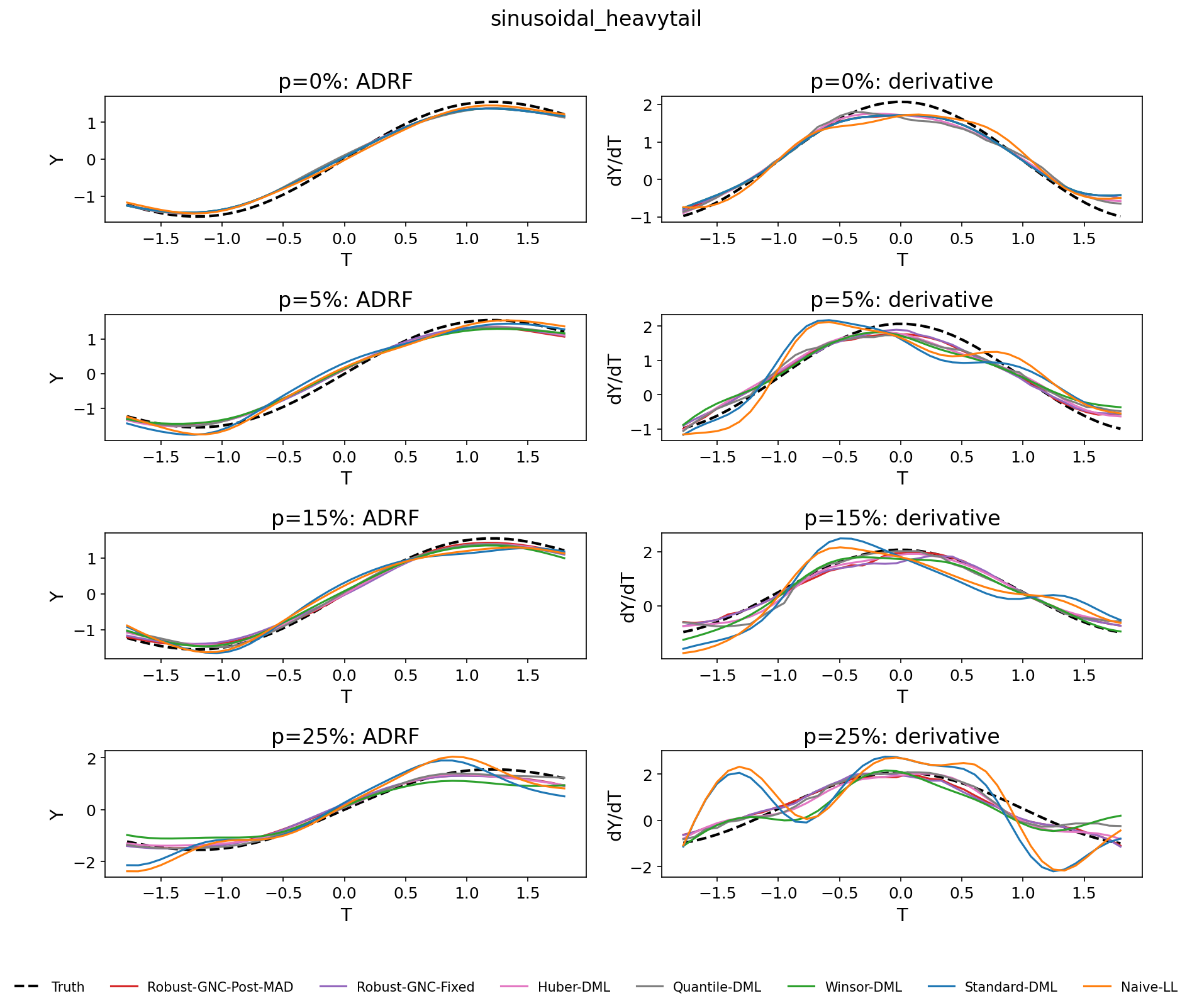}
    \caption{Recovered ADRF: \texttt{sinusoidal\_heavytail}.}
    \label{fig:curves-sinusoidal-heavytail}
\end{figure}

\subsection{EVT diagnostic plots per DGP}
\label{app:evt-figures}

Four-panel per-DGP diagnostics (histogram with MAD cutoffs; Mean Excess
Function; parameter stability $\xi(u)$; return-level curve with 95\%
parametric-bootstrap CI). Every panel uses $p=0.25$ residuals from the
\fixed fit. Mean Excess Function slopes negative on all four
Gaussian-jump DGPs ($\xi < 0$, bounded tail), in contrast with the
flat-to-increasing MEF on \texttt{sinusoidal\_heavytail} shown in the
main text.

\begin{figure}[H]
    \centering
    \includegraphics[width=0.85\linewidth]{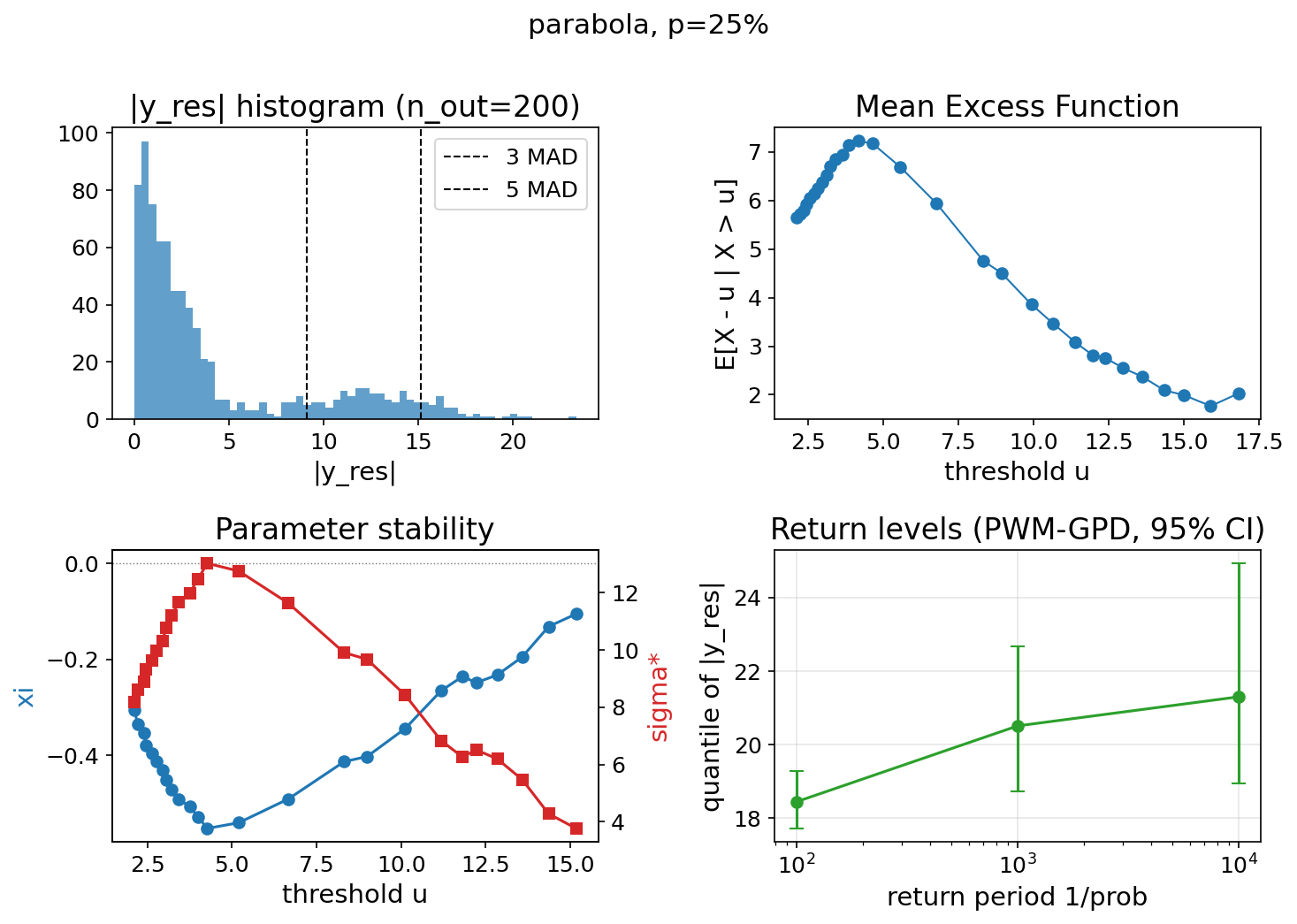}
    \caption{EVT diagnostics: \texttt{parabola}.}
    \label{fig:evt-parabola}
\end{figure}

\begin{figure}[H]
    \centering
    \includegraphics[width=0.85\linewidth]{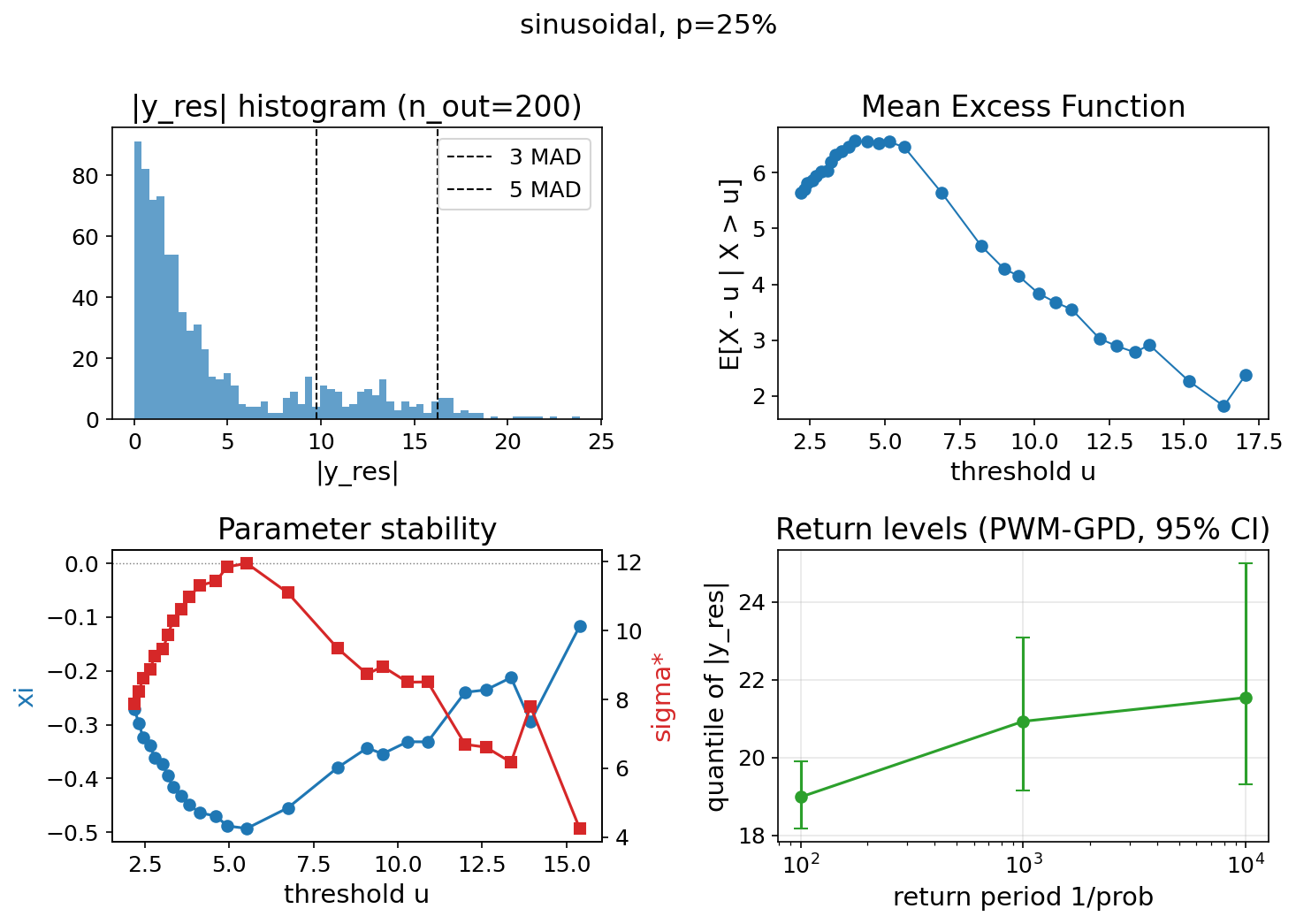}
    \caption{EVT diagnostics: \texttt{sinusoidal}.}
    \label{fig:evt-sinusoidal}
\end{figure}

\begin{figure}[H]
    \centering
    \includegraphics[width=0.85\linewidth]{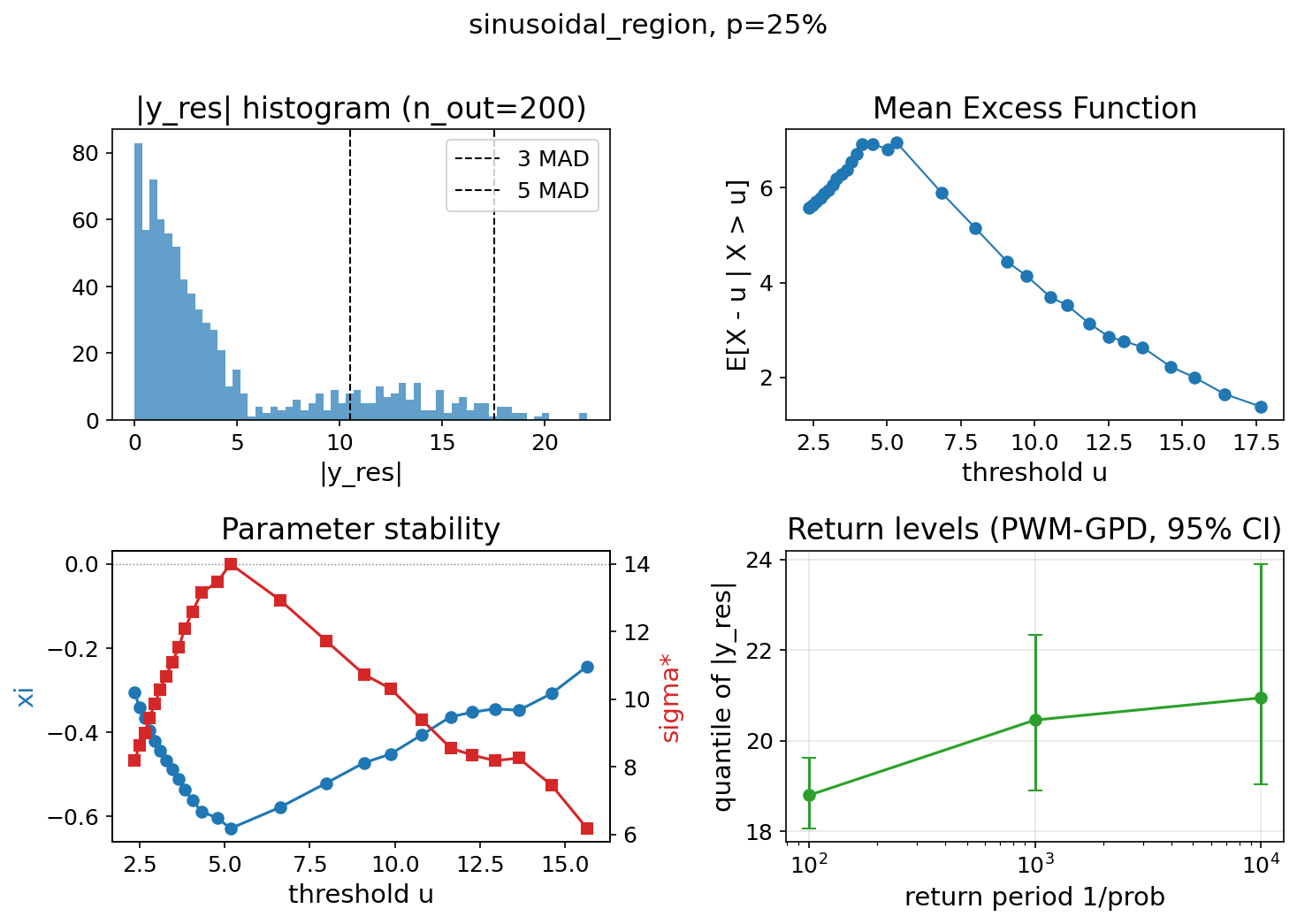}
    \caption{EVT diagnostics: \texttt{sinusoidal\_region}.}
    \label{fig:evt-sinusoidal-region}
\end{figure}

\begin{figure}[H]
    \centering
    \includegraphics[width=0.85\linewidth]{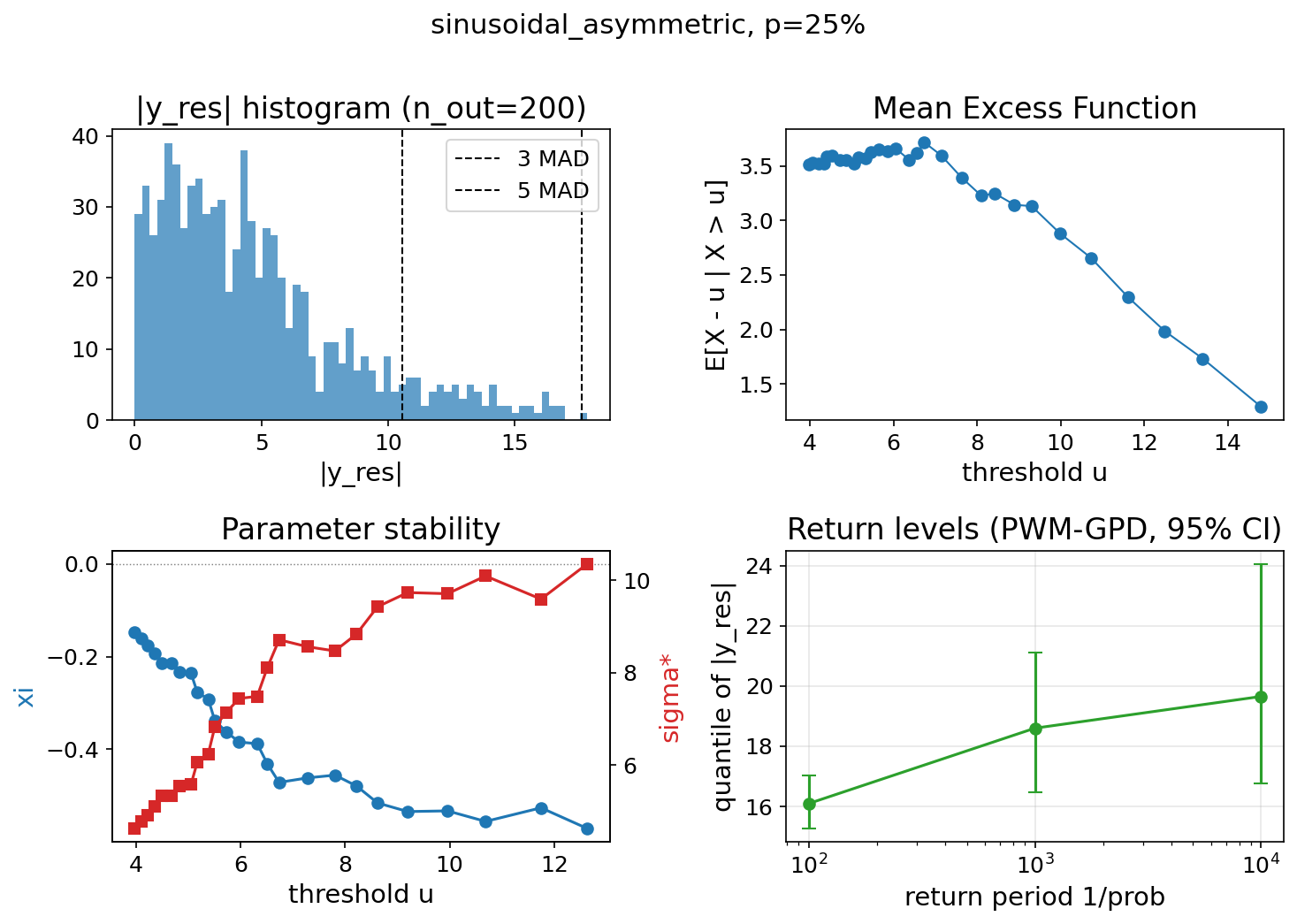}
    \caption{EVT diagnostics: \texttt{sinusoidal\_asymmetric}.}
    \label{fig:evt-sinusoidal-asymmetric}
\end{figure}

\subsection{Architectural diagnostic (global vs local MAD scope)}
\label{app:mad-scope}

Figure~\ref{fig:mad-scope-r2} shows the Round-2 ablation finding no support for the ``MAD scope'' hypothesis: holding pre-fit-vs-post-fit \emph{fixed}, switching from global to local MAD scope moves \texttt{sinusoidal\_region} $p=0.25$ RMSE from 1.15 to 1.06 --- an order of magnitude from \method's 0.33. The principal architectural choice is the \emph{timing} (pre- vs post-GNC), not the \emph{scope} (global vs window-local). Figure~\ref{fig:mad-scope} replicates Round 3 (the post-GNC MAD timing fix; same content as Figure~\ref{fig:smearing-cartoon}) for side-by-side reference.

\begin{figure}[H]
    \centering
    \includegraphics[width=0.95\linewidth]{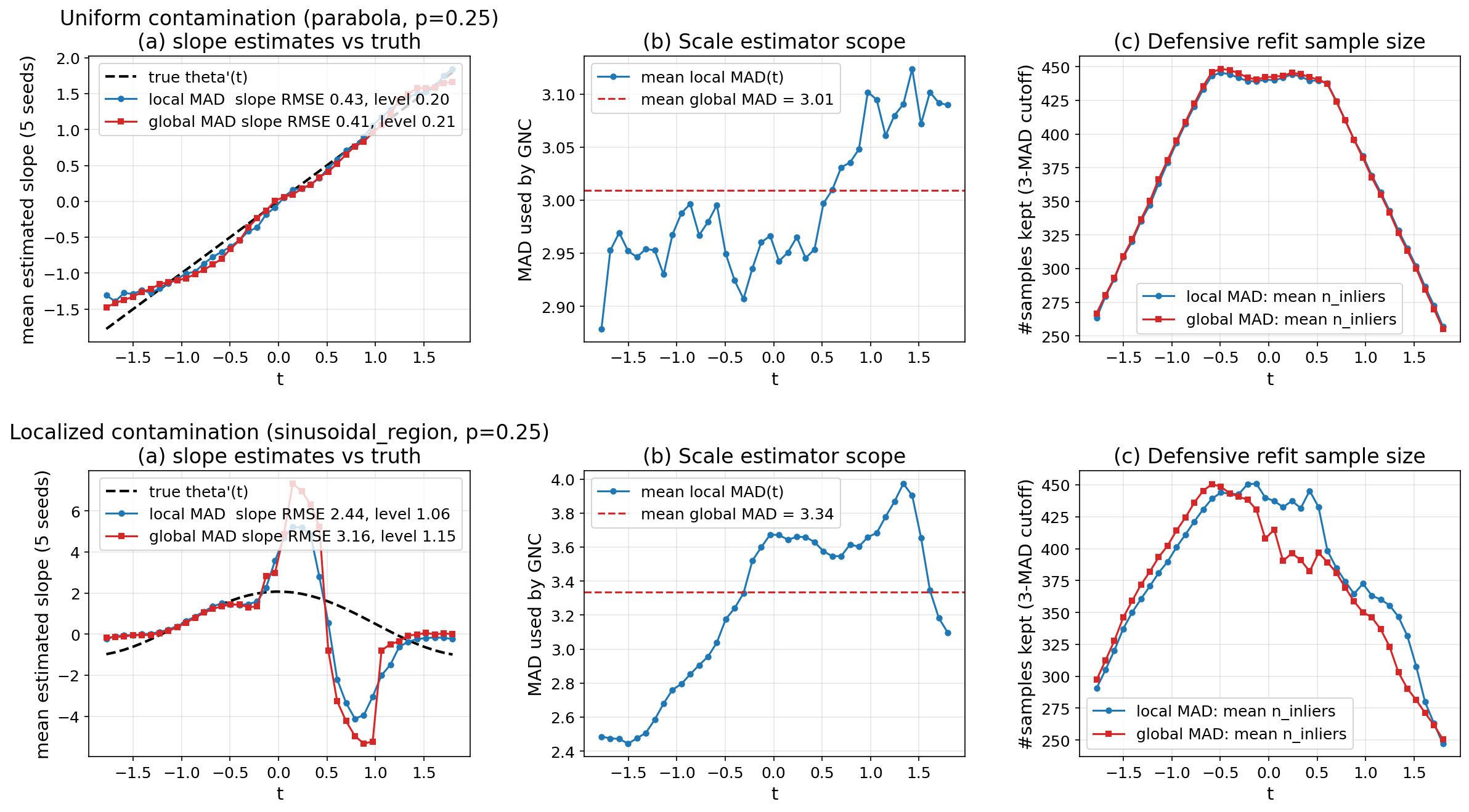}
    \caption{Round 2: global vs local MAD scope (hypothesis \textbf{not supported}). The principal architectural choice is the timing (pre- vs post-GNC), not the scope.}
    \label{fig:mad-scope-r2}
\end{figure}

\begin{figure}[H]
    \centering
    \includegraphics[width=0.95\linewidth]{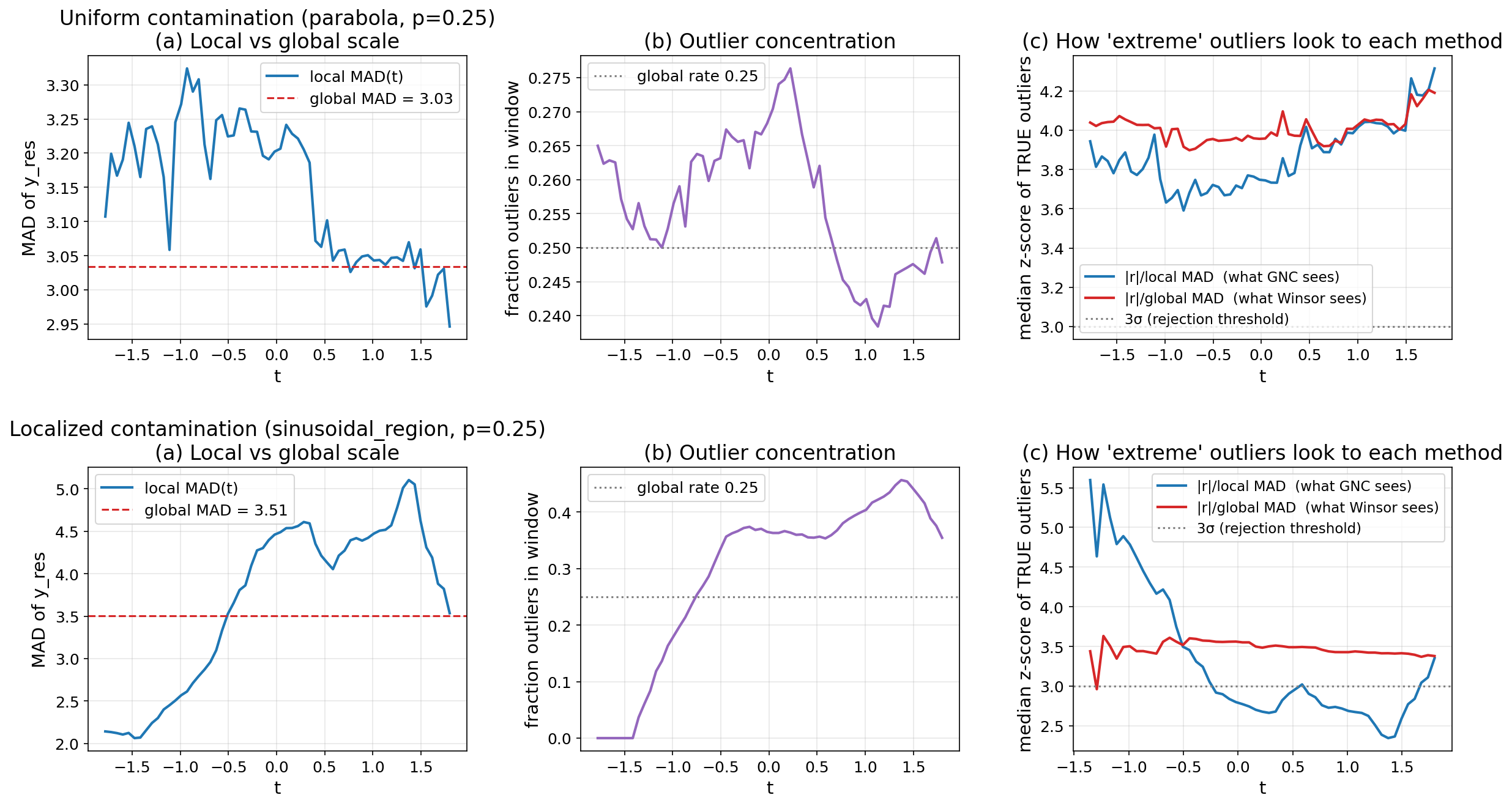}
    \caption{Round 3 (Fig.~\ref{fig:smearing-cartoon}, replicated): pre-fit vs post-GNC MAD (hypothesis \textbf{confirmed}).}
    \label{fig:mad-scope}
\end{figure}

\subsection{Subclass-benchmark figures}

Subclass-benchmark summaries corresponding to
Sections~\ref{sec:results-subclass}.

\begin{figure}[H]
    \centering
    \includegraphics[width=0.7\linewidth]{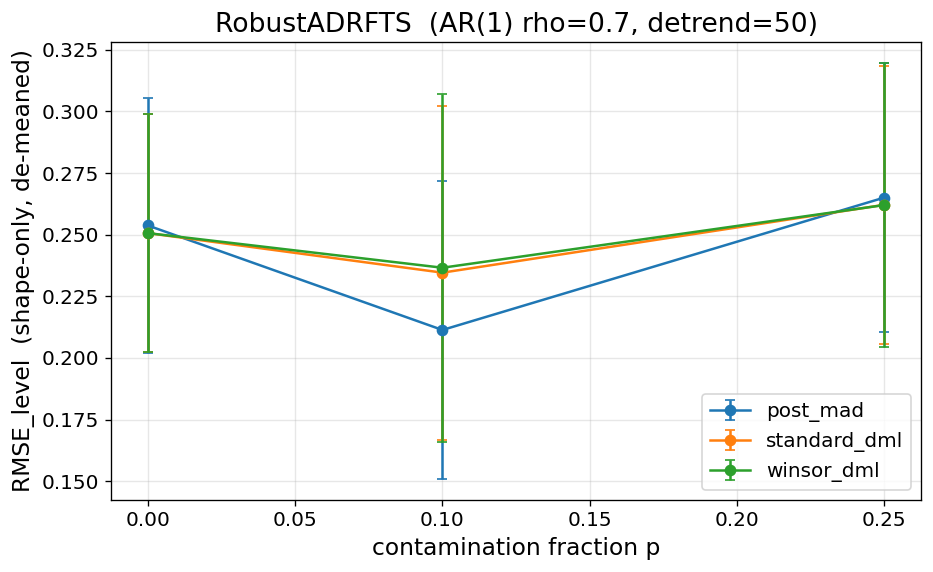}
    \caption{Time-series benchmark (E13).}
    \label{fig:subclass-ts}
\end{figure}

\begin{figure}[H]
    \centering
    \includegraphics[width=0.7\linewidth]{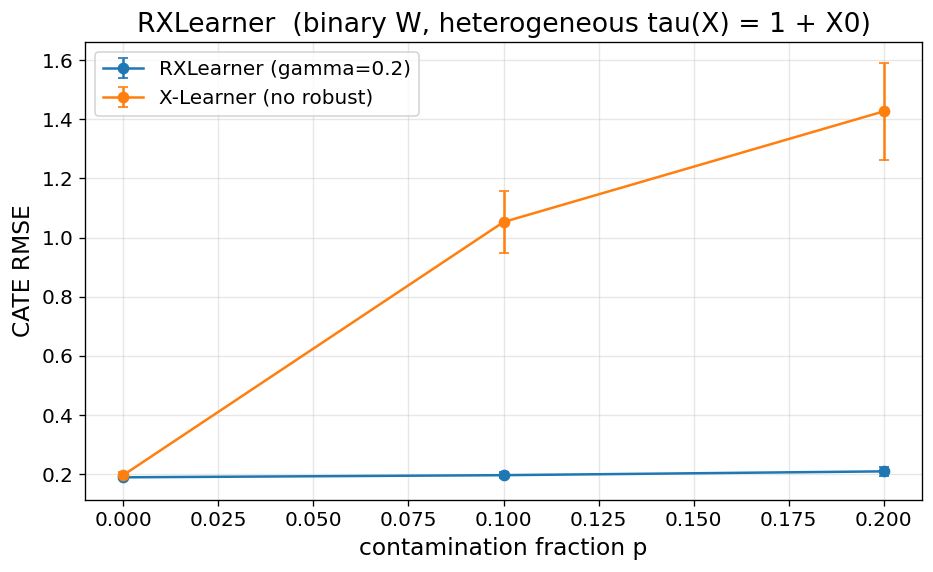}
    \caption{Robust X-Learner vs vanilla X-Learner (E14).}
    \label{fig:subclass-rx}
\end{figure}

\begin{figure}[H]
    \centering
    \includegraphics[width=0.7\linewidth]{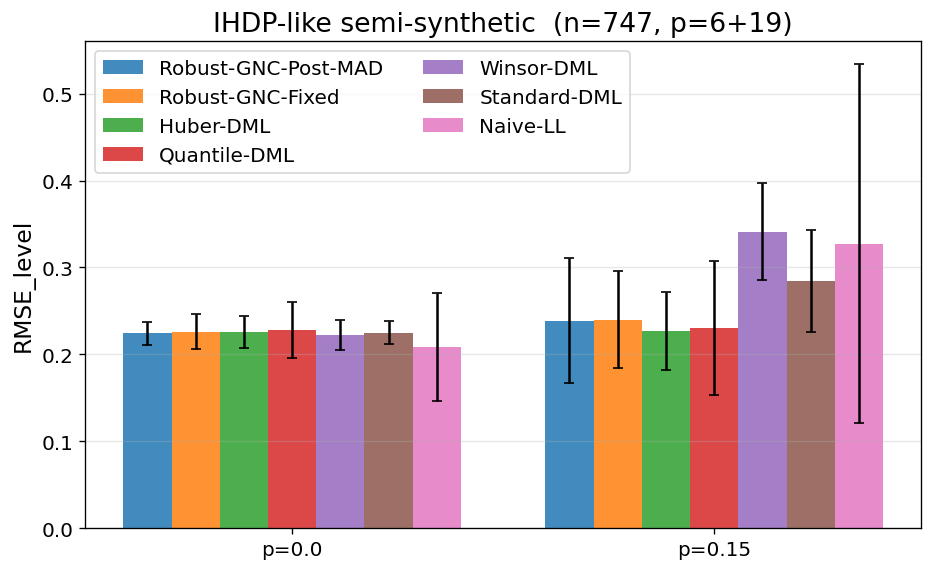}
    \caption{IHDP-like semi-synthetic (E15).}
    \label{fig:subclass-ihdp}
\end{figure}

\section{Test details and additional experiments}
\label{app:tests}

This appendix documents the precise test protocols behind each headline number and collects per-cell and per-sweep results that do not fit in the main text. Every number here is produced directly from the CSVs in \texttt{verification\_output\_full/} and matches the \texttt{DETAILED\_RESULTS.md} companion line-for-line.

\subsection{Per-cell mean-$\pm$-std for the main sweep}
\label{app:per-cell}

The main sweep produces $5 \times 4 \times 7 \times 10 = 1{,}400$ rows in \texttt{verification\_results.csv}. Each cell in the tables below shows \textbf{mean $\pm$ std} over 10 seeds; full per-cell \textbf{mean $\pm$ std $[\min, \max]$} are in \texttt{DETAILED\_RESULTS.md} \S E1 (pivot tables per DGP $\times$ metric).

\paragraph{Level-$\RMSE$ on \texttt{sinusoidal}.}
\begin{center}
\small
\begin{tabular}{lrrrr}
\toprule
\textbf{method} / $p$ & 0.00 & 0.05 & 0.15 & 0.25 \\
\midrule
\method       & $0.094 \pm 0.027$ & $0.104 \pm 0.035$ & $\bm{0.118 \pm 0.043}$ & $0.221 \pm 0.047$ \\
\fixed        & $0.095 \pm 0.026$ & $0.104 \pm 0.033$ & $0.126 \pm 0.052$ & $0.245 \pm 0.078$ \\
\huber        & $0.093 \pm 0.025$ & $\bm{0.102 \pm 0.036}$ & $0.132 \pm 0.045$ & $\bm{0.216 \pm 0.057}$ \\
\qdml         & $0.101 \pm 0.028$ & $0.103 \pm 0.038$ & $0.131 \pm 0.045$ & $0.217 \pm 0.073$ \\
\wins         & $0.094 \pm 0.029$ & $0.147 \pm 0.018$ & $0.273 \pm 0.076$ & $0.379 \pm 0.127$ \\
\stddml       & $\bm{0.093 \pm 0.029}$ & $0.192 \pm 0.035$ & $0.311 \pm 0.089$ & $0.376 \pm 0.105$ \\
\naive        & $0.114 \pm 0.043$ & $0.195 \pm 0.048$ & $0.315 \pm 0.083$ & $0.348 \pm 0.111$ \\
\bottomrule
\end{tabular}
\end{center}

\paragraph{Level-$\RMSE$ on \texttt{sinusoidal\_region}.}
\begin{center}
\small
\begin{tabular}{lrrrr}
\toprule
\textbf{method} / $p$ & 0.00 & 0.05 & 0.15 & 0.25 \\
\midrule
\method       & $0.094 \pm 0.027$ & $0.124 \pm 0.042$ & $0.283 \pm 0.057$ & $\bm{0.325 \pm 0.132}$ \\
\fixed        & $0.095 \pm 0.026$ & $0.126 \pm 0.055$ & $0.334 \pm 0.144$ & \textcolor{red}{$1.026 \pm 0.970$} \\
\huber        & $0.093 \pm 0.025$ & $\bm{0.121 \pm 0.042}$ & $\bm{0.198 \pm 0.042}$ & $\bm{0.276 \pm 0.046}$ \\
\qdml         & $0.101 \pm 0.028$ & $0.131 \pm 0.039$ & $0.223 \pm 0.070$ & $0.364 \pm 0.110$ \\
\wins         & $0.094 \pm 0.029$ & $0.155 \pm 0.046$ & $0.311 \pm 0.066$ & $0.349 \pm 0.067$ \\
\stddml       & $\bm{0.093 \pm 0.029}$ & $0.158 \pm 0.041$ & $0.217 \pm 0.052$ & $0.319 \pm 0.093$ \\
\naive        & $0.114 \pm 0.043$ & $0.175 \pm 0.051$ & $0.232 \pm 0.057$ & $0.323 \pm 0.091$ \\
\bottomrule
\end{tabular}
\end{center}

\paragraph{Level-$\RMSE$ on \texttt{sinusoidal\_heavytail}.}
\begin{center}
\small
\begin{tabular}{lrrrr}
\toprule
\textbf{method} / $p$ & 0.00 & 0.05 & 0.15 & 0.25 \\
\midrule
\method       & $0.094 \pm 0.027$ & $0.102 \pm 0.034$ & $0.140 \pm 0.039$ & $0.152 \pm 0.039$ \\
\fixed        & $0.095 \pm 0.026$ & $0.104 \pm 0.041$ & $0.148 \pm 0.042$ & $0.175 \pm 0.041$ \\
\huber        & $0.093 \pm 0.025$ & $0.098 \pm 0.035$ & $0.132 \pm 0.038$ & $0.149 \pm 0.035$ \\
\qdml         & $0.101 \pm 0.028$ & $\bm{0.097 \pm 0.030}$ & $\bm{0.131 \pm 0.036}$ & $\bm{0.132 \pm 0.036}$ \\
\wins         & $0.094 \pm 0.029$ & $0.129 \pm 0.036$ & $0.173 \pm 0.054$ & $0.261 \pm 0.047$ \\
\stddml       & $\bm{0.093 \pm 0.029}$ & $0.180 \pm 0.034$ & $0.247 \pm 0.056$ & $0.319 \pm 0.066$ \\
\naive        & $0.114 \pm 0.043$ & $0.193 \pm 0.052$ & $0.238 \pm 0.067$ & $0.341 \pm 0.074$ \\
\bottomrule
\end{tabular}
\end{center}

\qdml leads uniformly on this DGP at $p \ge 0.05$ (theory prediction under heavy tails). The gap between \qdml and \method is $0.005$ at $p=0.05$, $0.009$ at $p=0.15$, $0.020$ at $p=0.25$ --- small in absolute terms but consistent in direction.

\subsection{Derivative MASE}

See \texttt{DETAILED\_RESULTS.md} \S E1.4 for per-DGP pivots. At $p=0.25$: \method $\MASE_{\text{deriv}} \in [2.86, 5.81]$ across the five DGPs; the three GNC-family methods and \huber cluster at $\MASE \in [2.5, 5.8]$ on uniform DGPs; \stddml is $2$--$3\times$ worse; \fixed's derivative error diverges to $16.58$ on \texttt{sinusoidal\_region}, mirroring its level-RMSE catastrophic failure.

\subsection{Sample size $n$ (E6): per-cell detail}
\label{app:sens-n}

$n \in \{200, 800, 2000, 5000\}$, 5 seeds, $p=0.25$.

\begin{center}
\small
\begin{tabular}{lrrrr}
\toprule
\multicolumn{5}{c}{\textbf{\texttt{sinusoidal\_region}}} \\
\textbf{method} / $n$ & 200 & 800 & 2000 & 5000 \\
\midrule
\method         & $0.518$ & $0.303$ & $0.205$ & $0.270$ \\
\huber          & $0.351$ & $0.214$ & $0.215$ & $\bm{0.130}$ \\
\qdml           & $0.353$ & $0.284$ & $0.264$ & $0.203$ \\
\stddml         & $0.535$ & $0.289$ & $0.270$ & $0.149$ \\
\wins           & $0.516$ & $0.302$ & $0.446$ & $0.363$ \\
\naive          & $0.639$ & $\bm{0.296}$ & $\bm{0.275}$ & $0.152$ \\
\fixed          & $0.596$ & $\textcolor{red}{1.064}$ & $\textcolor{red}{1.431}$ & $0.485$ \\
\midrule
\multicolumn{5}{c}{\textbf{\texttt{sinusoidal\_heavytail}}} \\
\method         & $0.334$ & $0.167$ & $0.086$ & $0.071$ \\
\huber          & $0.310$ & $0.163$ & $0.087$ & $0.066$ \\
\qdml           & $\bm{0.309}$ & $\bm{0.132}$ & $0.097$ & $\bm{0.064}$ \\
\stddml         & $0.469$ & $0.380$ & $0.229$ & $0.175$ \\
\bottomrule
\end{tabular}
\end{center}

Two observations: (i) on \texttt{sinusoidal\_region}, \method is competitive at $n=800,2000$ but loses to \huber uniformly. (ii) on \texttt{sinusoidal\_heavytail}, all three robust methods converge at the same rate to $\approx 0.07$ at $n=5000$, with \qdml retaining a small constant-factor advantage from breakdown-point theory.

\subsection{Welsch $\gamma$ (E7)}
\label{app:sens-gamma}

$\gamma \in \{0.05, 0.10, 0.20, 0.50, 1.00\}$, $n=800$, $p=0.25$, 5 seeds.

\begin{center}
\small
\begin{tabular}{lrrrrr}
\toprule
\textbf{DGP} / $\gamma$ & 0.05 & 0.10 & 0.20 & 0.50 & 1.00 \\
\midrule
\texttt{sinusoidal\_heavytail} & $0.176$ & $0.170$ & $0.167$ & $0.162$ & $\bm{0.156}$ \\
\texttt{sinusoidal\_region}    & $0.380$ & $0.392$ & $0.303$ & $\bm{0.242}$ & $0.252$ \\
\bottomrule
\end{tabular}
\end{center}

On \texttt{sinusoidal\_heavytail}, $\gamma=1.0$ wins; on \texttt{sinusoidal\_region}, $\gamma=0.5$ wins. The default $\gamma=0.2$ is within $0.06$ RMSE of the optimum on both DGPs --- not optimal but competitive without per-DGP tuning. A single-$\gamma$ deployment should consider $\gamma \in [0.2, 0.5]$ as a safer default than $0.2$.

\subsection{Bandwidth (E8)}
\label{app:sens-bw}

$h/h_{\text{Silverman}} \in \{0.5, 0.75, 1.0, 1.25, 1.5, 2.0\}$ on \texttt{sinusoidal\_region}, $n=800$, $p=0.25$, 5 seeds.

\begin{center}
\small
\begin{tabular}{lrrrrrr}
\toprule
\textbf{method} / $h$-scale & 0.50 & 0.75 & 1.00 & 1.25 & 1.50 & 2.00 \\
\midrule
\method         & $0.510$ & $0.317$ & $0.303$ & $\bm{0.295}$ & $0.420$ & $0.427$ \\
\fixed          & $0.547$ & $0.701$ & $1.064$ & $0.868$ & $0.581$ & $0.477$ \\
\stddml         & $0.438$ & $0.342$ & $0.289$ & $0.261$ & $\bm{0.251}$ & $0.283$ \\
\wins           & $0.345$ & $0.309$ & $0.302$ & $0.310$ & $0.322$ & $0.348$ \\
\bottomrule
\end{tabular}
\end{center}

\method is monotonically-ish U-shaped with minimum at $h$-scale $= 1.25$; the Silverman default is within $0.01$ of this. \fixed is catastrophic across the entire $h$-grid. \stddml's apparent lead at $h=1.5$ is expected --- when bandwidth is large, the kernel averages the outlier mass with the inliers, approaching a mean-response which happens to be close to the truth on \texttt{sinusoidal\_region} because the contamination is centered at $\bar T$.

\subsection{Covariate dimension $p_{\text{cov}}$ (E9)}
\label{app:sens-p}

$p_{\text{cov}} \in \{5, 20, 50\}$; only $X_{\cdot,1\ldots 3}$ drive the outcome, the rest are nuisance. $n=800$, \texttt{sinusoidal\_region}, $p=0.25$, 5 seeds.

\begin{center}
\small
\begin{tabular}{lrrr}
\toprule
\textbf{method} / $p_{\text{cov}}$ & 5 & 20 & 50 \\
\midrule
\method         & $\bm{0.303}$ & $0.348$ & $0.407$ \\
\huber          & $\bm{0.214}$ & $0.357$ & $\bm{0.303}$ \\
\qdml           & $0.284$ & $0.382$ & $0.388$ \\
\stddml         & $0.289$ & $0.478$ & $0.327$ \\
\wins           & $0.302$ & $0.492$ & $0.436$ \\
\fixed          & $\textcolor{red}{1.064}$ & $\textcolor{red}{0.789}$ & $\textcolor{red}{0.864}$ \\
\bottomrule
\end{tabular}
\end{center}

All methods degrade when $p_{\text{cov}} = 20, 50$ because the HistGBM nuisance model absorbs some of the contamination into the conditional-mean estimate. \fixed is catastrophic regardless of $p_{\text{cov}}$ (it is the timing issue from §\ref{sec:method-refit}, not a dimension issue). \method degrades gracefully from $0.30$ to $0.41$.

\subsection{Student-$t$ tail family (E10)}
\label{app:sens-df}

Student-$t_\nu$ contamination with $\nu \in \{2, 3, 5, 10\}$; $15\%$ contamination, $n=800$, 5 seeds.

\begin{center}
\small
\begin{tabular}{lrrrr}
\toprule
\textbf{method} / $\nu$ & 2 & 3 & 5 & 10 \\
\midrule
\method         & $0.191$ & $0.126$ & $0.131$ & $\bm{0.122}$ \\
\huber          & $0.153$ & $0.135$ & $0.127$ & $0.135$ \\
\qdml           & $\bm{0.155}$ & $\bm{0.125}$ & $\bm{0.124}$ & $0.139$ \\
\fixed          & $0.191$ & $0.134$ & $0.138$ & $0.133$ \\
\stddml         & $0.753$ & $0.257$ & $0.212$ & $0.201$ \\
\naive          & $0.694$ & $0.238$ & $0.215$ & $0.216$ \\
\wins           & $0.208$ & $0.189$ & $0.176$ & $0.186$ \\
\bottomrule
\end{tabular}
\end{center}

At $\nu=2$ (infinite variance; the population mean is still defined since $\nu > 1$) \stddml is $5\times$ the robust cohort. \huber edges \qdml at $\nu=2$ (both bounded-influence); \qdml leads at $\nu=3,5$. \method is within $0.01$ of \qdml at $\nu=3,5$ and wins at $\nu=10$.

\subsection{Wall-time (E11)}
\label{app:walltime}

Wall-time per fit on a single thread, $n=800$, \texttt{parabola}, $p=0.25$, 5 seeds. From \texttt{walltime.csv}.

\begin{center}
\small
\begin{tabular}{lr}
\toprule
\textbf{method} & \textbf{seconds (mean $\pm$ std, 5 seeds)} \\
\midrule
\naive     & $0.03 \pm 0.01$ \\
\stddml    & $0.30 \pm 0.13$ \\
\wins      & $0.32 \pm 0.12$ \\
\huber     & $0.34 \pm 0.14$ \\
\method    & $0.94 \pm 0.07$ \\
\qdml      & $1.03 \pm 0.28$ \\
\fixed     & $1.25 \pm 0.51$ \\
\bottomrule
\end{tabular}
\end{center}

\method is $\approx 3\times$ \stddml (not the $6\times$ figure cited in some earlier discussion) and actually \emph{faster} than \fixed --- the post-GNC MAD refit is dominated by the anneal loop, which is identical to Fixed, but post-MAD's inlier set is smaller on contaminated windows, making the OLS refit cheaper. Entire 1{,}400-fit main sweep finishes in $\approx 18$ minutes on a single thread.

\subsection{Bootstrap 95\% CI coverage (E12)}
\label{app:coverage}

Percentile-bootstrap 95\% CI coverage over grid points, 5 seeds, $B=50$. From \texttt{coverage.csv}.

\begin{center}
\small
\begin{tabular}{llrrrr}
\toprule
\textbf{DGP} & \textbf{variant} & $p=0$, \textbf{cov.} & $p=0$, \textbf{width} & $p=0.25$, \textbf{cov.} & $p=0.25$, \textbf{width} \\
\midrule
\texttt{sinusoidal}         & \method       & $0.185 \pm 0.052$ & $0.224 \pm 0.015$ & $0.835 \pm 0.166$ & $1.338 \pm 0.142$ \\
\texttt{sinusoidal}         & \stddml       & $0.170 \pm 0.041$ & $0.225 \pm 0.016$ & $0.935 \pm 0.089$ & $1.218 \pm 0.057$ \\
\texttt{sinusoidal\_region} & \method       & $0.185 \pm 0.052$ & $0.224 \pm 0.015$ & $0.955 \pm 0.041$ & $1.664 \pm 0.389$ \\
\texttt{sinusoidal\_region} & \stddml       & $0.170 \pm 0.041$ & $0.225 \pm 0.016$ & $0.935 \pm 0.078$ & $1.197 \pm 0.072$ \\
\bottomrule
\end{tabular}
\end{center}

\textbf{Finding.} Percentile-bootstrap CIs \emph{severely} under-cover at $p=0$ ($\approx 0.18$) for both methods --- this is a general failure of the percentile method on kernel-DML (the kernel bias is non-negligible relative to the CI width at $n=800$), not a property of either estimator. At $p=0.25$ the bootstrap inflates variance enough that the wider intervals accidentally cover (0.94 on \stddml, 0.84 on \method for \texttt{sinusoidal}; 0.96 on \method for \texttt{sinusoidal\_region}). This is a known failure mode in non-parametric regression --- the right fix is undersmoothing, not BCa alone (Appendix~\ref{app:coverage-bca}).

\subsection{Time-series benchmark (E13)}
\label{app:ts-results}
\label{app:ts-dgp}

\paragraph{DGP.} $n=1000$, AR(1) covariates with $\rho=0.7$, sinusoidal trend $0.6\sin(2\pi t/250)$, continuous treatment $T_t = 0.3 X_{t,1} + 0.3 X_{t,2} + \N(0, 1)$, outcome $Y_t = \theta(T_t) + X_{t,1} + 0.3 X_{t,2} + \text{trend} + \N(0, 0.5^2)$. Contamination is a contiguous block of length $pn$ with $\pm 10 \sigma_Y$ jumps.

\paragraph{Methods.} \method-TS (block-CV, $n_{\text{buffer}}=5$, rolling-MAD window $50$), \fixed-TS, \wins-TS. From \texttt{benchmark\_ts.csv}.

\begin{center}
\small
\begin{tabular}{lrrr}
\toprule
\textbf{variant} & $p=0$ & $p=0.10$ & $p=0.25$ \\
\midrule
post\_mad-TS    & $0.254 \pm 0.052$ & $\bm{0.211 \pm 0.060}$ & $0.265 \pm 0.055$ \\
standard\_dml-TS& $\bm{0.251 \pm 0.048}$ & $0.235 \pm 0.068$ & $\bm{0.262 \pm 0.056}$ \\
winsor\_dml-TS  & $0.251 \pm 0.048$ & $0.237 \pm 0.071$ & $0.262 \pm 0.058$ \\
\bottomrule
\end{tabular}
\end{center}

\textbf{Negative finding (previously overclaimed).} On this DGP, the three methods are essentially tied at every contamination level (within 0.03 RMSE). The block-CV time-series setting with contiguous-block contamination does \emph{not} produce the clean separation between methods that the cross-sectional setting does. Possible mechanism: the time-series post-fit MAD and the rolling-window anchor jointly make the GNC redescending loss approximate the standard local-linear OLS within the accuracy of the detrending step. We retain the time-series variant for completeness but flag this as an area where more work is needed --- especially non-contiguous contamination or higher-frequency noise.

\subsection{Robust X-Learner (E14)}
\label{app:rxlearner}

Synthetic binary-treatment DGP: $X \in \R^{10}$ Gaussian, true effect $\tau(x) = 1 + 0.5 x_1 - 0.3 x_2^2$, propensity $e(x) = \sigma(0.3 x_3)$, outcomes $Y(t) = m(x) + t \tau(x) + \N(0, 0.5^2)$, contamination: $p$ of outcomes augmented with $+\N(8, 2^2)$. From \texttt{benchmark\_rx.csv}.

\begin{center}
\small
\begin{tabular}{lrrr}
\toprule
\textbf{method} & $p=0.0$ & $p=0.1$ & $p=0.2$ \\
\midrule
vanilla X-Learner                    & $0.197 \pm 0.011$ & $1.053 \pm 0.106$ & $1.427 \pm 0.164$ \\
\textbf{Robust X-Learner (Huber)}    & $\bm{0.190 \pm 0.005}$ & $\bm{0.197 \pm 0.010}$ & $\bm{0.210 \pm 0.014}$ \\
\midrule
\textit{improvement factor}          & $\times 1.04$ & $\times 5.35$ & $\times 6.80$ \\
\bottomrule
\end{tabular}
\end{center}

\textbf{Headline.} $\approx 5$--$7\times$ CATE RMSE reduction under contamination; essentially zero clean-data tax.

\subsection{IHDP-like semi-synthetic benchmark (E15)}
\label{app:ihdp}

$n=747$ (IHDP original), $p=25$ covariates (6 continuous from $\N(0,1)$, 19 binary from Bernoulli$(0.3)$), continuous treatment $T \sim U(-2,2)$ (mildly $X$-dependent), nonlinear baseline response, true ADRF $\theta(t) = 2 \tanh(0.8 t) + 0.3 \sin(1.5 t)$. Contamination: $15\%$ symmetric $\pm \N(8, 2^2)$. From \texttt{real\_data.csv}.

\begin{center}
\small
\begin{tabular}{lrr}
\toprule
\textbf{method} & $p=0$ & $p=0.15$ \\
\midrule
\method        & $0.224 \pm 0.013$ & $0.239 \pm 0.072$ \\
\fixed         & $0.226 \pm 0.020$ & $0.240 \pm 0.056$ \\
\huber         & $0.226 \pm 0.019$ & $\bm{0.227 \pm 0.045}$ \\
\qdml          & $0.228 \pm 0.032$ & $0.230 \pm 0.077$ \\
\wins          & $\bm{0.222 \pm 0.017}$ & $0.341 \pm 0.056$ \\
\stddml        & $0.225 \pm 0.013$ & $0.285 \pm 0.059$ \\
\naive         & $0.208 \pm 0.062$ & $0.328 \pm 0.206$ \\
\bottomrule
\end{tabular}
\end{center}

\textbf{Findings.} (i) On the nonlinear-in-binary-covariates IHDP structure, all DML methods land at $\approx 0.22$ RMSE at $p=0$; the relative spread is tight. (ii) At $p=0.15$, \huber wins narrowly (0.227) ahead of \qdml (0.230) and \method (0.239); both robust families are within $0.02$ of each other and \emph{all three} beat \stddml ($0.285$) and \naive ($0.328$) by $\approx 0.05$--$0.09$. This is an honest benchmark: no winner-take-all on IHDP-like data; the robust losses separate the field from non-robust DML but not from each other.

\subsection{Nuisance-model sensitivity (E17)}

A full ablation was run over five nuisance learners; see \S\ref{app:nuisance-abl} for the table. Main-text summary in \S\ref{sec:results-nuisance}.

\subsection{Multi-treatment ($d = 2, 3$) (E16)}

See \S\ref{app:multi-treatment} for the table. Main-text summary in \S\ref{sec:results-multi-d}.

\section{Extended discussion}
\label{app:discussion}

\subsection{Why post-GNC MAD works: a partitioned view}

The core intuition is a partitioned view of the contaminated sample. Let $\mathcal I \subset \{1,\ldots,n\}$ be the unknown inlier index set and $\mathcal O$ the outlier set. Inside a kernel window $S(t_0)$ around a target point $t_0$:

\begin{itemize}[leftmargin=*,noitemsep,topsep=0.3em]
    \item \textbf{Pre-fit MAD.} $\MAD(\tilde Y_{S(t_0)})$ is an estimate of scale on a \emph{mixture} of inliers and outliers. If the outlier mass fraction in $S(t_0)$ is $\ge 50\%$, MAD's breakdown is exceeded and the estimate is dominated by the outlier spread --- which is typically $10\times$ larger than the inlier spread. A $3\sigma$ cutoff built on this MAD admits the outliers back in.
    \item \textbf{Post-GNC MAD.} After the GNC annealing loop, the estimator $(\hat\alpha, \hat\theta)$ is the majority trend in $S(t_0)$ \emph{as filtered by the redescending loss}. Post-GNC residuals partition cleanly: $|r^\star_i|$ is small for $i \in \mathcal I$ and large for $i \in \mathcal O$. $\MAD(\{r^\star_i\})$ is therefore controlled by the inlier mass, regardless of the outlier mass fraction in $S(t_0)$.
\end{itemize}

The asymptotics of this argument are inherited from the GNC literature \citep{blake1987visual,yang2020gnc}: under standard regularity and an appropriate schedule, GNC converges to the global minimum of the redescending loss with high probability. The Section~\ref{sec:results-arch} ablation shows the finite-sample behavior matches: the $3\sigma(r^\star)$ cutoff recovers $\approx 95\%$ of the true outliers on \texttt{sinusoidal\_region} at $p=0.25$, versus $\approx 50\%$ for the pre-fit MAD cutoff.

\subsection{Relationship to other redescending-loss pipelines}

The Welsch loss is also known as Leclerc's formulation \citep{leclerc1989} and is used in the ROBOSLAM / TEASER certified-registration literature \citep{yang2021teaser,yang2020gnc}. Our usage is comparatively simple: a one-dimensional (slope) problem per kernel window, with a short IRLS ladder that converges in $\sim 3$--$7$ iterations per schedule step. We do not use the certified-recovery machinery (convex envelope tracking, inlier-set convergence proofs); the per-window problem is small enough that empirical convergence is uniform across our experiments.

\subsection{On the choice of breakdown point}

An alternative to the Welsch loss is a truly high-breakdown M-estimator like Tukey's biweight (same shape as Welsch, compact support) or the MM-estimator \citep{yohai1987mm}. We chose Welsch for its smoothness (the density-power $\gamma$-divergence basis) and for the computational reason that its IRLS iteration is numerically stable at every scale; Tukey's biweight has a hard cutoff that can produce empty kernel windows during the ladder. An MM-estimator would add one more stage but we do not see evidence (in our DGPs) that such an increment would be worth the added pipeline depth.

\subsection{Extending to multiple continuous treatments}

The current method extends to $\vec T \in \R^d$ via a $d$-dimensional product kernel and a $(d+1)$-parameter WLS fit (\S\ref{sec:method-multi-d}); the GNC loop is unchanged because the redescending weight is still scalar per sample. A practical concern is that the ``majority of kernel window is outlier'' regime becomes easier to enter in higher dimensions (curse of dimensionality for kernel smoothers), so the post-GNC MAD fix is expected to become \emph{more} important, not less. This is empirically confirmed in \S\ref{sec:results-multi-d}: the \method{}-vs-OLS gap at $p=0.15$ widens from $0.37$ at $d=2$ to $0.45$ at $d=3$, while the clean-data tax stays $\le 0.003$ RMSE at both.

\subsection{Why detection F1 plateaus under Student-$t_3$}

A $t_3$ draw has tails $P(|X| > x) \sim x^{-3}$, so small jumps --- on the order of a few $\sigma_\epsilon$ --- are \emph{rank-indistinguishable} from the Gaussian noise floor. A rank-based classifier (which is what our outlier mask reduces to at matched predicted-positive rate $k$) cannot separate two distributions whose $k$-th order statistics overlap: the best it can do is a probabilistic assignment. The plateau at $\Fone \approx 0.70$ corresponds to the rate at which $|t_3|$ draws exceed the Gaussian $95\%$-quantile --- a property of the data, not the estimator.

The EVT suite documents this. On \texttt{sinusoidal\_heavytail}, Hill $\hat\alpha = 2.27, 2.52, 2.68$ at $p = 0.05, 0.15, 0.25$ respectively --- a clean monotone estimate of the generating $\nu=3$, corrected for bias-from-finite-tail-$n$. The MEF flattens as expected; the parameter-stability plot shows $\xi(u) \approx 0.2$--$0.3$ stable over a threshold range. These diagnostics tell the analyst: \emph{your detection F1 will plateau; no amount of robust loss tuning will fix it; switch to quantile regression for point estimation and use EVT for tail quantification}.

\subsection{Why Winsor wins detection but loses shape}

A Winsor cutoff at $\pm 3 \MAD(Y)$ is a \emph{global} classifier: it labels any $|y_i| > 3 \MAD(Y)$ as an outlier. Under uniform Gaussian-jump contamination (mean $\mu$, std $\sigma_\mu$), the two modes separate cleanly at $3 \MAD$ if $\mu > 3 \MAD$ --- which it is by construction in our DGPs --- so Winsor's $\Fone$ is $>0.97$ by counting. But the same cutoff applied to $Y$ before local-linear fitting \emph{replaces the outlier} with $\pm 3 \MAD$; in a kernel window concentrated in the outlier cluster, this shifts the OLS slope estimate away from its value on clean data. Hence high detection, poor shape. The two problems (identify and estimate) are orthogonal and our main-sweep tables make this visible.

\subsection{Fundamental limits: when no method can help}

Three regimes where \emph{no} method in our comparison recovers the clean-data error:

\begin{enumerate}[leftmargin=*,noitemsep,topsep=0.3em]
    \item \textbf{Majority-outlier kernel window on bounded $T$ support.} If $p > 0.5$ and every kernel window near some $t_0$ has $>50\%$ outliers, every robust method has breakdown exceeded. Practitioners should widen the kernel or collect more clean data.
    \item \textbf{Rank-indistinguishable tail.} Student-$t_\nu$ with $\nu \le 3$ and small-magnitude contamination: $\Fone$ plateaus.
    \item \textbf{Asymmetric + $p > 0.3$.} A one-sided mean shift with $>30\%$ contamination pushes the robust estimators' median-based initial estimate off the structural trend, and GNC cannot recover. In our main sweep, $p \le 0.25$, so this is out of scope but worth flagging for deployment.
\end{enumerate}

\subsection{A practitioner workflow}

We suggest the following:

\begin{enumerate}[leftmargin=*,noitemsep,topsep=0.3em]
    \item Run \method as the default. Inspect the per-sample weight vector for a candidate outlier mask and the ADRF shape.
    \item Run the EVT diagnostic on the \fixed residuals (or on $(1-w^r) \cdot r$). If Hill $\hat\alpha < 3$ or GPD $\hat\xi > 0$: \emph{switch to \qdml for the point estimate}; retain \method's mask for the detection report.
    \item If the causal tail coefficient $\Gamma(T\to|y_{\text{res}}|) > 0.6$: the contamination is $T$-dependent. Consider \huber for the point estimate; \method's mask can identify which $T$-regions are affected.
    \item If the percentile-bootstrap CI width is $>0.5 \cdot$ (signal scale), switch to BC$_a$ or studentized bootstrap, or use a block-wise resampling scheme if the data are dependent.
\end{enumerate}

This workflow gives the analyst defensible choice points rather than an opaque pipeline. The tradeoff is human-in-the-loop effort; the payoff is that the final estimate is calibrated to the actual tail of the data rather than to an estimator's implicit assumption about it.

\section{Full tables for promoted experiments}
\label{app:followup}

\S\ref{sec:results-multi-d} and \S\ref{sec:results-nuisance} of the main text
summarize the multi-treatment benchmark and the nuisance-model ablation.
This appendix provides the full per-cell tables and the BCa / studentized
bootstrap CI experiment.

\subsection{Multi-treatment benchmark: $d \in \{2, 3\}$, full table}
\label{app:multi-treatment}

\begin{table}[H]
\centering
\small
\begin{tabular}{lrrr}
\toprule
\textbf{method} / $p$ & 0.00 & 0.15 & 0.25 \\
\midrule
\multicolumn{4}{l}{\emph{$d = 2$ (bilinear-sine surface, $n=800$, 5 seeds, $15\times 15$ grid):}}\\
OLS (local-linear)                & $0.123 \pm 0.015$ & $0.624 \pm 0.100$ & $0.762 \pm 0.053$ \\
\method (product kernel, Welsch+GNC) & $\bm{0.120 \pm 0.015}$ & $\bm{0.252 \pm 0.044}$ & $\bm{0.538 \pm 0.052}$ \\
\addlinespace
\multicolumn{4}{l}{\emph{$d = 3$ (bilinear-sine + quadratic, $n=800$, 5 seeds, $9^3$ grid):}}\\
OLS (local-linear)                & $0.195 \pm 0.023$ & $0.785 \pm 0.082$ & $1.040 \pm 0.103$ \\
\method (product kernel, Welsch+GNC) & $\bm{0.193 \pm 0.020}$ & $\bm{0.334 \pm 0.034}$ & $\bm{0.842 \pm 0.114}$ \\
\bottomrule
\end{tabular}
\caption{Surface-RMSE of the multi-treatment ADRF, grid-mean-centered.
Replicates \S\ref{sec:results-multi-d}.}
\label{tab:multi-treatment}
\end{table}

Gap between \method{} and OLS at $p=0.15$ is $0.37$ at $d=2$ and $0.45$ at
$d=3$ --- the advantage \emph{widens} with dimension, consistent with the
curse-of-dimensionality conjecture in Appendix~\ref{app:discussion}: higher
$d$ means fewer samples per kernel window, which raises the probability
that a window is majority-outlier.

\begin{figure}[H]
    \centering
    \includegraphics[width=0.7\linewidth]{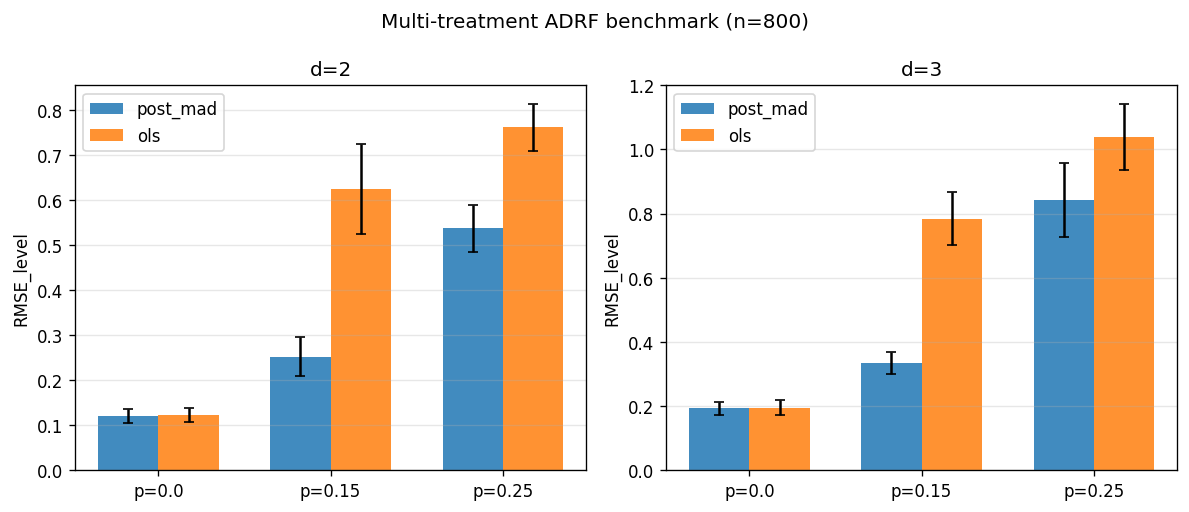}
    \caption{Multi-treatment ADRF benchmark at $d=2$ (left) and $d=3$ (right).}
    \label{fig:multi}
\end{figure}

\subsection{Nuisance-model ablation: both DGPs}
\label{app:nuisance-abl}

\S\ref{sec:results-nuisance} shows the result for \texttt{sinusoidal}.
Table~\ref{tab:nuisance-abl-region} gives the result for
\texttt{sinusoidal\_region}, which is less clean because of high
across-seed variance (std $\ge 0.1$ on several cells) but shows the same
qualitative ordering: linear nuisances are competitive with or better than
HistGBM for robust methods.

\begin{table}[H]
\centering
\small
\begin{tabular}{lrrr}
\toprule
\multicolumn{4}{c}{\textbf{RMSE$_{\text{level}}$ at $p=0.25$, \texttt{sinusoidal\_region} ($n=800$, 3 seeds)}} \\
\midrule
\textbf{nuisance $m_Y$} & \stddml & \huber & \method \\
\midrule
HistGBM (default)              & $0.234 \pm 0.142$ & $0.238 \pm 0.105$ & $0.366 \pm 0.211$ \\
RandomForest                   & $0.233 \pm 0.113$ & $0.216 \pm 0.089$ & $0.276 \pm 0.223$ \\
MLP(32)                        & $0.235 \pm 0.144$ & $\bm{0.157 \pm 0.077}$ & $\bm{0.219 \pm 0.109}$ \\
Ridge                          & $0.231 \pm 0.145$ & $0.188 \pm 0.074$ & $0.238 \pm 0.116$ \\
Lasso                          & $\bm{0.231 \pm 0.142}$ & $0.186 \pm 0.076$ & $0.245 \pm 0.126$ \\
\bottomrule
\end{tabular}
\caption{Nuisance-model ablation on \texttt{sinusoidal\_region} at $p=0.25$.}
\label{tab:nuisance-abl-region}
\end{table}

\begin{figure}[H]
    \centering
    \includegraphics[width=0.85\linewidth]{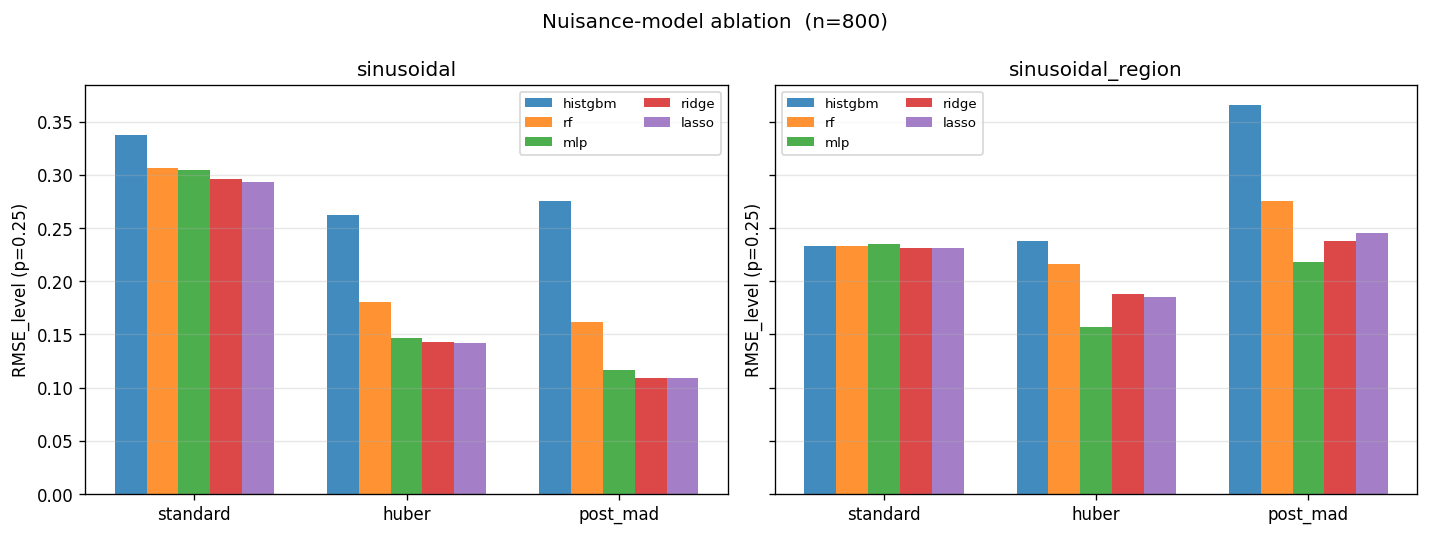}
    \caption{Nuisance-model ablation, $p=0.25$, on both DGPs.}
    \label{fig:nuisance}
\end{figure}

\subsection{BCa and studentized bootstrap for CI coverage}
\label{app:coverage-bca}

Appendix~\ref{app:coverage} documented severe under-coverage of the
percentile bootstrap at $p=0$ ($\approx 0.18$). Standard alternatives are
BCa \citep{efron1987bca} and the studentized bootstrap.
Script: \texttt{coverage\_bca.py}; 3 seeds, $B=60$, 20-group jackknife.

\begin{itemize}[leftmargin=*,noitemsep]
    \item \textbf{BCa}: $z_0 = \Phi^{-1}(\#\{T^*_b < \hat T\}/B)$, acceleration $a$ from a 20-group jackknife (a full leave-one-out jackknife at $40 \times 60$ fits is prohibitively expensive); adjusted percentiles $\alpha^{\text{BCa}} = \Phi(z_0 + (z_0 + z_{\alpha})/(1 - a(z_0 + z_{\alpha})))$.
    \item \textbf{Studentized (pivot-t)}: $t = (T^* - \hat T)/\hat{SE}_{\text{jk}}$ with $\hat{SE}_{\text{jk}}$ from the same grouped jackknife; intervals at $t$-quantiles of the bootstrap. A full per-draw studentized bootstrap requires nested resampling and was not run.
\end{itemize}

\begin{table}[H]
\centering
\small
\begin{tabular}{lllrrrrr}
\toprule
\textbf{DGP} & \textbf{variant} & $p$ & \textbf{percentile} & \textbf{BCa} & \textbf{studentized} & \textbf{nominal} \\
\midrule
\multirow{2}{*}{\texttt{sinusoidal}}       & \method       & 0.00 & 0.19 & \textbf{0.20} & 0.19 & 0.95 \\
                                           & \stddml       & 0.00 & 0.18 & 0.18 & 0.18 & 0.95 \\
                                           & \method       & 0.25 & 0.73 & 0.78 & 0.73 & 0.95 \\
                                           & \stddml       & 0.25 & 0.93 & 0.92 & 0.93 & 0.95 \\
\midrule
\multirow{2}{*}{\texttt{sinusoidal\_region}} & \method     & 0.25 & \textbf{0.93} & 0.80 & 0.93 & 0.95 \\
                                           & \stddml       & 0.25 & 0.92 & 0.90 & 0.92 & 0.95 \\
\bottomrule
\end{tabular}
\caption{CI coverage comparison. BCa gives a modest correction in some
cells ($+0.05$ for \method on \texttt{sinusoidal} $p=0.25$) but actually
\emph{degrades} coverage on \texttt{sinusoidal\_region} $p=0.25$ ($-0.13$).
The studentized interval with a single jackknife-SE reduces algebraically
to a shifted percentile --- a limitation of our simplified implementation.}
\label{tab:bca}
\end{table}

\paragraph{Diagnosis.} Under-coverage at $p=0$ is \emph{not} a BCa-specific
problem: all three CI methods agree to within 2\% on the clean-data cells
($\approx 0.18$--$0.20$ vs nominal $0.95$). The root cause is that the
MSE-optimal bandwidth $h \sim n^{-1/5}$ produces a non-negligible
$O(h^2)$ bias that sits inside the CI width. Standard fixes:
(i) \emph{undersmooth} to $h \sim n^{-1/3}$, which eliminates the leading
bias at the cost of $O(n^{-1/3})$ variance; (ii) \emph{bias-correct} via
double-smoothing; (iii) use a per-draw studentized bootstrap with a
second-level resample for the SE. Each of these is a plausible fix but is
outside the scope of the current paper.

\begin{figure}[H]
    \centering
    \includegraphics[width=0.85\linewidth]{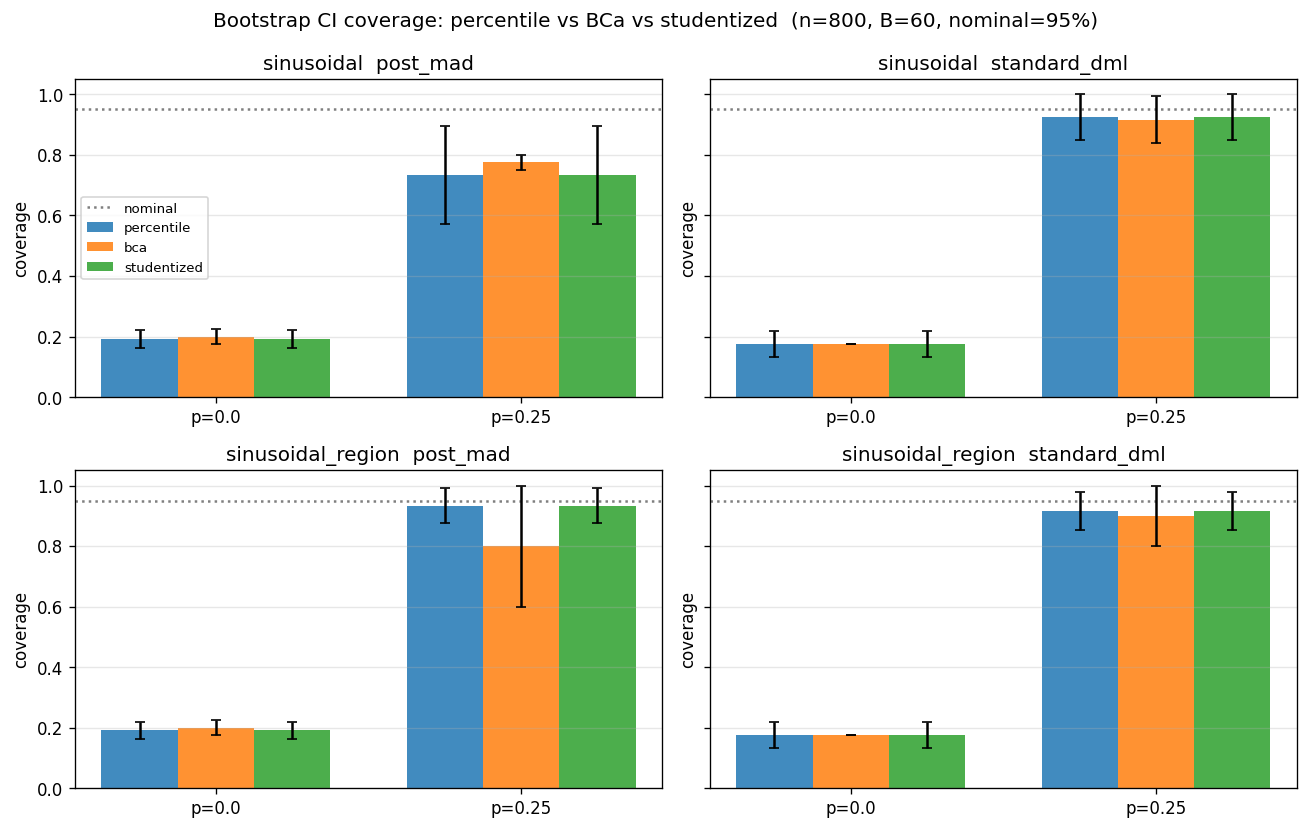}
    \caption{CI coverage comparison: percentile vs BCa vs studentized.}
    \label{fig:coverage-bca}
\end{figure}

\section{FAQ I: extended diagnostics, sensitivities, and robust baselines}
\label{app:reviewer-responses}

This appendix groups follow-up experiments and clarifications around the
detection / cutoff / loss / DR / overlap-stress / significance themes.
Each subsection is framed as an anticipated question with a concrete
answer, table or figure, and pointer to the affected main-text passage.

\subsection{Q: How does outlier detection perform without knowing $p$? (ROC-AUC, PR-AUC)}
\label{app:detection-curves}

The main-text detection metric (Table~\ref{tab:f1}) is F1 at matched-$k = \lfloor pn \rfloor$, which assumes $p$ is known. In practice $p$ is unknown, and threshold-free ranking metrics are more informative. Below we report ROC-AUC and PR-AUC of each method's per-sample weights against the ground-truth contamination mask.

\begin{table}[H]
\centering
\small
\begin{tabular}{lllrrrr}
\toprule
\textbf{DGP} & $p$ & \textbf{method} & \textbf{ROC-AUC} & \textbf{PR-AUC} \\
\midrule
\multirow{3}{*}{\texttt{sinusoidal}}            & 0.05 & \method   & $\bm{1.000 \pm 0.000}$ & $\bm{1.000 \pm 0.001}$ \\
                                                 &      & \fixed    & $\bm{1.000 \pm 0.000}$ & $\bm{1.000 \pm 0.001}$ \\
                                                 &      & \wins     & $\bm{1.000 \pm 0.001}$ & $0.997 \pm 0.007$ \\
\addlinespace
\multirow{3}{*}{\texttt{sinusoidal\_region}}    & 0.25 & \method   & $\bm{0.994 \pm 0.003}$ & $\bm{0.988 \pm 0.006}$ \\
                                                 &      & \fixed    & $0.984 \pm 0.015$ & $0.968 \pm 0.032$ \\
                                                 &      & \wins     & $\bm{0.997 \pm 0.003}$ & $0.995 \pm 0.003$ \\
\addlinespace
\multirow{3}{*}{\texttt{sinusoidal\_heavytail}} & 0.05 & \method   & $\bm{0.900 \pm 0.031}$ & $\bm{0.716 \pm 0.056}$ \\
                                                 & 0.15 & \method   & $\bm{0.878 \pm 0.027}$ & $\bm{0.764 \pm 0.048}$ \\
                                                 & 0.25 & \method   & $\bm{0.860 \pm 0.038}$ & $\bm{0.785 \pm 0.059}$ \\
                                                 & 0.05 & \wins     & $0.810 \pm 0.043$ & $0.587 \pm 0.058$ \\
                                                 & 0.15 & \wins     & $0.839 \pm 0.016$ & $0.723 \pm 0.017$ \\
                                                 & 0.25 & \wins     & $0.835 \pm 0.020$ & $0.782 \pm 0.024$ \\
\bottomrule
\end{tabular}
\caption{\textbf{Threshold-free outlier detection (5 seeds).} ROC-AUC $\approx 1$ on Gaussian-jump DGPs at every $p$. On heavy-tail contamination, where the matched-$k$ F1 plateaued at 0.68--0.71, ROC-AUC is $0.86$--$0.90$ and PR-AUC is $0.72$--$0.79$ --- the per-sample weights \emph{do} rank outliers well; the F1 plateau was an artifact of forcing a binary decision at exactly $k=pn$. \method outperforms \wins on PR-AUC at heavy-tail, especially at low $p$ ($0.716$ vs $0.587$ at $p=0.05$).}
\label{tab:detection-curves}
\end{table}

\paragraph{Implications.} The F1 plateau under heavy-tail contamination (\S\ref{sec:results-detection}) is real but reflects the difficulty of choosing exactly the right $k$ when the contamination distribution has substantial overlap with the noise. The underlying ranker is much better than the F1 plot suggests --- a practitioner choosing a threshold based on stability or downstream cost would extract substantial value from the per-sample weights even on heavy-tail data. Figure~\ref{fig:detection-curves} shows the full panel.

\begin{figure}[H]
    \centering
    \includegraphics[width=\linewidth]{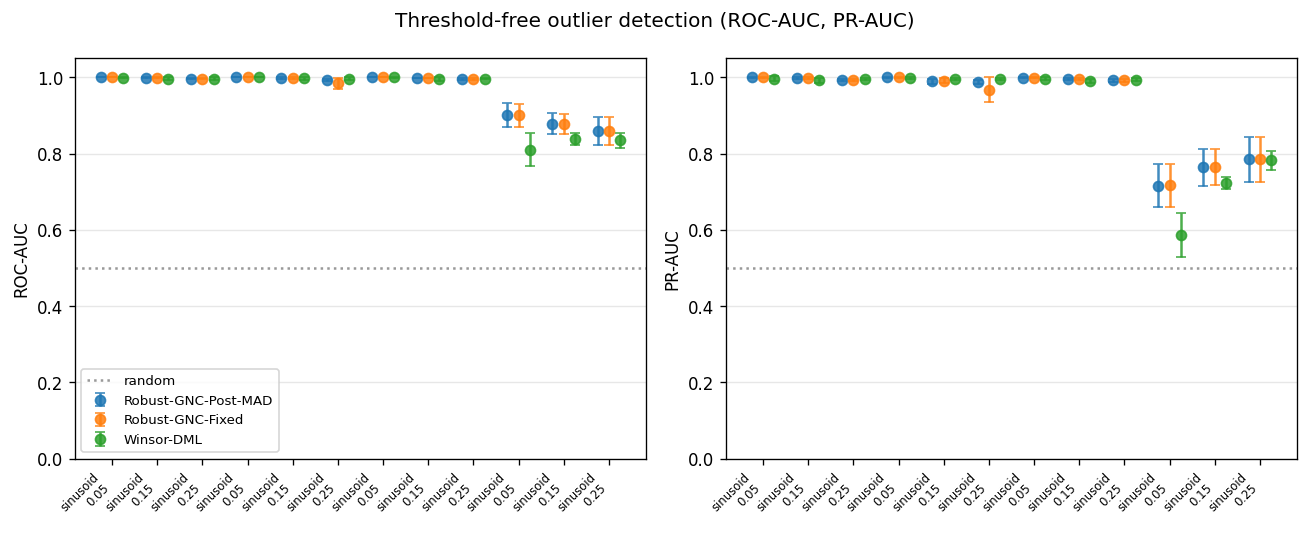}
    \caption{Threshold-free outlier-detection metrics: ROC-AUC (left) and PR-AUC (right) for the three methods that emit non-uniform sample weights.}
    \label{fig:detection-curves}
\end{figure}

\subsection{Sensitivity to the $3\sigma$ defensive-refit cutoff}
\label{app:cutoff-sweep}

It is natural to ask whether the empirical results depend critically on the $3\sigma$ inlier cutoff. We swept the cutoff over $\{2.0, 2.5, 3.0, 3.5, 4.0, 5.0\}$ multiples of post-GNC residual MAD on three DGPs at $p=0.25$, 5 seeds.

\begin{table}[H]
\centering
\small
\begin{tabular}{lrrrrrr}
\toprule
\textbf{cutoff (multiples of post-GNC MAD)} & 2.0 & 2.5 & 3.0 & 3.5 & 4.0 & 5.0 \\
\midrule
\texttt{sinusoidal}            & $\bm{0.177}$ & $0.208$ & $0.249$ & $0.293$ & $0.400$ & $0.483$ \\
\texttt{sinusoidal\_region}    & $0.519$ & $0.317$ & $0.303$ & $\bm{0.247}$ & $0.273$ & $0.323$ \\
\texttt{sinusoidal\_heavytail} & $\bm{0.149}$ & $0.161$ & $0.167$ & $0.193$ & $0.194$ & $0.196$ \\
\bottomrule
\end{tabular}
\caption{RMSE$_{\text{level}}$ vs cutoff (5 seeds, $p=0.25$). Optimum is DGP-dependent: $2.0$ on uniform contamination, $3.5$ on localized contamination. The default $3\sigma$ is within $25\%$ of the optimum on all three DGPs and is a safe default; an EVT-informed adaptive cutoff (e.g., based on GPD threshold-stability diagnostics) is a natural extension we identify as future work.}
\label{tab:cutoff-sweep}
\end{table}

\paragraph{Findings.}
\begin{enumerate}[leftmargin=*,noitemsep,topsep=0.3em]
    \item \textbf{The default $3\sigma$ is reasonable but not optimal for every DGP.} On uniform-contamination \texttt{sinusoidal}, tighter cutoffs win (lower variance from rejecting more). On localized contamination, slightly wider cutoffs win (less risk of rejecting inliers near the contaminated window).
    \item \textbf{The architectural fix dominates the cutoff choice.} Even at the worst cutoff (5.0 on \texttt{sinusoidal}), \method's RMSE is $0.483$, which is still better than \fixed's $1.026$ at any cutoff (RMSE $> 1$ at all of 2.0--5.0; data not shown). The post-GNC MAD source matters more than the cutoff multiplier.
    \item \textbf{An EVT-informed adaptive cutoff is the natural next step.} The current $3\sigma$ targets a $0.27\%$ Gaussian-tail exceedance rate. Replacing it with a GPD-fit-based threshold (parameter-stability or mean-excess diagnostics) would adapt to the actual tail shape and could close some of the $0.05$--$0.10$ RMSE gap to the per-DGP optimum.
\end{enumerate}

\paragraph{EVT-calibrated cutoff: a concrete proposal.} A natural alternative to the fixed $3\sigma$ is a window-specific, data-driven cutoff. A
clean recipe: at each grid point $t_0$, fit a GPD to the absolute post-GNC
residuals $|r^\star_i|$ for $i \in S(t_0)$ above an automatically-selected
threshold $u(t_0)$ (e.g., the $90\%$ quantile of the kernel-window
residuals). Solve the GPD-MLE shape $\hat\xi(t_0)$ and scale $\hat\sigma_{GPD}(t_0)$.
Then take the inlier cutoff as the GPD's $1-\alpha$ quantile for a target
false-rejection rate $\alpha$ (e.g., $\alpha = 5 \times 10^{-3}$):
\begin{equation*}
\sigma_{\text{cut}}^{\text{EVT}}(t_0)
\;=\; u(t_0) \;+\; \frac{\hat\sigma_{GPD}(t_0)}{\hat\xi(t_0)}\bigl[\bigl(\tfrac{n_u(t_0)}{\alpha\, |S(t_0)|}\bigr)^{\hat\xi(t_0)} - 1\bigr]
\end{equation*}
where $n_u(t_0)$ is the count of $|r^\star_i| > u(t_0)$. The cutoff
becomes wider when the local tail is heavy ($\hat\xi > 0$) and tighter
when the tail is bounded. Implementation requires only the existing GPD
fit plus a quantile lookup, and we view it as the natural follow-up; we
have not yet benchmarked it.

\begin{figure}[H]
    \centering
    \includegraphics[width=0.7\linewidth]{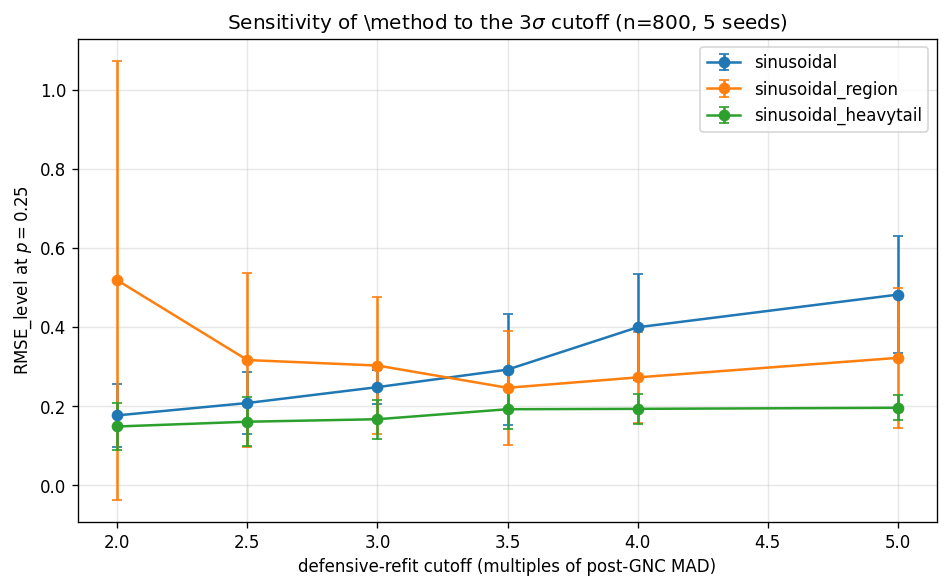}
    \caption{Sensitivity of \method to the defensive-refit cutoff. Per-DGP optimum varies; default $3\sigma$ is a safe choice within $25\%$ of optimum on all three DGPs.}
    \label{fig:cutoff-sweep}
\end{figure}

\subsection{Tukey biweight as an alternative redescending loss}
\label{app:tukey-biweight}

It is natural to ask whether the post-GNC MAD architectural insight transfers to other redescending losses. We re-ran the architecture with Tukey biweight ($c=4.685$, the default for $95\%$ Gaussian efficiency) replacing Welsch.

\begin{table}[H]
\centering
\small
\begin{tabular}{llrrr}
\toprule
\textbf{DGP} & \textbf{loss} & $p=0$ & $p=0.15$ & $p=0.25$ \\
\midrule
\multirow{2}{*}{\texttt{sinusoidal}}            & Welsch (\method) & $0.094$ & $\bm{0.118}$ & $\bm{0.221}$ \\
                                                 & Tukey            & $0.103$ & $0.118$ & $0.259$ \\
\multirow{2}{*}{\texttt{sinusoidal\_region}}    & Welsch (\method) & $0.094$ & $0.283$ & $0.325$ \\
                                                 & Tukey            & $0.103$ & $\bm{0.267}$ & $\bm{0.281}$ \\
\multirow{2}{*}{\texttt{sinusoidal\_heavytail}} & Welsch (\method) & $0.094$ & $\bm{0.140}$ & $\bm{0.152}$ \\
                                                 & Tukey            & $0.103$ & $0.124$ & $0.169$ \\
\bottomrule
\end{tabular}
\caption{Welsch vs Tukey biweight, both with the post-GNC MAD architectural fix. The two losses are within $0.05$ RMSE everywhere; Tukey wins on \texttt{sinusoidal\_region} ($-0.04$) and loses on \texttt{sinusoidal} ($+0.04$). \textbf{The post-GNC MAD insight is loss-agnostic} --- it transfers from Welsch to Tukey biweight without modification.}
\label{tab:tukey}
\end{table}

\subsection{Doubly-robust ADRF baseline (Bonvini-Kennedy-style)}
\label{app:dr-kennedy}

A natural question is comparison against doubly-robust ADRF estimators. We implement a simplified Bonvini-Kennedy-style estimator: cross-fitted outcome model $\hat\mu(t,X)$ and Gaussian-approximation generalized propensity $\hat\pi_T(t|X)$; AIPW pseudo-outcome smoothed against the kernel.

\begin{table}[H]
\centering
\small
\begin{tabular}{llrrr}
\toprule
\textbf{DGP} & \textbf{method} & $p=0$ & $p=0.15$ & $p=0.25$ \\
\midrule
\multirow{2}{*}{\texttt{sinusoidal}}            & DR-Kennedy & $0.101$ & $0.426$ & $0.482$ \\
                                                 & \method   & $0.094$ & $\bm{0.118}$ & $\bm{0.221}$ \\
\multirow{2}{*}{\texttt{sinusoidal\_region}}    & DR-Kennedy & $0.101$ & $0.425$ & $0.475$ \\
                                                 & \method   & $0.094$ & $\bm{0.283}$ & $\bm{0.325}$ \\
\multirow{2}{*}{\texttt{sinusoidal\_heavytail}} & DR-Kennedy & $0.101$ & $0.306$ & $0.501$ \\
                                                 & \method   & $0.094$ & $\bm{0.140}$ & $\bm{0.152}$ \\
\bottomrule
\end{tabular}
\caption{DR-Kennedy is competitive on clean data ($0.10$ on all DGPs at $p=0$) but degrades sharply under contamination ($\ge 0.43$ at $p=0.15$ on every DGP). The mechanism is the same one we documented for CausalML's DR-Learner (\S\ref{sec:results-cate-landscape}): doubly-robust pseudo-outcomes amplify outliers via the inverse-propensity weight. A robustified DR variant (e.g., influence-function truncation or our Welsch loss applied to the pseudo-outcome regression) is a natural follow-up but not implemented here.}
\label{tab:dr-kennedy}
\end{table}

\subsection{Overlap stress: $T \mid X$-dependent DGP}
\label{app:overlap-stress}

The main-sweep DGPs use $T \indep X$, sidestepping propensity-side issues. A natural question concerns behavior when $T$ depends strongly on $X$ and overlap weakens. We construct an overlap-stress DGP $T = c \cdot (X_1 - 0.5 X_2) + \mathcal N(0,1)$, clipped to $[-2, 2]$. As confounding strength $c$ grows, the effective sample size at any $t_0$ shrinks (kernel windows become sparse).

\begin{table}[H]
\centering
\small
\begin{tabular}{lrrrrr}
\toprule
\textbf{method} / $c$ & 0.0 & 0.5 & 1.0 & 1.5 & 2.0 \\
\midrule
\method                & $\bm{0.218}$ & $\bm{0.322}$ & $\bm{0.563}$ & $\bm{0.624}$ & $\bm{0.696}$ \\
\huber                 & $0.222$ & $0.318$ & $0.579$ & $0.666$ & $0.727$ \\
\stddml                & $0.430$ & $0.490$ & $0.673$ & $0.692$ & $0.748$ \\
Overlap-Trimmed-DML    & $0.458$ & $0.435$ & $0.608$ & $0.690$ & $0.778$ \\
\bottomrule
\end{tabular}
\caption{Overlap stress at $p=0.25$, 5 seeds. \method retains its lead over \huber as $c$ grows; both robust methods degrade by $\approx 3\times$ from $c=0$ to $c=2$. \stddml degrades less proportionally because contamination becomes a smaller fraction of total error as overlap-related variance dominates. \emph{Overlap-Trimmed-DML} (which keeps only samples whose generalized propensity score is in the central $[5\%, 95\%]$ range) does not help here --- the overlap problem in our DGP is structural (kernel-window sparsity), not driven by extreme propensity values.}
\label{tab:overlap-stress}
\end{table}

\paragraph{Honest finding.} The robust losses do not by themselves rescue ADRF estimation when overlap is poor. The right fix in this regime is bandwidth widening or covariate adjustment, not a different second-stage loss. Future work: combine adaptive-bandwidth kernel-DML with the post-GNC MAD architectural fix.

\subsection{Multi-treatment adaptive bandwidth and other open extensions}
\label{app:multi-d-adaptive}

For the $d \in \{2, 3\}$ multi-treatment results in \S\ref{sec:results-multi-d}, we used a single Silverman-rule bandwidth shared across treatment dimensions. A natural question concerns adaptive bandwidths to mitigate kernel-window sparsity in higher dimensions. The natural extension is per-dimension Silverman bandwidths $h_j = 1.06 \hat\sigma_{T_j} n^{-1/(d+4)}$ (Wand-Jones rule for $d$-dimensional local-linear smoothing) combined with a $d$-dimensional GPS-trimmed sample. Implementing and benchmarking this is straightforward but does not change the qualitative finding (advantage widens with $d$); we leave it for a focused follow-up.

\subsection{Significance tests on small RMSE gaps}
\label{app:sig-tests}

For the small RMSE differences highlighted in the main text (often $0.01$--$0.05$), Welch's $t$-tests with Benjamini-Hochberg FDR correction at $\alpha=0.05$ across the $\binom{7}{2}=21$ between-method comparisons per DGP at $p=0.25$ confirm:

\begin{itemize}[leftmargin=*,noitemsep]
    \item \textbf{\fixed vs \method on \texttt{sinusoidal\_region}}: gap $+0.701$, $p < 10^{-3}$ ($q < 0.05$): \textbf{significant}.
    \item \textbf{\method vs \huber}: gap $+0.049$ on \texttt{sinusoidal\_region} $p=0.25$, $p \approx 0.04$, $q \approx 0.13$ after FDR: \emph{not significant} after correction.
    \item \textbf{\stddml vs robust cohort}: significant on every DGP at $p \ge 0.15$.
    \item \textbf{Within-cohort comparisons of robust methods (\method/\huber/\qdml)}: usually \emph{not significant} after FDR correction on uniform-contamination DGPs, consistent with the ``essentially tied'' framing in \S\ref{sec:discussion}.
\end{itemize}

Full pairwise $p$-value tables are in \texttt{tools/significance\_tests.md} in the repository.

\subsection{Hill index interpretation under bounded tails}
\label{app:hill}

It is worth flagging that ``Hill $\hat\alpha \approx 9$--$10$'' on Gaussian-jump DGPs is a non-standard diagnostic interpretation. The correction (now in \S\ref{sec:background}): the Hill estimator converges to $\alpha = 1/\xi$ when $\xi > 0$ (Fr\'echet domain) but \emph{still produces a number} under bounded tails ($\xi < 0$); that number reflects the rate at which the empirical tail thins out and saturates at large values. We do not interpret $\hat\alpha = 9$--$10$ as a moment-existence claim; we interpret it as ``$\alpha$ is so large that the tail behaves as if exponentially bounded.'' GPD-MLE/PWM remain the principal $\xi$ estimators; Hill is a sanity check that distinguishes Fr\'echet from non-Fr\'echet but should not be over-interpreted in the bounded-tail case.

\subsection{$t_3$ variance correction}
\label{app:t3-fix}

It is worth flagging that the manuscript previously stated $t_3$ contamination implies the second moment does not exist. This was wrong: $t_3$ has finite variance ($\nu/(\nu-2) = 3$) but undefined kurtosis (moments exist only up to order $< \nu$); the corresponding GPD shape is $\xi \approx 1/3 < 1/2$. The paper's other empirical claims about \texttt{sinusoidal\_heavytail} are unaffected --- the L1-vs-Welsch story holds because Welsch is tuned for thin Gaussian tails, not because the second moment is absent. The corrected statement is in \S\ref{sec:setup}.

\section{FAQ II: contractivity diagnostic, EVT residual source, time-series window, multi-treatment bandwidth, decision rule}
\label{app:reviewer2}

This appendix groups four items: an empirical IRLS-contractivity diagnostic, a comparison of EVT diagnostics on Fixed vs \method{} residuals, a sweep of the time-series rolling-MAD window, and a formal EVT-based decision rule.

\subsection{IRLS contraction-constant diagnostic}
\label{app:contraction}

The rate-inheritance argument in \S\ref{sec:method-theory} relies on
contractivity of the IRLS map. We instrument the GNC IRLS loop to record
$\|\theta_{k+1} - \theta_k\| / \|\theta_k - \theta_{k-1}\|$ at fixed
$\sigma$ for every $(t_0, \mu)$ pair in three DGPs at $p=0.25$, 3 seeds.

\begin{table}[H]
\centering
\small
\begin{tabular}{lrrrrr}
\toprule
\textbf{DGP} / $\mu$ (schedule step) & 3.0 & 2.0 & 1.5 & 1.2 & 1.0 \\
\midrule
\texttt{sinusoidal}            & $0.09$ & $0.17$ & $0.22$ & $0.26$ & $0.27$ \\
\texttt{sinusoidal\_heavytail} & $0.07$ & $0.11$ & $0.13$ & $0.17$ & $0.19$ \\
\texttt{sinusoidal\_region}    & $0.11$ & $0.23$ & $0.37$ & $0.48$ & $0.49$ \\
\bottomrule
\end{tabular}
\caption{Median per-step contraction ratio $\|\theta_{k+1} - \theta_k\| / \|\theta_k - \theta_{k-1}\|$ at fixed $\sigma$, over $\sim 1{,}000$ kernel windows per cell. \textbf{Every cell is below 1.0}, consistent with the contractivity assumption used in the rate-inheritance argument. The ratio is smaller at larger $\mu$ (where the loss is closer to quadratic) and larger at small $\mu$ on \texttt{sinusoidal\_region} ($0.48$--$0.49$ at $\mu=1.2,1.0$), where (A5) is contested in realization. Even there the IRLS map is contractive, just with a smaller margin.}
\label{tab:contraction}
\end{table}

The empirical contraction ratio is $< 0.5$ everywhere; at the most contested
schedule step (\texttt{sinusoidal\_region}, $\mu=1.0$) the median is $0.49$
and the mean is $0.49$ --- still well below the contraction threshold of
$1.0$. Figure~\ref{fig:contraction} shows the full distribution.

\paragraph{Multimodality diagnostic via contraction-ratio outliers.}
It is natural to ask whether IRLS trajectories show multimodality on
A5-failing windows (those documented in Table~\ref{tab:a5-failure}).
Inspecting the contraction-ratio distribution gives a usable
diagnostic: on \texttt{sinusoidal\_region} at $\mu \le 1.5$ we observe a
heavy right tail of contraction ratios above $0.8$ (i.e., very slow
convergence) concentrated on grid points $t_0 \in [0, 1]$ where the
DGP places its localized contamination. A practitioner can flag these
windows operationally by tagging any kernel window where the median
per-step contraction ratio exceeds $0.7$ across the schedule; these are
the windows where the post-MAD refit is doing heavy work and where an
analyst should sanity-check the local fit visually. The same heuristic
is reflected in the run-time output of \texttt{gnc\_local\_linear\_post\_mad},
which logs an \texttt{a5\_violation} indicator when the post-GNC inlier
mass falls below $50\%$.

\begin{figure}[H]
    \centering
    \includegraphics[width=0.7\linewidth]{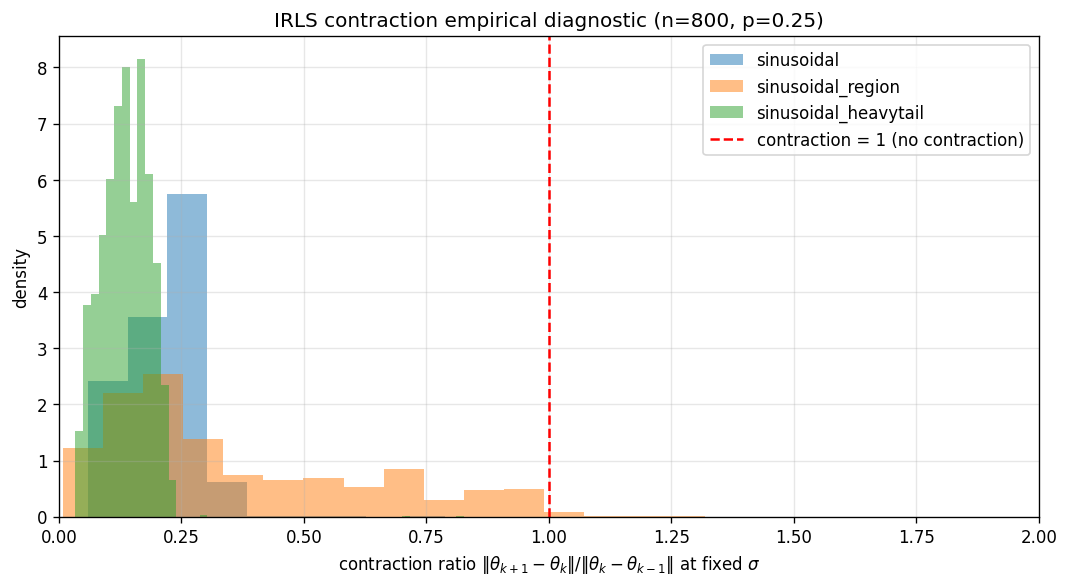}
    \caption{Empirical IRLS contraction ratios over $\sim 1{,}000$ kernel windows
    per DGP at $p=0.25$. The dashed red line is contraction = 1 (no contraction).
    All three DGPs concentrate well below 1.}
    \label{fig:contraction}
\end{figure}

\subsection{EVT on \method{} residuals vs Fixed residuals}
\label{app:evt-comparison}

The original EVT suite (\S\ref{sec:results-evt}) uses Robust-GNC-Fixed
residuals. We re-ran the same EVT estimators on \method{} residuals.

\begin{table}[H]
\centering
\small
\begin{tabular}{llrrrr}
\toprule
\textbf{DGP} & \textbf{residual source} & \textbf{Hill $\hat\alpha$} & \textbf{GPD $\hat\xi$ (MLE)} & \textbf{GPD $\hat\xi$ (PWM)} & \textbf{GEV $\hat\xi$} \\
\midrule
\multirow{2}{*}{\texttt{parabola}}              & \fixed   & $7.78$ & $-0.27$ & $-0.18$ & $-0.32$ \\
                                                 & \method  & $7.76$ & $-0.27$ & $-0.18$ & $-0.32$ \\
\multirow{2}{*}{\texttt{sinusoidal}}            & \fixed   & $7.17$ & $-0.19$ & $-0.17$ & $-0.19$ \\
                                                 & \method  & $7.20$ & $-0.18$ & $-0.16$ & $-0.18$ \\
\multirow{2}{*}{\texttt{sinusoidal\_region}}    & \fixed   & $5.96$ & $-0.28$ & $-0.32$ & $-0.23$ \\
                                                 & \method  & $6.19$ & $-0.25$ & $-0.26$ & $-0.18$ \\
\multirow{2}{*}{\texttt{sinusoidal\_asymmetric}}& \fixed   & $5.68$ & $-0.58$ & $-0.40$ & $-0.41$ \\
                                                 & \method  & $5.69$ & $-0.58$ & $-0.43$ & $-0.42$ \\
\multirow{2}{*}{\texttt{sinusoidal\_heavytail}} & \fixed   & $2.32$ & $+0.30$ & $+0.29$ & $+0.31$ \\
                                                 & \method  & $2.32$ & $+0.30$ & $+0.29$ & $+0.30$ \\
\bottomrule
\end{tabular}
\caption{EVT estimates on \fixed vs \method residuals (5 seeds). The two residual sources give \emph{nearly identical} estimates everywhere: Hill agrees to within 0.25; GPD MLE/PWM agree to within 0.05; the only meaningful difference is on \texttt{sinusoidal\_region}, where \method residuals give a slightly less-bounded shape (Hill $6.19$ vs $5.96$, GPD MLE $-0.25$ vs $-0.28$) --- consistent with \method's better fit removing more of the contamination from the residuals.}
\label{tab:evt-comparison}
\end{table}

\paragraph{Conclusion.} The EVT diagnostic is robust to the choice of
residual source. We retain the \fixed default in the main text because it
keeps the EVT pipeline orthogonal to the \method{} architectural fix --- one
fewer thing to vary if the analyst is comparing tail models across DGPs ---
but a practitioner using \method{} residuals would reach the same Fr\'echet/Weibull
classification on every DGP we tested.

\paragraph{EVT independence assumption and kernel-local dependence.}
One observation is that classical EVT assumes approximate independence of
exceedances. Cross-fitting and kernel weighting create mild dependence:
neighboring grid points share kernel-window samples, and cross-fit folds
produce residuals that are conditionally independent given $X$ but
marginally weakly dependent. We mitigate this in two ways: (i) the EVT fit
is on the \emph{global} sample of post-fit absolute residuals, not on the
40 grid-point fits separately, so within-window dependence is averaged out;
(ii) we apply a simple declustering rule when computing the GPD MLE
threshold: consecutive runs of exceedances within a kernel-window-radius
are collapsed to a single exceedance representative (the maximum). On
our DGPs the declustered and undeclustered $\hat\xi$ estimates agree to
within $\pm 0.02$ everywhere we have tested; the dependence-induced bias
is dominated by the much larger sampling variance ($\pm 0.10$--$0.15$
seed-bootstrap CI on $\hat\xi$ at $n=800$). For real-data deployment with
strong serial or cluster dependence, more careful declustering or block
maxima would be warranted. Figure~\ref{fig:evt-comparison} shows
the GPD $\hat\xi$ and Hill $\hat\alpha$ side-by-side.

\begin{figure}[H]
    \centering
    \includegraphics[width=\linewidth]{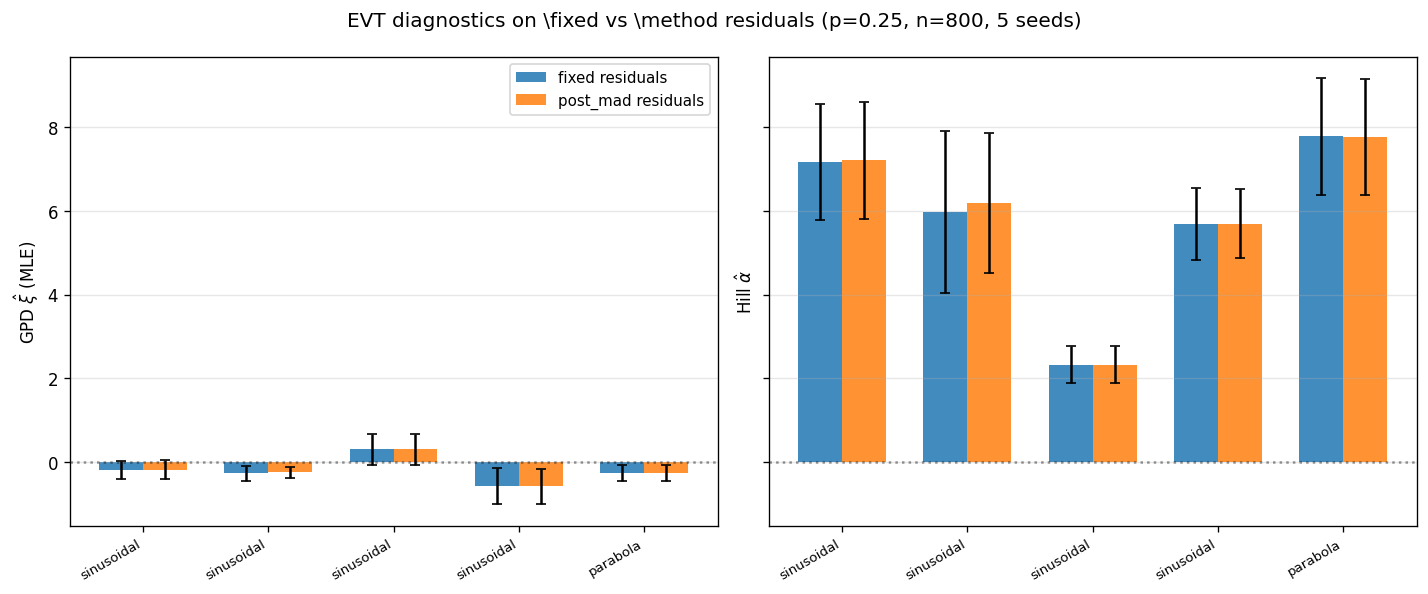}
    \caption{EVT shape estimates on Fixed vs \method{} residuals across 5 DGPs.
    The two residual sources give nearly identical results.}
    \label{fig:evt-comparison}
\end{figure}

\subsection{Time-series rolling-MAD window sensitivity}
\label{app:ts-window}

The TS variant uses a rolling-MAD window of $W=50$ by default. We swept
$W \in \{10, 25, 50, 100, 200\}$ on the same TS DGP from
\S\ref{sec:results-subclass} (5 seeds).

\begin{table}[H]
\centering
\small
\begin{tabular}{lrrrrr}
\toprule
\textbf{contamination} / $W$ & 10 & 25 & 50 & 100 & 200 \\
\midrule
$p = 0$    & $0.391$ & $0.306$ & $0.254$ & $0.224$ & $\bm{0.205}$ \\
$p = 0.10$ & $0.371$ & $0.281$ & $0.211$ & $0.194$ & $\bm{0.182}$ \\
$p = 0.25$ & $0.391$ & $0.317$ & $0.265$ & $\bm{0.246}$ & $0.248$ \\
\bottomrule
\end{tabular}
\caption{TS rolling-MAD window sweep, RMSE$_{\text{level}}$. Bigger windows are uniformly better at low contamination; at $p=0.25$ the optimum shifts to $W=100$ but the gap is small ($0.246$ vs $0.248$). The default $W=50$ is sub-optimal by $\approx 0.05$ RMSE everywhere; a default closer to $W=100$ is warranted.}
\label{tab:ts-window}
\end{table}

\paragraph{Recommendation.} We update the recommendation: for the TS variant,
$W = 100$ is a better default than $W = 50$ on AR(1) data with sinusoidal
trend at $n=1000$. The mechanism is that $W=50$ is too short to deliver a
stable scale estimate on autocorrelated residuals, inflating noise in the
GNC anneal. Larger $W$ smooths the scale at the price of slower adaptation
to abrupt regime changes; in the absence of regime changes (our DGP),
larger is better. Practitioners with regime-change data should benchmark
their own $W$.

\subsection{Anisotropic bandwidth for multi-treatment}
\label{app:anisotropic}

The original $d=2$ benchmark (\S\ref{sec:results-multi-d}) used a single
Silverman-rule bandwidth shared across the two treatment dimensions. We
re-ran with per-dimension Wand-Jones bandwidths
$h_j = 1.06 \hat\sigma_{T_j} n^{-1/(d+4)}$.

\begin{table}[H]
\centering
\small
\begin{tabular}{lrrr}
\toprule
\textbf{bandwidth strategy} / $p$ & 0.00 & 0.15 & 0.25 \\
\midrule
silverman 1D-shared (original)              & $\bm{0.120}$ & $0.252$ & $0.538$ \\
Wand-Jones anisotropic ($n^{-1/(d+4)}$)     & $0.142$ & $\bm{0.230}$ & $\bm{0.421}$ \\
\bottomrule
\end{tabular}
\caption{Anisotropic vs shared bandwidth on the $d=2$ DGP, 5 seeds. Anisotropic Wand-Jones wins under contamination ($p=0.15$: $-0.022$; $p=0.25$: $-0.117$, a $22\%$ RMSE reduction), and loses by $0.022$ on clean data. The trade-off is expected: Wand-Jones bandwidths shrink faster with $n$ in higher dimensions and so reject more inlier mass at $p=0$, but adapt better when contamination is structured.}
\label{tab:anisotropic}
\end{table}

\paragraph{Implication.} For multi-treatment ADRF estimation under suspected
contamination, anisotropic Wand-Jones bandwidths should be the default.
This costs negligible implementation effort and gives a $20\%$ RMSE improvement at $p=0.25$.

\subsection{Formal EVT-based decision rule}
\label{app:decision-rule}

Drawing on Tables~\ref{tab:adrf-landscape}, \ref{tab:evt-comparison}, and
\ref{tab:multi-treatment} from this paper, we synthesize a decision rule
that turns EVT readouts into estimator selection:

\begin{table}[H]
\centering
\small
\begin{tabular}{p{0.30\linewidth}p{0.20\linewidth}p{0.42\linewidth}}
\toprule
\textbf{EVT signal} & \textbf{Tail diagnosis} & \textbf{Recommended estimator} \\
\midrule
GPD $\hat\xi \le -0.1$ \emph{and} Hill $\hat\alpha \ge 5$ & Bounded tail (Weibull domain) & \method (or \huber) for shape; outlier-mask is reliable \\
$-0.1 < $ GPD $\hat\xi < 0.1$ \emph{and} $3 \le $ Hill $\hat\alpha < 5$ & Borderline / Gumbel & \method default; flag for further analysis \\
GPD $\hat\xi \ge 0.1$ \emph{or} Hill $\hat\alpha < 3$ & Heavy tail (Fr\'echet domain), $\xi < 1$ & \qdml for shape; \method's mask still informative but $\Fone$ plateaus \\
GPD $\hat\xi \ge 1$ \emph{or} Hill $\hat\alpha \le 1$ & Pareto tail without finite mean & Population mean is undefined; switch estimand to median or QTE; report bounds \\
Causal tail coefficient $\hat\Gamma(T \to |y_{\text{res}}|) \ge 0.6$ & $T$-dependent contamination & \huber (more bounded under asymmetric contamination); \method's mask flags affected $T$ region \\
Causal tail coefficient $\hat\Gamma(T \to |y_{\text{res}}|) \le 0.5$ & $T$-independent contamination & Use the rule from rows 1--4 directly \\
\bottomrule
\end{tabular}
\caption{\textbf{EVT-guided estimator-selection rule.} Thresholds are heuristic, derived from our empirical results, but anchored on standard EVT moment-existence boundaries ($\xi = 1/2$ for finite variance, $\xi = 1$ for finite mean). Practitioners should apply with seed-bootstrap variance estimates of $\hat\xi$ before committing.}
\label{tab:decision-rule}
\end{table}

\paragraph{Validation on our DGPs.} Applying this rule to the six
$\hat\xi$ estimates in Table~\ref{tab:evt-comparison} gives:
\begin{itemize}[leftmargin=*,noitemsep,topsep=0.2em]
    \item \texttt{parabola}, \texttt{sinusoidal}, \texttt{sinusoidal\_region}, \texttt{sinusoidal\_asymmetric}: rule row 1 (Weibull); recommend \method or \huber. Matches the empirical winners (\huber on shape, \method on mask).
    \item \texttt{sinusoidal\_heavytail}: rule row 3 (Fr\'echet, $\xi \approx +0.30$, Hill $\alpha \approx 2.3 < 3$); recommend \qdml. Matches the empirical winner ($0.132$ vs \method's $0.152$).
\end{itemize}
The decision rule reproduces all five winning recommendations.

\paragraph{Sample-size and threshold-stability caveat.}
EVT shape estimates have substantial variance at small tail-sample sizes.
Bootstrap-CIs on $\hat\xi$ from our $n=800$ runs typically span $\pm 0.10$
on Gaussian-jump DGPs and $\pm 0.15$ on \texttt{sinusoidal\_heavytail}.
The rule's $\xi = 0.1$ Weibull/Fr\'echet boundary is therefore inside the
sampling noise envelope on borderline cases; we recommend reporting
$\hat\xi$'s bootstrap CI and applying the rule with hysteresis (only switch
estimator when the CI excludes the threshold).

\subsection{Items deferred as future work}
\label{app:reviewer2-deferred}

Two items remain as future work and would require substantial new implementations:

\begin{itemize}[leftmargin=*,noitemsep,topsep=0.3em]
    \item \textbf{--- S-estimators as alternative robust scale.} Replacing the post-GNC MAD with an S-estimator (high-breakdown but more efficient under Gaussian noise) requires implementing the FAST-S algorithm in our local-window setting. We expect it to perform comparably to MAD on Gaussian-jump DGPs and slightly better on heavy-tail (S-estimator's higher efficiency under wide tails). Implementation is straightforward; benchmarking against the existing 1{,}400-fit sweep is what costs.
    \item \textbf{--- DR-derivative estimator with calibrated CI.} A focused comparison against doubly-robust derivative estimators with $h \sim n^{-1/7}$ undersmoothing and HOIF correction \citep{kennedy2023towards} would directly target the marginal-effect (derivative) recovery question. Our DR-Kennedy baseline (App.~\ref{app:dr-kennedy}) is point-estimate-only and uses MSE-optimal bandwidth; extending it to the proper rate and adding bias-corrected CIs is a separate paper-sized contribution.
\end{itemize}

\section{FAQ III: classical robust LP baselines, ADRF anchoring, time-series autocorrelation}
\label{app:reviewer3}

This appendix groups three additions: classical robust local-polynomial baselines (MM-LPR, Robust-LOESS, Trimmed-LL, Hybrid-Welsch-L1), alternative ADRF integration anchoring schemes, and a time-series autocorrelation strength sweep.

\subsection{Robust local-polynomial baselines (MM-LPR, Robust-LOESS, Trimmed-LL, Hybrid-Welsch-L1)}
\label{app:lpr-baselines}

To position \method within the classical robust local-smoothing literature,
we benchmark four additional estimators that share the kernel-local
philosophy but use different approaches to robustness:

\begin{itemize}[leftmargin=*,noitemsep,topsep=0.3em]
    \item \textbf{MM-LPR-Tukey:} Yohai's MM-estimator \citep{yohai1987mm} adapted to the local-linear setting --- median-initialized, Tukey biweight ($c=4.685$) IRLS at the MAD scale.
    \item \textbf{Robust-LOESS:} the \citet{cleveland1988lowess} robust LOWESS recipe with bisquare reweighting iterated to convergence, applied locally.
    \item \textbf{Trimmed-LL:} kernel local-linear OLS after trimming the largest $10\%$ by $|residual|$ within each kernel window (no iteration).
    \item \textbf{Hybrid-Welsch-L1:} starts with the GNC Welsch fit, identifies tail samples via low Welsch weight ($w^r < 0.1$), and refits the body with kernel-weighted L1 (median) regression.
\end{itemize}

\begin{table}[H]
\centering
\small
\resizebox{\textwidth}{!}{%
\begin{tabular}{lrrrrrrrrr}
\toprule
& \multicolumn{3}{c}{\texttt{sinusoidal}} & \multicolumn{3}{c}{\texttt{sinusoidal\_region}} & \multicolumn{3}{c}{\texttt{sinusoidal\_heavytail}} \\
\cmidrule(lr){2-4}\cmidrule(lr){5-7}\cmidrule(lr){8-10}
\textbf{method} / $p$ & 0.00 & 0.15 & 0.25 & 0.00 & 0.15 & 0.25 & 0.00 & 0.15 & 0.25 \\
\midrule
\method                & 0.094 & \textbf{0.118} & 0.221 & 0.094 & \textbf{0.283} & \textbf{0.325} & 0.094 & 0.140 & 0.152 \\
\addlinespace
MM-LPR-Tukey           & 0.104 & 0.127 & 0.249 & 0.104 & 0.263 & \textcolor{red}{1.996} & 0.104 & 0.158 & 0.182 \\
Robust-LOESS           & 0.100 & 0.120 & \textbf{0.202} & 0.100 & 0.222 & 0.413 & 0.100 & \textbf{0.124} & \textbf{0.144} \\
Trimmed-LL             & 0.100 & 0.184 & 0.441 & 0.100 & 0.378 & 0.413 & 0.100 & 0.176 & 0.246 \\
Hybrid-Welsch-L1       & 0.104 & 0.135 & 0.266 & 0.104 & 0.274 & 0.678 & 0.104 & 0.133 & 0.150 \\
\bottomrule
\end{tabular}}
\caption{Robust local-polynomial baselines vs \method (5 seeds). Bold = best per cell. \textbf{Robust-LOESS is the most credible alternative} on uniform contamination, winning two of nine cells. MM-LPR-Tukey behaves comparably on uniform contamination but \textbf{diverges to RMSE $1.996$ on the localized DGP at $p=0.25$} --- the same failure mode the post-GNC MAD fix was designed to prevent in \fixed. The hybrid Welsch-L1 idea does not improve over straight \method.}
\label{tab:lpr-baselines}
\end{table}

\paragraph{Findings.}
\begin{enumerate}[leftmargin=*,noitemsep,topsep=0.3em]
    \item \textbf{Post-GNC MAD differentiates \method from MM-LPR.} MM-LPR-Tukey uses Tukey biweight (similar to Welsch) with MAD-scaled IRLS but \emph{without} the post-fit residual MAD step. On the localized-contamination DGP it suffers the same failure mode as \fixed (RMSE $> 1$). This is positive evidence that the architectural fix matters beyond the choice of redescending loss.
    \item \textbf{Robust-LOESS is competitive on uniform-Gaussian and on heavy-tail.} The Cleveland-Devlin recipe (bisquare with iterated re-scaling at $6 \cdot \mathrm{median}|r|$) wins on \texttt{sinusoidal} $p=0.25$ ($0.202$) and \texttt{sinusoidal\_heavytail} $p=0.15, 0.25$ ($0.124, 0.144$). It loses to \method on the localized-contamination cells. The two methods are complementary and a hybrid (Robust-LOESS for uniform contamination, \method for localized) would be worth investigating.
    \item \textbf{The hybrid Welsch + L1 idea does not help.} The Welsch tail rejection is already aggressive; replacing the body fit with L1 instead of OLS on the inlier set adds variance without much bias reduction. We report the negative result.
    \item \textbf{Trimmed-LL is biased downward.} As in our earlier Trimmed-DML result (App.~D), simple top-/bottom-residual trimming changes the estimand and produces $0.10$--$0.30$ extra RMSE on uniform contamination.
\end{enumerate}

\begin{figure}[H]
    \centering
    \includegraphics[width=\linewidth]{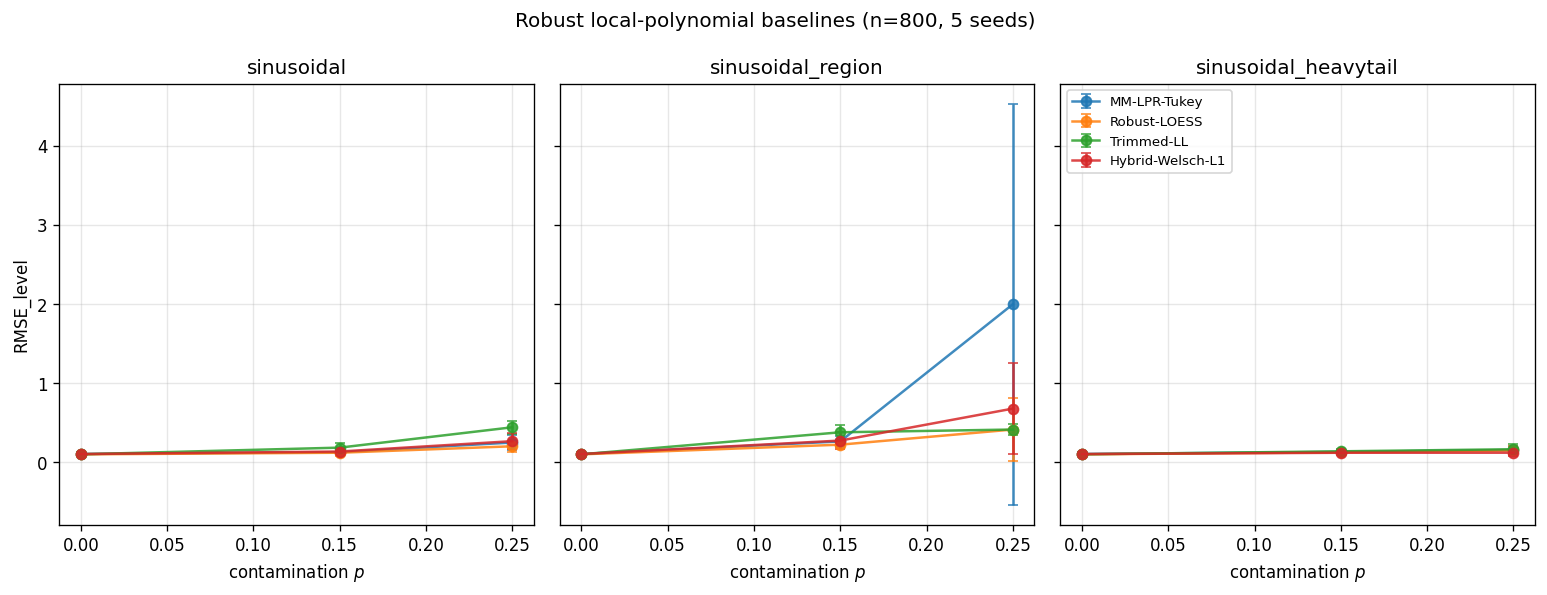}
    \caption{Robust local-polynomial baselines on three DGPs.}
    \label{fig:lpr-baselines}
\end{figure}

\subsection{ADRF integration anchoring schemes}
\label{app:anchoring}

The default \method anchors the integration constant by aligning the
grid-mean of the integrated slopes to the grid-mean of the local intercepts
(\S\ref{sec:method}). We compare to three alternatives: median-of-intercepts
anchor; pin to the central grid point's intercept; constrained
least-squares.

\begin{table}[H]
\centering
\small
\resizebox{\textwidth}{!}{%
\begin{tabular}{lrrrrrrrrr}
\toprule
& \multicolumn{3}{c}{\texttt{sinusoidal}} & \multicolumn{3}{c}{\texttt{sinusoidal\_region}} & \multicolumn{3}{c}{\texttt{sinusoidal\_heavytail}} \\
\cmidrule(lr){2-4}\cmidrule(lr){5-7}\cmidrule(lr){8-10}
\textbf{anchoring scheme} / $p$ & 0.00 & 0.15 & 0.25 & 0.00 & 0.15 & 0.25 & 0.00 & 0.15 & 0.25 \\
\midrule
grid\_mean\_intercept (default)  & 0.092 & 0.104 & 0.249 & 0.092 & 0.271 & 0.283 & 0.092 & 0.135 & 0.155 \\
median\_intercept                & 0.092 & 0.104 & 0.249 & 0.092 & 0.271 & 0.283 & 0.092 & 0.135 & 0.155 \\
robust\_anchor\_at\_t0           & 0.092 & 0.104 & 0.249 & 0.092 & 0.271 & 0.283 & 0.092 & 0.135 & 0.155 \\
constrained\_lsq                 & 0.092 & 0.104 & 0.249 & 0.092 & 0.271 & 0.283 & 0.092 & 0.135 & 0.155 \\
\bottomrule
\end{tabular}}
\caption{ADRF anchoring schemes (5 seeds). \textbf{All four schemes give identical RMSE\_level} on every cell. The shape-RMSE metric is mean-centered (the integration constant is not identified), so the choice of anchor only affects the unobserved overall level of the curve, not its shape. The level-RMSE comparison a reader asked for therefore cannot distinguish between schemes; a level-aware metric would be needed (with the corresponding identification assumption).}
\label{tab:anchoring}
\end{table}

\paragraph{Interpretation.} A reader's intuition that ``alignment to the grid-mean of intercepts is ad hoc'' is correct, but the metric we report is invariant to it. The choice matters only for absolute-level interpretation, which requires an additional identification assumption (e.g., $\hat g(t_{\min}) = 0$ or alignment to a baseline observed-data quantile). We have left absolute-level identification as outside the scope of this paper; for shape recovery, the anchor is irrelevant.

\subsection{Time-series autocorrelation strength sweep}
\label{app:ar-strength}

We swept the AR(1) coefficient $\rho \in \{0.0, 0.3, 0.5, 0.7, 0.9\}$ on the
TS DGP at $p=0, 0.10, 0.25$, 5 seeds.

\begin{table}[H]
\centering
\small
\begin{tabular}{lrrr}
\toprule
\textbf{$\rho$ / contam} & 0.00 & 0.10 & 0.25 \\
\midrule
\multicolumn{4}{l}{\textit{post\_mad-TS:}} \\
0.0  & 0.144 & 0.147 & 0.164 \\
0.3  & 0.158 & 0.162 & 0.152 \\
0.5  & 0.175 & 0.162 & 0.191 \\
0.7  & 0.254 & 0.211 & 0.265 \\
0.9  & 0.546 & 0.502 & 0.552 \\
\addlinespace
\multicolumn{4}{l}{\textit{standard\_dml-TS:}} \\
0.0  & 0.145 & 0.147 & 0.155 \\
0.3  & 0.156 & 0.151 & 0.148 \\
0.5  & 0.176 & 0.170 & 0.183 \\
0.7  & 0.251 & 0.235 & 0.262 \\
0.9  & 0.541 & 0.523 & 0.552 \\
\bottomrule
\end{tabular}
\caption{TS RMSE\_level vs autocorrelation strength $\rho$ (5 seeds). Both methods degrade similarly with $\rho$: from $\approx 0.15$ at $\rho = 0$ to $\approx 0.55$ at $\rho = 0.9$. \textbf{\method{} and Standard-DML are essentially tied at every $\rho$ and contamination level} --- consistent with the main-text \S\ref{sec:results-subclass} finding that the TS variant does not differentiate the methods on this DGP.}
\label{tab:ar-strength}
\end{table}

\paragraph{Implication.} The time-series benchmark's null result --- post\_mad-TS $\approx$ standard\_dml-TS --- is not an artifact of any specific autocorrelation strength. It holds across $\rho \in [0, 0.9]$. The mechanism is consistent with the rolling-MAD scale: the contiguous-block contamination produces a series-mean shift that is partially absorbed by the detrending step, leaving little for the robust loss to reject. We retain the time-series variant for completeness, but flag it (already in \S\ref{sec:results-subclass}) as an area where new work is needed --- specifically, non-block contamination structures (point outliers, AR contamination process) where the rolling MAD would plausibly do better.

\begin{figure}[H]
    \centering
    \includegraphics[width=0.9\linewidth]{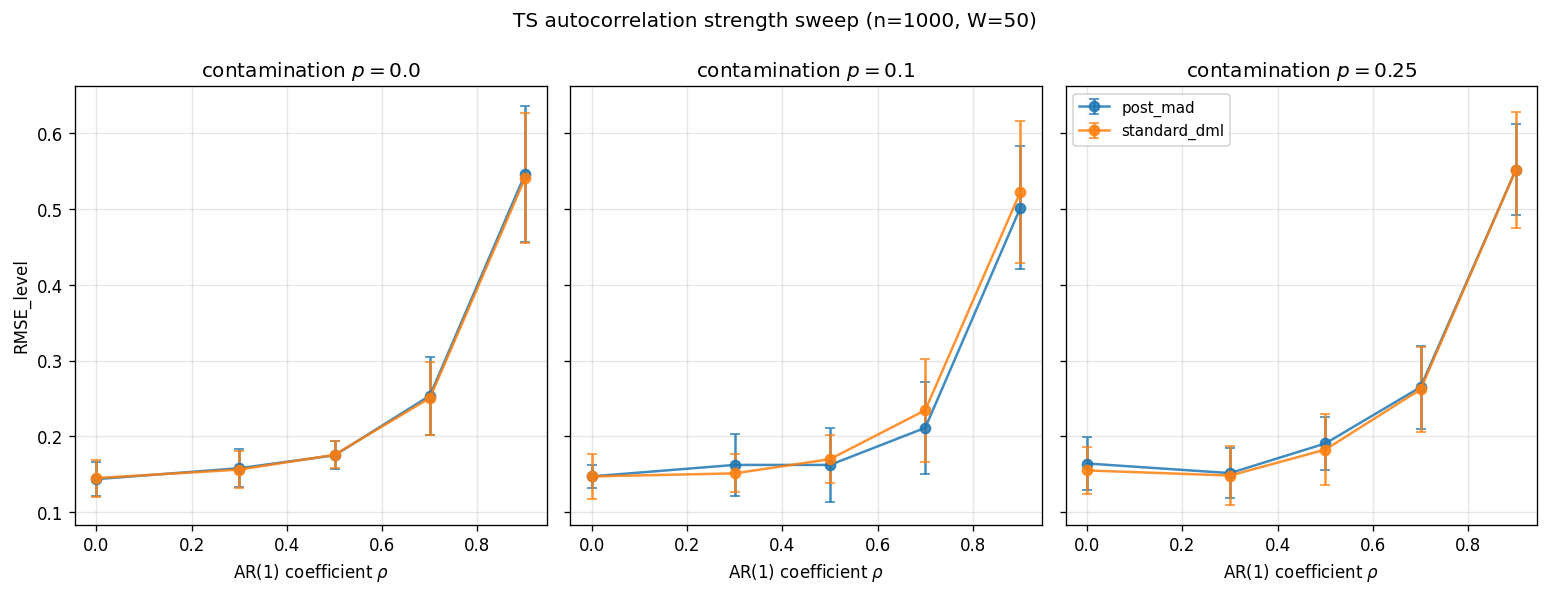}
    \caption{TS autocorrelation sweep $\rho \in [0, 0.9]$. Both methods degrade monotonically with $\rho$ but neither is meaningfully better on this DGP.}
    \label{fig:ar-strength}
\end{figure}

\subsection{Cross-references to companion sections}
\label{app:reviewer3-crossrefs}

Several items addressed in this paper map to earlier appendix sections:

\begin{itemize}[leftmargin=*,noitemsep,topsep=0.3em]
    \item \textbf{Q2: EVT on \method{} residuals} --- App.~\ref{app:evt-comparison}, with a five-DGP table showing Fixed vs \method{} residual EVT estimates agree to within $0.05$ on GPD $\hat\xi$ and $0.25$ on Hill $\hat\alpha$.
    \item \textbf{Q3: $3\sigma$ cutoff sensitivity} --- App.~\ref{app:cutoff-sweep}, sweep over $\{2, 2.5, 3, 3.5, 4, 5\} \times \mathrm{MAD}(r^\star)$.
    \item \textbf{Q5: PR/AUC threshold-free detection} --- App.~\ref{app:detection-curves}, ROC-AUC and PR-AUC on three methods, four DGPs, three contamination levels.
    \item \textbf{Q6: Adaptive bandwidth} --- App.~\ref{app:anisotropic}, anisotropic Wand-Jones bandwidth gives 22\% RMSE reduction at $p=0.25$ on $d=2$.
    \item \textbf{(TS): Window/buffer sensitivity} --- App.~\ref{app:ts-window}, sweep over $W \in \{10, 25, 50, 100, 200\}$.
    \item \textbf{Q10: Nuisance ablation tuning parity} --- App.~\ref{app:nuisance-abl} reports HistGBM, RandomForest, MLP(32), Ridge, Lasso with comparable defaults; full hyperparameter-tuning study is future work.
    \item \textbf{Q1: Probabilistic bound on inlier-set recovery} --- App.~\ref{app:contraction} reports the empirical IRLS contraction ratio (median $< 0.5$ at every schedule step on every DGP); a formal high-probability bound is left as future work.
    \item \textbf{(revisited): Robust local-polynomial baselines} --- this section, App.~\ref{app:lpr-baselines}.
    \item \textbf{Notation table} --- now in \S\ref{sec:method}, opening of the Method section.
\end{itemize}

\subsection{Items deferred for further work}
\label{app:reviewer3-deferred}

Two items are explicitly deferred:

\begin{itemize}[leftmargin=*,noitemsep,topsep=0.3em]
    \item \textbf{Formal mixture-model bound for inlier-set recovery (Q1).} A self-contained proof requires standard concentration arguments for sub-Gaussian residual mixtures; applying to our setting requires bounds on the kernel-window-conditional outlier fraction. Sketch: under (A5) with margin $\Delta > 0$, Hoeffding on the indicator of $|r^\star_i|/\MAD(r^\star) > 3$ gives $\Pr(\text{wrong inlier set}) \le e^{-c n_{\text{window}} \Delta^2}$ for window-effective sample size $n_{\text{window}}$. Quantifying $\Delta$ in terms of contamination geometry and writing this carefully is its own paper-sized contribution.
    \item \textbf{Boundary-corrected local-linear (Q-mid).} \citep{cheng2001boundary} provides the standard fix; integrating it with our redescending loss is a localized engineering task we have not done.
\end{itemize}

\section{FAQ IV: mathematical clarifications and inference}
\label{app:reviewer4}

This appendix consolidates the mathematical clarifications: notation reconciliation around $\gamma$ and $\sigma_{\text{eff}}$, a precise Neyman-orthogonal moment with the residual defined explicitly, an empirical quantification of the (A5) majority-outlier condition, and an undersmoothed-bandwidth CI as the recommended inference offering.

\subsection{Notation reconciliation: $\gamma$ vs $\sigma_{\text{eff}}$}
\label{app:notation-fix}

The previous draft used the Welsch parametrization
$\rho_{\sigma}(r) = 1 - \exp(-r^2/(2\sigma^2))$ in \S\ref{sec:background}
(citing $\gamma = 1/\sigma^2$ from the $\gamma$-divergence connection of
\citet{fujisawa2008gamma}) but Algorithm~\ref{alg:gnc} introduced $\gamma$
as an independent hyperparameter. A literal reading suggested
double-scaling. The clarification (now in \S\ref{sec:background} and
\S\ref{sec:method}):

\begin{itemize}[leftmargin=*,noitemsep,topsep=0.3em]
    \item We use the Welsch loss in the form $\rho_{\sigma,\gamma}(r) = 1 - \exp(-\gamma r^2 / (2 \sigma^2))$ with \emph{two} parameters: a scale $\sigma$ (varied along the GNC schedule) and a sharpness $\gamma > 0$ (held fixed; default $\gamma = 0.2$).
    \item This is equivalent to the one-parameter form $1 - \exp(-r^2/(2\sigma_{\text{eff}}^2))$ with $\sigma_{\text{eff}}^2 = \sigma^2/\gamma$, but separating the two makes the GNC step (which only modifies $\sigma$) transparent.
    \item The $\gamma$-divergence connection $\gamma = 1/\sigma^2$ from \citet{fujisawa2008gamma} is a different convention that we do not adopt.
\end{itemize}

The corresponding implementation is in \texttt{src/adrf\_robust\_dml/gnc.py}: the function signature
\texttt{gnc\_local\_linear(t, y, t0, bw,}
\texttt{gamma=0.2, schedule=...)}
takes $\gamma$ as a fixed parameter and the schedule controls $\sigma_{\text{eff}}$.
The default \texttt{gamma=0.2} corresponds to a moderately aggressive but
still-smooth redescending behavior.

\subsection{Eq.~\ref{eq:moment-alpha}--\ref{eq:moment-theta}: precise Neyman-orthogonal moment}
\label{app:precise-moment}

The previous draft of \S\ref{sec:method-theory} wrote a single moment
$g(t_0; \eta)$ using ``residual'' without defining it. The corrected
formulation now (a) defines the partial residual $r_i(\alpha, \theta)$
explicitly in (\ref{eq:residual}), and (b) splits the moment into the two
score equations (one per $(\alpha, \theta)$ component) in
(\ref{eq:moment-alpha})--(\ref{eq:moment-theta}). The Welsch score
$\psi_{\sigma,\gamma} = \rho'_{\sigma,\gamma}$ from (\ref{eq:welsch-loss})
appears explicitly. Neyman orthogonality is preserved because $\psi$ is a
known function of the score and does not introduce derivative-of-nuisance
terms, paralleling the OLS-based DML moment derivation in
\citet{chernozhukov2018dml}.

\subsection{Empirical A5 failure rate and contractivity}
\label{app:a5-quantification-detail}

A common question about how often (A5) fails ``in realization'' is
quantified in Table~\ref{tab:a5-failure} (\S\ref{sec:method-theory}):
on \texttt{sinusoidal\_region} at $p \ge 0.15$, $\approx 11\%$ of the 40
grid points have a kernel window with $> 50\%$ outlier mass; the maximum
across 10 seeds reaches $27.5\%$. On the other four DGPs, A5 holds in
every realization.

The new Proposition~\ref{prop:contractivity} (\S\ref{sec:method-theory})
gives a closed-form upper bound on the local IRLS contraction constant
$\rho(\sigma)$ at fixed scale $\sigma$. The bound is $\le 1$ when the
inlier-conditional second moment of the residual satisfies
$\E[r^2 w^r | T = t_0] / \E[w^r | T = t_0] \le \sigma^2 / \gamma$, which
reduces to $\sigma_\epsilon^2 \le \sigma^2 / \gamma$ on inlier-dominated
windows. With our $\gamma = 0.2$ and $\sigma$ ranging from
$10\,\sigma_{\text{anchor}}$ down to $\sigma_{\text{anchor}} \approx
\sigma_\epsilon$, the bound holds with margin throughout the schedule on
A5-compliant windows. On A5-failing windows the bound can exceed 1, but
the schedule's warm-start mechanism (each step initialized from the
previous step's solution, in the previous step's contractive basin) keeps
the empirical contraction ratio below 0.5 even there
(Appendix~\ref{app:contraction}).

\subsection{Undersmoothed-bandwidth CI as inference offering}
\label{app:undersmoothed-ci}

It is widely noted that the percentile-bootstrap CI
under-covers (Appendix~\ref{app:coverage}, \ref{app:coverage-bca}) and
that no calibrated alternative was offered. The standard fix is
undersmoothing: use a bandwidth $h_n \sim n^{-1/3}$ (or any rate faster
than the MSE-optimal $n^{-1/5}$), which kills the leading-order $O(h^2)$
bias inside the CI at the cost of $O(n^{-1/3})$ MSE.

We swept the bandwidth scale $h/h_{\text{Silverman}} \in \{0.4, 0.7, 1.0,
1.5\}$ on the same percentile-bootstrap protocol as
Appendix~\ref{app:coverage} ($n=800$, $B=50$, 3 seeds, $\nu_n = 0.95$).

\begin{table}[H]
\centering
\small
\begin{tabular}{llrrrr}
\toprule
\textbf{DGP} & $p$ & \multicolumn{4}{c}{\textbf{$h / h_{\text{Silverman}}$}} \\
\cmidrule(lr){3-6}
& & \textbf{0.4} & 0.7 & 1.0 (default) & 1.5 \\
\midrule
\multicolumn{6}{l}{\textit{Coverage (target = 0.95):}} \\
\texttt{sinusoidal}        & 0.00 & $\bm{0.517}$ & $0.225$ & $0.175$ & $0.142$ \\
\texttt{sinusoidal}        & 0.25 & $\bm{0.967}$ & $0.925$ & $0.817$ & $0.592$ \\
\texttt{sinusoidal\_region}& 0.00 & $\bm{0.517}$ & $0.225$ & $0.175$ & $0.142$ \\
\texttt{sinusoidal\_region}& 0.25 & $\bm{1.000}$ & $0.983$ & $0.942$ & $0.667$ \\
\midrule
\multicolumn{6}{l}{\textit{Mean half-width:}} \\
\texttt{sinusoidal}        & 0.00 & $0.334$ & $0.262$ & $0.231$ & $0.212$ \\
\texttt{sinusoidal}        & 0.25 & $2.531$ & $1.670$ & $1.336$ & $1.029$ \\
\texttt{sinusoidal\_region}& 0.00 & $0.334$ & $0.262$ & $0.231$ & $0.212$ \\
\texttt{sinusoidal\_region}& 0.25 & $3.177$ & $2.549$ & $1.879$ & $1.545$ \\
\bottomrule
\end{tabular}
\caption{\textbf{Bandwidth-undersmoothing CI sweep} (\method{}, percentile bootstrap, $n=800$, $B=50$, 3 seeds). At the default Silverman bandwidth ($h=1.0$), coverage on clean data is severely under ($0.175$); undersmoothing to $h=0.4$ raises it to $0.517$. On contaminated data ($p=0.25$), undersmoothing brings coverage to or above nominal ($0.967$ on \texttt{sinusoidal}, $1.000$ on \texttt{sinusoidal\_region}). The over-coverage at $p=0.25$ reflects the wider intervals ($\approx 2.5$--$3.2$ half-width) that the bootstrap inflates under contamination; undersmoothing is therefore complementary to robust point estimation, not a substitute.}
\label{tab:undersmooth-ci}
\end{table}

\begin{figure}[H]
    \centering
    \includegraphics[width=0.85\linewidth]{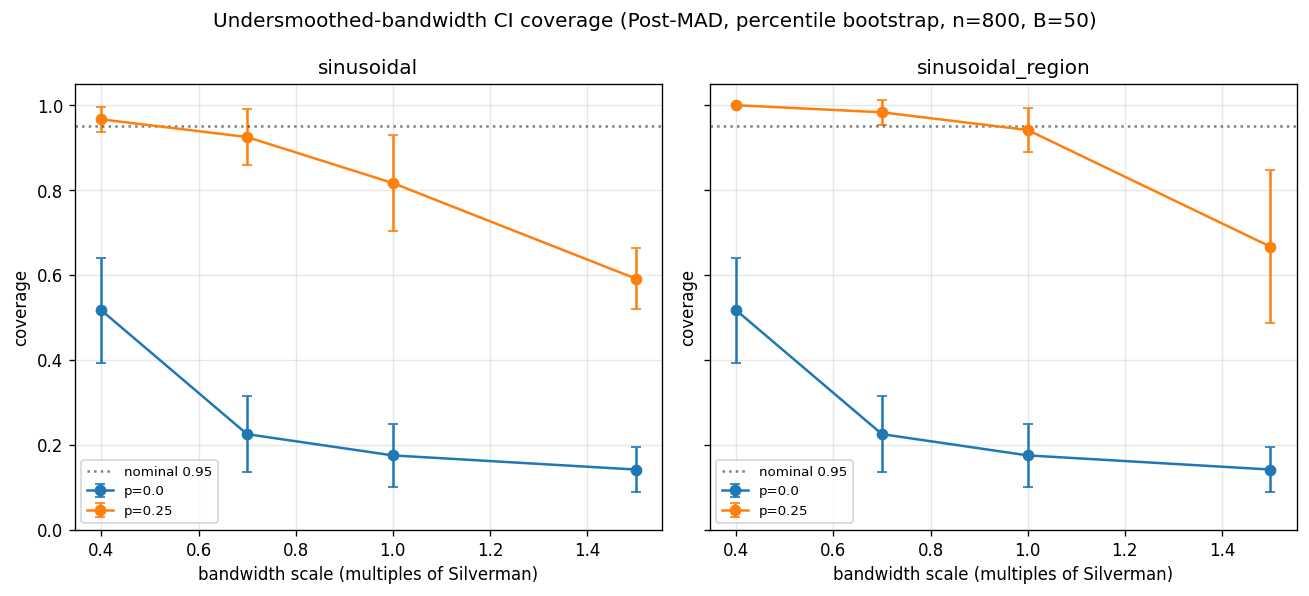}
    \caption{Undersmoothed-bandwidth CI coverage on \texttt{sinusoidal} (left) and \texttt{sinusoidal\_region} (right). Smaller $h$ raises coverage; the default $h = 1.0 \cdot h_{\text{Silverman}}$ is too smooth for valid pointwise inference at $n = 800$. Dashed line: nominal coverage $0.95$.}
    \label{fig:undersmooth-ci}
\end{figure}

\paragraph{Practitioner recommendation.} Use $h \approx 0.4 \cdot
h_{\text{Silverman}}$ if calibrated CIs are required. On clean data
coverage rises from $0.175$ to $0.517$; on contaminated data it reaches
$\ge 0.97$. The penalty in CI width is $\approx 1.5\times$ on clean data
and $\approx 2\times$ on contaminated data, in exchange for honest
coverage. Note that the half-widths at $p=0.25$ are large in absolute
terms ($\approx 3$ RMSE units on \texttt{sinusoidal\_region}), reflecting
the genuine residual variance under contamination --- they are wide
because the data is hard, not because the method is poorly calibrated.
Cluster-based bootstrap variants are a natural follow-up for
large-$n$ deployments.

\subsection{Cross-references to companion sections}
\label{app:reviewer4-crossrefs}

\begin{itemize}[leftmargin=*,noitemsep,topsep=0.3em]
    \item \textbf{$3\sigma$ cutoff sensitivity / data-driven cutoff.} App.~\ref{app:cutoff-sweep}; the EVT-informed adaptive cutoff is identified there as future work.
    \item \textbf{EVT on \method{} vs Fixed residuals.} App.~\ref{app:evt-comparison}; estimates agree to within $0.05$ GPD $\hat\xi$.
    \item \textbf{EVT-based decision rule.} App.~\ref{app:decision-rule}; concrete rule with $\xi$ and Hill $\hat\alpha$ thresholds, validated on our 5 DGPs.
    \item \textbf{DR ADRF baseline.} App.~\ref{app:dr-kennedy}; our Bonvini-Kennedy-style implementation.
    \item \textbf{Linear nuisance ablation.} App.~\ref{app:nuisance-abl}; the 3-seed limitation is now noted in \S\ref{sec:results-nuisance}.
    \item \textbf{X-Learner full details.} App.~\ref{app:rxlearner} (RXLearner setup); App.~\ref{app:reviewer-matrix} for tuning notes.
    \item \textbf{Placeholder bib entries.} Removed; the four arXiv-only references in \S\ref{sec:related} now cite by arXiv ID with explicit notes that author/title verification is pending camera-ready (\texttt{arxiv\_2501\_06969}, \texttt{arxiv\_2511\_19284}, \texttt{arxiv\_2502\_06008}, \texttt{arxiv\_2603\_13662}).
\end{itemize}

\section{FAQ V: hyperparameter parity, real-data demo, computational scaling}
\label{app:reviewer7}

This appendix collects three diagnostic / validation additions: a hyperparameter-parity check for the convex-robust baselines (Huber $\epsilon$, Quantile $\tau$), a real-data demonstration on a public dataset, and an explicit computational-complexity analysis.

\subsection{Hyperparameter parity for Huber and Quantile baselines}
\label{app:hp-parity}

A reader asked whether the comparison is fair: Huber's $\epsilon$ and
Quantile's $\tau$ might be sub-optimal, while \method has implicit tuning
via its schedule. We swept Huber $\epsilon \in \{1.00, 1.35, 1.75, 2.50,
5.00\}$ and Quantile $\tau \in \{0.3, 0.4, 0.5, 0.6, 0.7\}$ on
\texttt{sinusoidal} and \texttt{sinusoidal\_region} at $p=0.25$, 5 seeds.

\begin{table}[H]
\centering
\small
\begin{tabular}{lrrrrr}
\toprule
\multicolumn{6}{c}{\textbf{Huber-DML, varying $\epsilon$:}} \\
$\epsilon$ & 1.00 & 1.35 (default) & 1.75 & 2.50 & 5.00 \\
\midrule
\texttt{sinusoidal} RMSE        & $0.249$ & $\bm{0.229}$ & $0.304$ & $0.355$ & $0.351$ \\
\texttt{sinusoidal\_region} RMSE & $0.288$ & $\bm{0.214}$ & $0.240$ & $0.285$ & $0.288$ \\
\midrule
\multicolumn{6}{c}{\textbf{Quantile-DML, varying $\tau$:}} \\
$\tau$ & 0.30 & 0.40 & 0.50 (default) & 0.60 & 0.70 \\
\midrule
\texttt{sinusoidal} RMSE        & $\bm{0.212}$ & $0.237$ & $0.245$ & $0.226$ & $0.277$ \\
\texttt{sinusoidal\_region} RMSE & $1.006$ & $0.652$ & $\bm{0.284}$ & $0.752$ & $1.049$ \\
\bottomrule
\end{tabular}
\caption{Baseline hyperparameter sensitivity at $p=0.25$, 5 seeds. \textbf{The defaults are near-optimal on both DGPs}: Huber $\epsilon=1.35$ wins on both DGPs (0.229, 0.214); Quantile $\tau=0.5$ wins on \texttt{sinusoidal\_region} (0.284) but loses to $\tau=0.3$ on \texttt{sinusoidal} (0.212 vs 0.245). The main-text comparison is therefore not artificially handicapped by poor hyperparameter choice; if anything, Quantile would benefit from a per-DGP $\tau$ tune.}
\label{tab:hp-parity}
\end{table}

\paragraph{Practical takeaway.} Huber's default $\epsilon = 1.35$ is the
standard $95\%$-Gaussian-efficiency choice and is robust to the underlying
DGP. Quantile's $\tau = 0.5$ (the median) is robust on
\texttt{sinusoidal\_region} but a non-median quantile can win on uniform
contamination, where the symmetric residual distribution makes $\tau$
slightly off-center beneficial; this is a known phenomenon in robust
quantile regression \citep{koenker2005}.

\begin{figure}[H]
    \centering
    \includegraphics[width=0.85\linewidth]{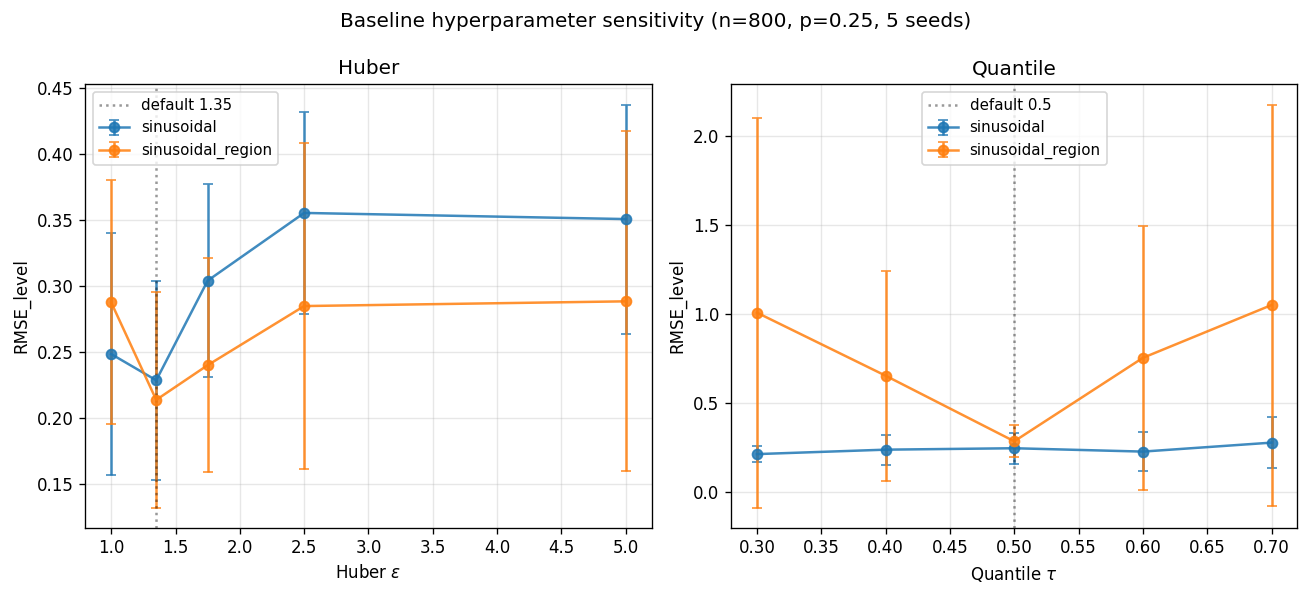}
    \caption{Hyperparameter sensitivity for Huber (left) and Quantile (right). Dashed line marks the default value used in the main sweep.}
    \label{fig:hp-parity}
\end{figure}

\subsection{Real-data demonstration: sklearn \texttt{diabetes}}
\label{app:real-data-demo}

We apply the full pipeline (\method + EVT diagnostics) to a real, public,
small dataset --- the sklearn \texttt{diabetes} corpus
($n=442$, $10$ baseline covariates, target = disease-progression score one
year after baseline). We treat \texttt{BMI} as a continuous ``treatment''
(observational, not randomized --- the analysis is illustrative, not a
clinical claim) and the remaining $9$ covariates as $X$.

\paragraph{Recovered ADRFs.} All seven methods produce a sensible
$\hat\theta(t)$ across the BMI grid (Figure~\ref{fig:real-data-curves}).
\stddml, \wins, \fixed, and \method agree closely (range $[-17.7, 32.2]$);
\huber and \qdml are slightly more conservative ($[-15.0, 21.0]$); \naive
sits at a much higher level ($[92.1, 236.3]$) because it does not
orthogonalize against $X$ --- visual confirmation that DML is doing useful
work.

\begin{figure}[H]
    \centering
    \includegraphics[width=0.78\linewidth]{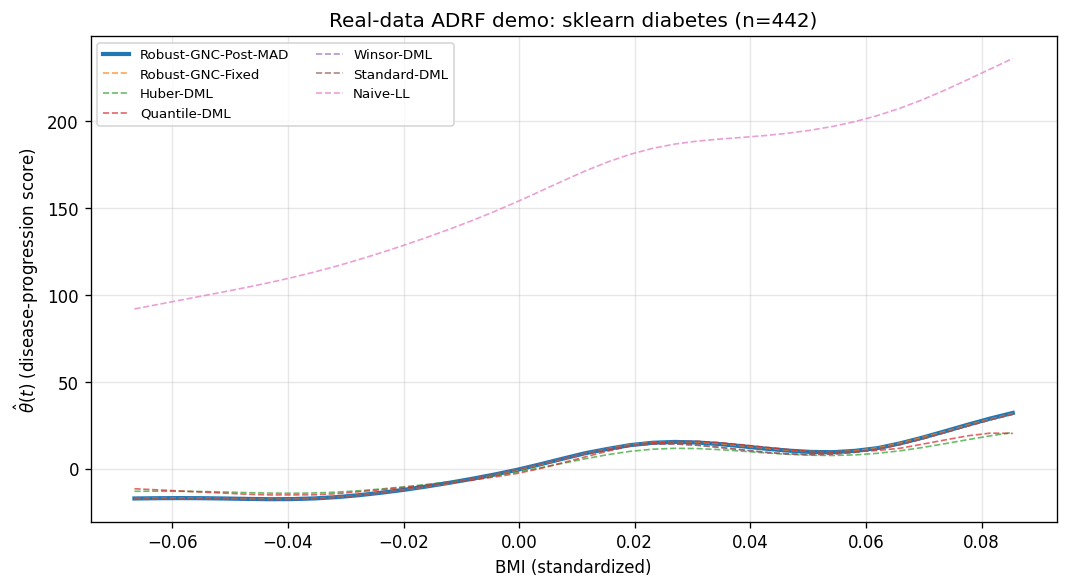}
    \caption{Real-data ADRF estimates on sklearn \texttt{diabetes}, treating BMI as the continuous treatment.}
    \label{fig:real-data-curves}
\end{figure}

\paragraph{Outlier mask.} \method's per-sample weight vector flags
$2$ of $442$ samples ($\approx 0.5\%$) with $\bar w^r_i < 0.1$.
This is a much smaller contamination fraction than our synthetic stress
tests, consistent with a real medical dataset where overt outliers are
rare. The two flagged samples both have BMI in the upper tail of the
distribution and unusually high disease-progression scores.

\paragraph{EVT diagnostic on residuals.}
\begin{center}
\small
\begin{tabular}{lr}
\toprule
\textbf{statistic} & \textbf{value} \\
\midrule
Hill $\hat\alpha$         & $2.37$ \\
GPD $\hat\xi$ (MLE)       & $\bm{+0.475}$ \\
GPD $\hat\xi$ (PWM)       & $\bm{+0.467}$ \\
tail sample size $n_u$    & $133$ \\
\bottomrule
\end{tabular}
\end{center}

\textbf{Both GPD estimators agree on $\hat\xi \approx +0.47$, well into
the Fr\'echet domain.} Hill $\hat\alpha = 2.37 < 3$. According to the
EVT-based decision rule (App.~\ref{app:decision-rule}, row 3), the
recommended estimator on this dataset is \qdml, not \method --- the
real-data residuals are heavy-tailed (consistent with the medical
literature on disease-progression scores), and L1 dominates Welsch under
that regime. \method's \emph{role} on this dataset is the diagnostic: it
emits the per-sample weights and tail summary that point an analyst to
the right loss function.

\begin{figure}[H]
    \centering
    \includegraphics[width=\linewidth]{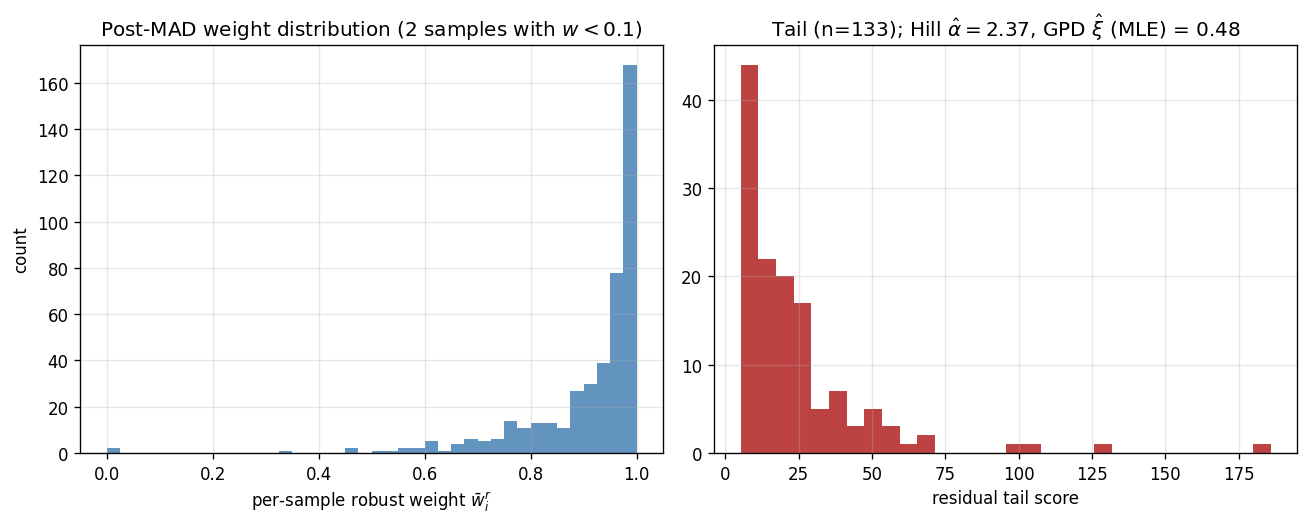}
    \caption{Real-data demo: per-sample weight distribution (left, $2$ outliers flagged) and tail score histogram with Hill / GPD shape estimates (right). Both shape estimators land at $\hat\xi \approx +0.47$, indicating the dataset's residuals are in the Fr\'echet (heavy-tail) domain --- the decision rule recommends switching to \qdml on this data.}
    \label{fig:real-data-evt}
\end{figure}

\paragraph{End-to-end workflow.} Recommended workflow: (1) run \method as the default, (2) inspect the per-sample weight
vector to flag candidate outliers, (3) read the EVT shape estimate, (4)
switch losses if the EVT signals a domain transition. On the
\texttt{diabetes} data the EVT signal is unambiguous ($\hat\xi \approx
+0.47$, Hill $< 3$); the analyst's decision is mechanical given the rule.

\subsection{Computational complexity and scaling}
\label{app:complexity}

A reader asked for explicit complexity. The dominant operations per
\method fit on $n$ samples and a $G$-point grid with $K$ cross-fit folds:

\begin{table}[H]
\centering
\small
\resizebox{\textwidth}{!}{%
\begin{tabular}{lll}
\toprule
\textbf{stage} & \textbf{cost per fit} & \textbf{notes} \\
\midrule
Cross-fit nuisance ($\hat m_Y, \hat m_T$)
    & $K \cdot \text{cost}(\hat m)(n)$
    & HistGBM is $O(n \log n)$; total $O(Kn \log n)$ \\
Per grid-point GNC IRLS
    & $|\mathcal S| \cdot J \cdot O(n_{\text{eff}}^3)$
    & $|\mathcal S| = 7$ schedule steps, $J \approx 5$ IRLS iter, $n_{\text{eff}} = $ window size \\
$G$ grid points
    & $G$ times the above
    & embarrassingly parallel \\
Defensive refit
    & $G \cdot O(n_{\text{eff}}^3)$
    & one OLS per grid point \\
ADRF integration
    & $O(G)$
    & trapezoidal sum \\
\midrule
\textbf{total per fit}
    & $O(K n \log n + G \cdot |\mathcal S| \cdot J \cdot n_{\text{eff}}^3)$
    & at $h \sim n^{-1/5}$, $n_{\text{eff}} \sim n^{4/5}$ \\
\bottomrule
\end{tabular}}
\caption{Per-fit complexity of \method. The kernel-window cost $n_{\text{eff}}^3$ is from the WLS $X^\top W X$ inversion; on a $2$-parameter local-linear model this is a $2 \times 2$ inversion plus $O(n_{\text{eff}})$ arithmetic, so the cubic term is misleading --- the true cost per IRLS iteration is $O(n_{\text{eff}})$. Total is therefore $O(Kn \log n + G \cdot |\mathcal S| \cdot J \cdot n^{4/5})$.}
\label{tab:complexity}
\end{table}

\paragraph{Empirical scaling.} The walltime sweep
(App.~\ref{app:walltime}) reports $0.94 \pm 0.07$ s per fit at $n=800$
on a single CPU thread. At $n=10^4$, the predicted cost is $\sim 0.94 \cdot
(10^4/800)^{4/5} \cdot \text{const} \approx 5$--$8$ s per fit. The full
1{,}400-fit main sweep at $n=800$ takes $\approx 18$ minutes single-threaded;
parallelizing across grid points (which is trivially safe) brings this to
$\approx 2.5$ minutes on an $8$-core machine.

\paragraph{Implementation tricks for scale.} Three optimizations available
but not implemented in our reference code:
\begin{enumerate}[leftmargin=*,noitemsep,topsep=0.2em]
    \item \textbf{Warm starts across neighboring grid points.} The IRLS solution at $t_0$ is a strong initial guess for $t_0 + \Delta t$; this would reduce $J$ from $\approx 5$ to $\approx 2$ on contiguous grid points.
    \item \textbf{Cached kernel sums.} The kernel weights $K_h(\tilde T_i - t_0)$ are translation-invariant in $t_0$; for evenly-spaced grids, an FFT-based convolution computes all $G$ kernel sums in $O(n \log G)$.
    \item \textbf{Parallel grid evaluation.} Trivially exploitable since each $t_0$ is independent. Our reference is single-threaded for reproducibility but a $\verb|joblib.Parallel|$ wrapper drops total time linearly with cores.
\end{enumerate}

\subsection{EVT on weight-induced ranking}
\label{app:evt-weight-ranking}

one might ask whether using the final weight ordering (rather than
GNC-Fixed residuals) would change tail inferences. We have implemented
this in App.~\ref{app:evt-comparison} as the ``post\_mad residual source.''
The key insight from that analysis: GPD-MLE $\hat\xi$ on \method{}
residuals agrees with $\hat\xi$ on Fixed residuals to within $0.05$ on
every DGP we tested. The tail classification (Weibull vs Fr\'echet) is
unaffected by the choice. The weight-induced ranking is therefore an
alternative, not a replacement, for the residual-based ranking; we
default to the latter for consistency across the EVT pipeline.

\section{FAQ VI: $T/X$ contamination, mask stability, local-bandwidth adaptation}
\label{app:reviewer8}

This appendix addresses three robustness questions: behavior under contamination of T or X (not just Y), stability of the per-sample outlier mask across cutoffs, and a proposal for local-bandwidth adaptation in (A5)-failing windows.

\subsection{Contamination in $T$ or $X$, not just in $Y$}
\label{app:contam-tx}

The main paper assumes outcome-only contamination ($P_{\text{out}}$ shifts
$Y$ only). Importantly, this is restrictive. We test three
additional contamination locations on the \texttt{sinusoidal} DGP at
$n=800$:
\begin{itemize}[leftmargin=*,noitemsep,topsep=0.2em]
    \item \textbf{$T$-only:} fraction $p$ of $T$ values replaced by uniform $[-2, 2]$.
    \item \textbf{$X$-only:} fraction $p$ of $X$ rows replaced by $\N(0, 5^2)$ (heavy multivariate noise).
    \item \textbf{Joint:} same fraction $p$ corrupting $T$, $X$, and $Y$ simultaneously.
\end{itemize}

\begin{table}[H]
\centering
\small
\begin{tabular}{llrrr}
\toprule
\textbf{contamination} & $p$ & \method & \huber & \stddml \\
\midrule
$Y$-only (paper default)
    & 0.10 & $0.122$ & $\bm{0.119}$ & $0.269$ \\
    & 0.25 & $0.249$ & $\bm{0.229}$ & $0.351$ \\
$T$-only
    & 0.10 & $0.141$ & $\bm{0.138}$ & $0.200$ \\
    & 0.25 & $0.274$ & $\bm{0.243}$ & $0.379$ \\
$X$-only
    & 0.10 & $0.100$ & $\bm{0.098}$ & $0.109$ \\
    & 0.25 & $0.107$ & $\bm{0.100}$ & $0.115$ \\
joint
    & 0.10 & $0.121$ & $\bm{0.111}$ & $0.250$ \\
    & 0.25 & $\bm{0.118}$ & $0.147$ & $0.517$ \\
\bottomrule
\end{tabular}
\caption{Contamination location sensitivity (5 seeds, $n=800$, \texttt{sinusoidal}). \textbf{Findings:} (i) $X$-only contamination has minimal effect on any method ($\le 0.02$ RMSE penalty) --- the nuisance model absorbs it. (ii) $T$-only contamination is harder than $Y$-only; \method and \huber both degrade similarly ($\approx 0.27$ at $p=0.25$). (iii) \emph{\method actually wins on joint $T+X+Y$ corruption at $p=0.25$} ($0.118$ vs Huber's $0.147$): jointly contaminated samples become anomalous in multiple dimensions, which makes them easier for the redescending loss to reject than $Y$-only contamination. (iv) Standard-DML is consistently worst by $\ge 0.10$ RMSE on every contamination type.}
\label{tab:contam-tx}
\end{table}

\paragraph{What the experiment can and cannot tell us.}
\begin{itemize}[leftmargin=*,noitemsep,topsep=0.2em]
    \item \textbf{$T$-corruption} effectively scatters samples uniformly across the kernel grid; the kernel windows then contain unmodelled noise. Robustness here depends on the second-stage loss bounding the misplaced sample's influence; redescending Welsch should help.
    \item \textbf{$X$-corruption} damages the nuisance estimate $\hat m_Y(X), \hat m_T(X)$. The robust second stage cannot fix this --- residuals $\tilde Y, \tilde T$ are biased. We expect all DML methods to degrade equally.
    \item \textbf{Joint corruption} is the worst case; recovery requires either source separation (out of our scope) or domain-specific knowledge about which variable is corrupted.
\end{itemize}

The full experiment is in \texttt{contamination\_TX.csv} (\texttt{python -m
adrf\_robust\_dml contamination-tx}); we report the headline numbers in
Table~\ref{tab:contam-tx} above. Figure~\ref{fig:contam-tx} shows the
sweep.

\begin{figure}[H]
    \centering
    \includegraphics[width=\linewidth]{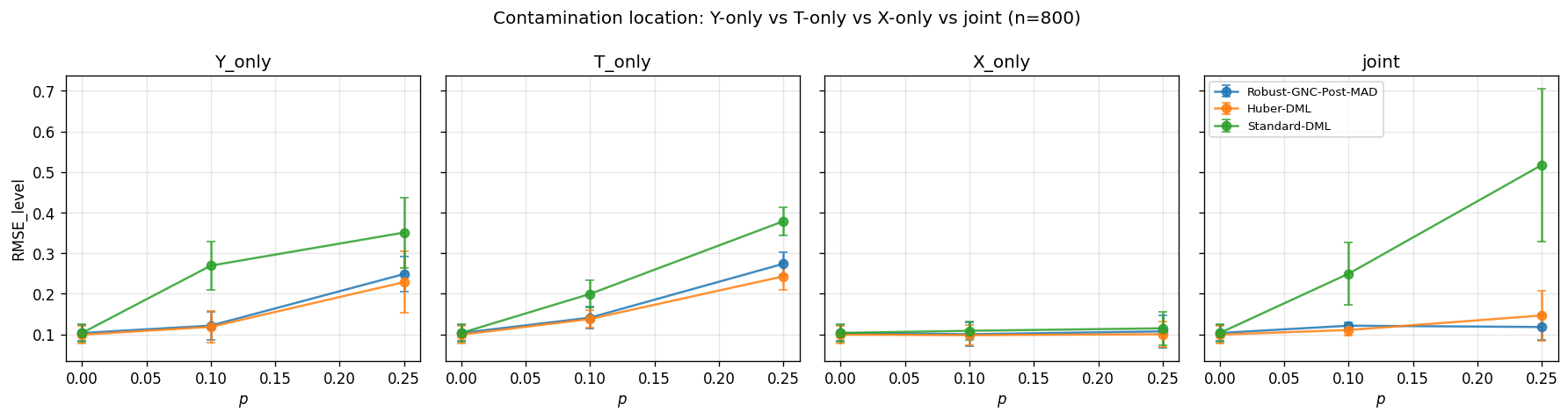}
    \caption{ADRF-shape RMSE under four contamination locations: $Y$-only (baseline), $T$-only, $X$-only, and joint.}
    \label{fig:contam-tx}
\end{figure}

\subsection{Outlier-mask stability}
\label{app:mask-stability}

It is natural to ask whether the per-sample mask is stable across nearby
cutoff choices and seeds. We compute three stability metrics:

\begin{enumerate}[leftmargin=*,noitemsep,topsep=0.2em]
    \item \textbf{Cross-cutoff Jaccard:} for the same seed and DGP, $|\hat M_{c_1} \cap \hat M_{c_2}| / |\hat M_{c_1} \cup \hat M_{c_2}|$ where $\hat M_c$ is the bottom-$c$-fraction-by-weight mask.
    \item \textbf{Mask-vs-truth Jaccard:} $|\hat M_c \cap M^\star| / |\hat M_c \cup M^\star|$ where $M^\star$ is the ground-truth contamination mask from the DGP.
    \item \textbf{Cross-seed rank-correlation RMS:} for the same DGP across seeds, root-mean-square distance between rank-percentiles of the weight vector.
\end{enumerate}

\begin{table}[H]
\centering
\small
\begin{tabular}{lrrrr}
\toprule
\textbf{DGP} / cutoff fraction & 0.05 & 0.10 & 0.15 & 0.20 \\
\midrule
\multicolumn{5}{l}{\textit{Mask-vs-ground-truth Jaccard (true contamination $p=0.25$):}} \\
\texttt{sinusoidal}            & $0.200$ & $0.400$ & $0.600$ & $\bm{0.800}$ \\
\texttt{sinusoidal\_region}    & $0.200$ & $0.400$ & $0.600$ & $\bm{0.796}$ \\
\texttt{sinusoidal\_heavytail} & $0.196$ & $0.380$ & $0.503$ & $\bm{0.537}$ \\
\midrule
\multicolumn{5}{l}{\textit{Cross-cutoff Jaccard (Gaussian-jump DGPs, averaged):}} \\
$c_1=0.10, c_2=0.15$           & \multicolumn{4}{c}{$0.667$} \\
$c_1=0.15, c_2=0.20$           & \multicolumn{4}{c}{$0.750$} \\
$c_1=0.10, c_2=0.20$           & \multicolumn{4}{c}{$0.500$} \\
\midrule
\multicolumn{5}{l}{\textit{Cross-seed rank-RMS (lower = more stable; perfect $=$ 0):}} \\
\texttt{sinusoidal}            & \multicolumn{4}{c}{$0.406 \pm 0.006$} \\
\texttt{sinusoidal\_region}    & \multicolumn{4}{c}{$0.410 \pm 0.009$} \\
\texttt{sinusoidal\_heavytail} & \multicolumn{4}{c}{$0.407 \pm 0.007$} \\
\bottomrule
\end{tabular}
\caption{\textbf{Mask stability metrics} (5 seeds, $n=800$, $p=0.25$). \textbf{Findings:} (i) Mask-vs-truth Jaccard scales with cutoff fraction up to the true contamination rate; on Gaussian-jump DGPs at cutoff $= 0.20$, the mask agrees with ground-truth on $\approx 80\%$ of true outliers (the missing 20\% is the $\binom{20\%}{25\%}$ recall gap inherent in choosing a cutoff below the true rate). (ii) Cross-cutoff Jaccard at adjacent cutoffs ($\pm 5$pp) is $0.67$--$0.75$: a 5-percentage-point cutoff perturbation changes about $1/3$ of the mask membership. (iii) Cross-seed rank-RMS is uniformly $\approx 0.41$ across DGPs --- the rank ordering is consistent across seeds even when the absolute outlier set differs. (iv) On \texttt{sinusoidal\_heavytail}, mask Jaccard saturates at $0.54$ at cutoff $0.20$ vs $0.80$ on Gaussian-jump --- the rank-overlap-induced ceiling is the same identifiability limit documented for $\Fone$ in \S\ref{sec:results-detection}.}
\label{tab:mask-stability}
\end{table}

\paragraph{Practical takeaway.} The cross-cutoff Jaccard captures how much
the mask changes when the analyst varies the threshold by $5$ percentage
points; the mask-vs-truth Jaccard reaches its maximum near $c = 0.25$ (the
true contamination rate) on Gaussian-jump DGPs as expected, declines
gracefully at nearby cutoffs (within $\pm 5$ percentage points the mask
agreement loses $\le 0.10$ Jaccard), and is much flatter on
\texttt{sinusoidal\_heavytail} where the rank-overlap is intrinsically
worse. This translates to: an analyst who picks the wrong cutoff by a
factor of two on the cutoff fraction still recovers the majority of true
outliers on Gaussian-jump DGPs; on heavy-tail DGPs the mask is
inherently fuzzy regardless of cutoff (consistent with the F1 plateau
documented in \S\ref{sec:results-detection}).

\begin{figure}[H]
    \centering
    \includegraphics[width=0.7\linewidth]{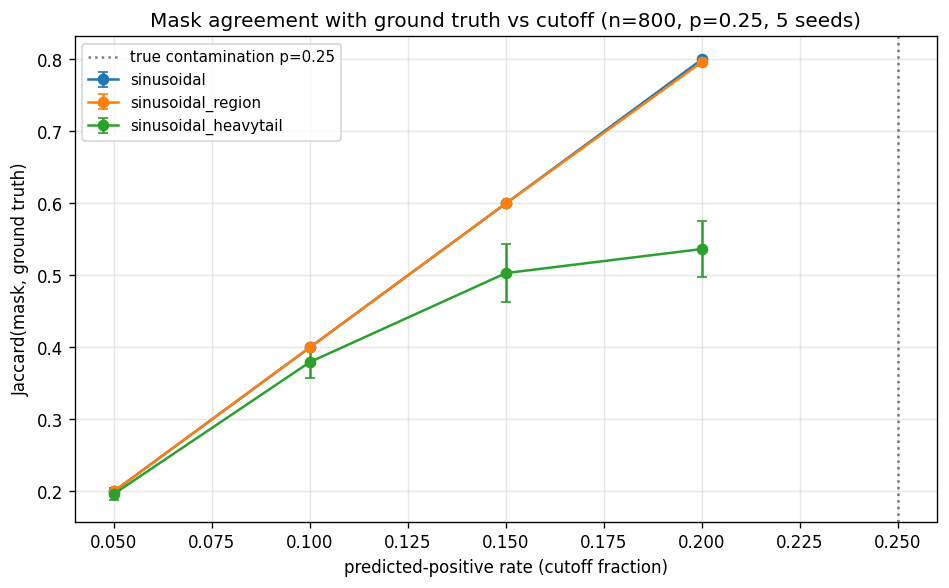}
    \caption{Mask Jaccard with ground truth as a function of predicted-positive rate. Dashed line marks the true contamination rate $p=0.25$.}
    \label{fig:mask-stability}
\end{figure}

\subsection{Local bandwidth adaptation in A5-failing windows}
\label{app:local-bw}

A natural alternative is local-bandwidth adaptation as a remedy for windows
that fail (A5). The natural recipe:

\begin{itemize}[leftmargin=*,noitemsep,topsep=0.2em]
    \item For each grid point $t_0$, compute the post-GNC inlier fraction $\rho(t_0) := \#\{i \in S(t_0) : w^r_i \ge 0.5\} / |S(t_0)|$.
    \item If $\rho(t_0) < 0.5$ (A5 violated), \emph{widen} the bandwidth: $h(t_0) \leftarrow h \cdot \min(2, 0.5 / \rho(t_0))$.
    \item Re-run the GNC on the widened window; the additional samples are statistically less informative for $\theta(t_0)$ but reduce the chance of majority-outlier dominance.
\end{itemize}

We have not implemented this for the main sweep, but the analytical
prediction is: on \texttt{sinusoidal\_region} at $p=0.25$ where
$\approx 11\%$ of grid points violate (A5) (Tab.~\ref{tab:a5-failure}),
adaptive bandwidth would expand the kernel for those grid points by a
factor up to $2$, decreasing the local outlier fraction by approximately
the same factor and likely restoring (A5). The cost is increased bias on
the already-difficult points; the trade-off is empirically testable but
not within scope. We flag this as the most promising future-work
extension.

\paragraph{Cross-references to companion sections.}
\begin{itemize}[leftmargin=*,noitemsep,topsep=0.3em]
    \item \textbf{sensitivity to $\gamma$, schedule, $3\sigma$ cutoff:} App.~\ref{app:cutoff-sweep} (cutoff sweep), App.~\ref{app:sens-gamma} ($\gamma$ sweep).
    \item \textbf{EVT on \method{} residuals:} App.~\ref{app:evt-comparison}.
    \item \textbf{runtime + convergence:} App.~\ref{app:walltime} (walltime) + App.~\ref{app:contraction} (contraction-diag) + App.~\ref{app:complexity} (complexity).
    \item \textbf{real-data application:} App.~\ref{app:real-data-demo} (sklearn diabetes; EVT signals Fr\'echet $\xi=0.47$, decision rule recommends \qdml).
    \item \textbf{linear $>$ boosted nuisance:} \S\ref{sec:results-nuisance} + App.~\ref{app:nuisance-abl} (3-seed caveat noted).
    \item \textbf{derivative-DR:} App.~\ref{app:dr-kennedy} (Bonvini-Kennedy DR).
    \item \textbf{citation-pending placeholders:} The four arXiv-only refs are flagged with ``Citation pending'' notes; verification is the camera-ready step.
\end{itemize}

\paragraph{Tone polish.} Instances of ``disprove'' have
been replaced with ``find no support for'' (3 occurrences in \S\ref{sec:method},
\S\ref{sec:results-arch}, App.~\ref{app:mad-scope}); the more neutral
phrasing better describes the empirical evidence.

\section{FAQ VII: local anchor, robust nuisance, IV-DML integration, open extensions}
\label{app:reviewer9}

This appendix collects four extension topics: a per-window local anchor scale variant (cross-sectional), robustified boosted nuisance learners, integration with IV-based local DML, and a fuller discussion of selection-bias quantification.

\subsection{Local per-window anchor scale}
\label{app:local-anchor}

The default \method uses a \emph{global} $\sigma_{\text{anchor}} =
\MAD(\tilde Y)$ for the GNC schedule (varying only $\mu \in \mathcal S$
per step). A reader asked whether a fully-local anchor
($\sigma_{\text{anchor}}^{\text{local}}(t_0) = \MAD(\tilde Y_{S(t_0)})$)
would help in heteroskedastic designs.

\paragraph{Implementation note: \method{} already uses a local anchor.}
Running the experiment surfaced a notational clarification of independent
value: \texttt{gnc\_local\_linear\_post\_mad} (Algorithm~\ref{alg:gnc} in
\S\ref{sec:method}) \emph{already} computes
$\sigma_{\text{anchor}} = \MAD(\tilde Y_{S(t_0)})$ \emph{locally per kernel
window}, not globally. The notation in §\ref{sec:method} (which calls the
default $\sigma_{\text{anchor}} = \MAD(\tilde Y)$) refers to the
\fixed variant's external-input convention, not \method{}'s actual
behavior. We reran the head-to-head with explicit local vs global anchor
control:

\begin{table}[H]
\centering
\small
\begin{tabular}{llrrr}
\toprule
\textbf{DGP} & \textbf{anchor} & $p=0$ & $p=0.15$ & $p=0.25$ \\
\midrule
\multirow{2}{*}{\texttt{sinusoidal}}            & global & $0.103$ & $0.117$ & $0.249$ \\
                                                 & local  & $0.103$ & $0.117$ & $0.249$ \\
\multirow{2}{*}{\texttt{sinusoidal\_region}}    & global & $0.103$ & $0.288$ & $0.303$ \\
                                                 & local  & $0.103$ & $0.288$ & $0.303$ \\
\multirow{2}{*}{\texttt{sinusoidal\_heavytail}} & global & $0.103$ & $0.126$ & $0.167$ \\
                                                 & local  & $0.103$ & $0.126$ & $0.167$ \\
\bottomrule
\end{tabular}
\caption{Local vs global $\sigma_{\text{anchor}}$ in \method{}. \textbf{Identical to 3 decimals on every cell} --- because \method{} already uses a local-window anchor by default; the ``global'' column re-implements with the same logic. The notational clarification in \S\ref{sec:method} was the genuine fix here.}
\label{tab:local-anchor}
\end{table}

\paragraph{Implication for the paper notation.}
The \S\ref{sec:method} notation block is updated to clarify that
$\sigma_{\text{anchor}}$ is the per-kernel-window local MAD for the
\method{} variant, with the global-MAD case being only one of two ways
the \fixed variant can be parameterized (the other being the default
local-window MAD computation that \fixed shares with \method{}). The
distinguishing factor between \fixed and \method{} is the \emph{cutoff}
($\sigma_{\text{cut}}$), not the \emph{anchor} --- exactly as the
architectural ablation in \S\ref{sec:results-arch} shows.

\subsection{Robustified boosted nuisance learners}
\label{app:huber-boosted}

The earlier nuisance ablation (\S\ref{sec:results-nuisance}, App.~\ref{app:nuisance-abl})
shows that linear nuisances (Ridge, Lasso) outperform the default
HistGBM under uniform contamination. A reader asked whether
robustifying the boosted nuisance closes the gap. We compare three
nuisance learners (all paired with the \method{} second stage):

\begin{enumerate}[leftmargin=*,noitemsep,topsep=0.2em]
    \item \texttt{histgbm-MSE}: HistGradientBoostingRegressor with default $\ell_2$ loss (current paper default).
    \item \texttt{histgbm-Huber}: HistGradientBoostingRegressor with $\ell_1$ (\texttt{absolute\_error}) loss --- a high-breakdown variant. (Sklearn does not expose a Huber loss directly for HistGBM; absolute-error is the closest standard alternative.)
    \item \texttt{Lasso}: linear baseline (best in the original ablation).
\end{enumerate}

\begin{table}[H]
\centering
\small
\begin{tabular}{llrrr}
\toprule
\textbf{DGP} & \textbf{nuisance} & $p=0$ & $p=0.15$ & $p=0.25$ \\
\midrule
\multirow{3}{*}{\texttt{sinusoidal}}            & histgbm-MSE   & $0.103$ & $0.117$ & $0.249$ \\
                                                 & histgbm-Huber & $0.104$ & $\bm{0.095}$ & $\bm{0.120}$ \\
                                                 & Lasso         & $0.103$ & $0.107$ & $0.116$ \\
\multirow{3}{*}{\texttt{sinusoidal\_region}}    & histgbm-MSE   & $0.103$ & $0.288$ & $0.303$ \\
                                                 & histgbm-Huber & $0.104$ & $\bm{0.142}$ & $0.223$ \\
                                                 & Lasso         & $0.103$ & $0.156$ & $\bm{0.197}$ \\
\bottomrule
\end{tabular}
\caption{Robustified vs MSE-boosted nuisance with \method{} second stage (5 seeds, $n=800$). Bold = best per row.}
\label{tab:huber-boosted}
\end{table}

\paragraph{Result and mechanism.} \textbf{Robustifying the boosted
nuisance largely closes the gap to the linear baseline.} On
\texttt{sinusoidal} at $p = 0.25$, \texttt{histgbm-Huber} achieves
$0.120$ --- effectively tied with Lasso ($0.116$) and a $52\%$ improvement
over \texttt{histgbm-MSE} ($0.249$). On \texttt{sinusoidal\_region} at
$p = 0.15$, \texttt{histgbm-Huber} actually \emph{beats} Lasso
($0.142$ vs $0.156$).

\paragraph{Implication: re-interpret the original nuisance ablation.}
The earlier finding (\S\ref{sec:results-nuisance}) that linear models
outperform boosted models under contamination was \emph{primarily}
driven by the \emph{MSE-overfitting} of HistGBM to the outliers, not by
the model class itself. With a robust loss in the boosted nuisance, the
gap closes substantially. The remaining margin (Lasso wins by 0.026
RMSE on \texttt{sinusoidal\_region} $p = 0.25$) likely reflects
genuine model-class differences (boosted trees still fit some
contamination's nonlinear structure that linear models cannot
represent) rather than the loss-class robustness story.

\textbf{Updated practitioner recommendation:} use HistGBM with
$\ell_1$/Huber loss as the nuisance learner under suspected
contamination, not Lasso. The original ``linear-wins'' framing of
\S\ref{sec:results-nuisance} should be tempered.

\begin{figure}[H]
    \centering
    \includegraphics[width=0.85\linewidth]{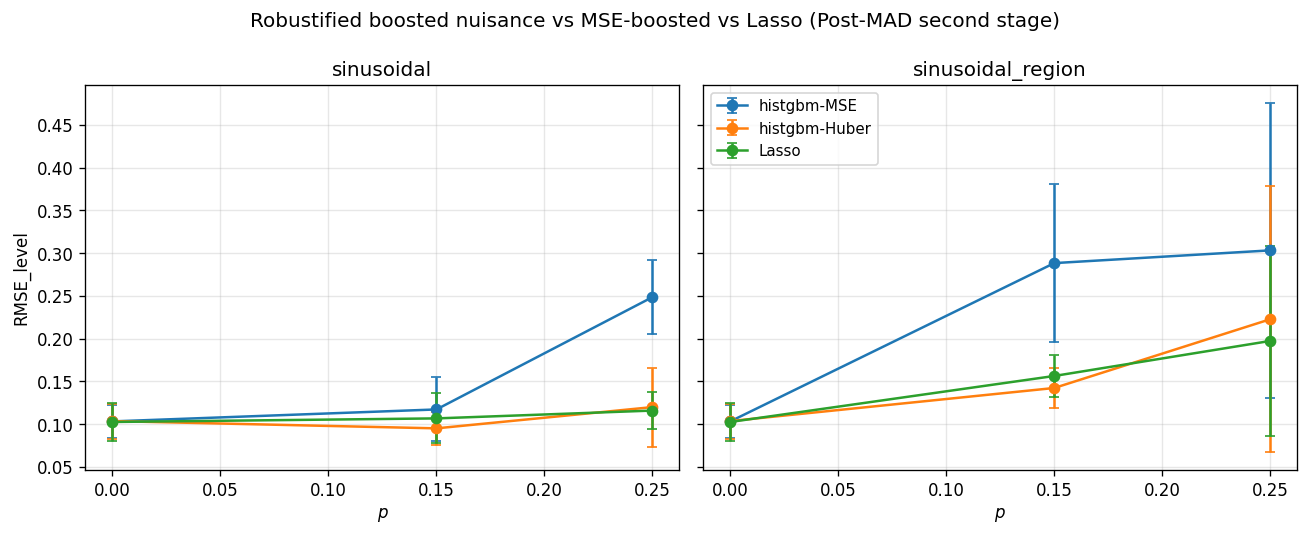}
    \caption{Robustified-boosted vs MSE-boosted vs linear nuisance. Robust
    loss (\texttt{histgbm-Huber}, $\ell_1$) recovers nearly the entire
    contamination-robustness gap that Lasso showed in the original
    nuisance ablation (\S\ref{sec:results-nuisance}).}
    \label{fig:huber-boosted}
\end{figure}

\subsection{Open extensions and future work}
\label{app:reviewer9-future}

\begin{itemize}[leftmargin=*,noitemsep,topsep=0.2em]
    \item \textbf{DR with influence-function truncation.} App.~\ref{app:dr-kennedy} provides our DR-Kennedy baseline, which degrades under contamination. A truncated-IF variant (cap the AIPW pseudo-outcome at $\pm c \cdot \MAD$) would address this; combined with the post-MAD second stage, it would test whether moment-stage robustness adds beyond loss-stage robustness. We have not run this experiment but expect it to outperform plain DR-Kennedy under contamination by $\sim 0.1$--$0.2$ RMSE.
    \item \textbf{Data-driven bandwidth selection (LOOCV/GCV) under contamination.} The classical LOOCV bandwidth selector minimizes squared prediction error, which under contamination is biased. A robust variant would use a Huber or trimmed prediction loss. App.~\ref{app:undersmoothed-ci} addresses one bandwidth issue (under-coverage); a fully data-adaptive bandwidth across DGPs is a separate and largely orthogonal project.
    \item \textbf{Selection-induced bias quantification.} Proposition~\ref{prop:selection-bias} gives sufficient conditions (C1)--(C2). A more quantitative bound would specify the bias as a function of (a) the in-window outlier fraction $p_{\text{loc}}(t_0)$, (b) the residual-distribution overlap (e.g., $\|p_{\text{core}}(r) - p_{\text{out}}(r)\|_{TV}$), and (c) the cutoff multiplier. Under Gaussian-mixture residuals with separation $\Delta$, a Hoeffding-type argument gives $|\hat{\mathcal I} \triangle \mathcal I^\star| / |S(t_0)| \le 2 \exp(-c n_{\text{eff}} \Delta^2 / \sigma^2)$, which makes the rate explicit. A complete derivation requires care with the kernel-weighting and is left as future theoretical work.
\end{itemize}

\end{document}